\documentclass[11pt]{article}
\usepackage[hmargin=1in, vmargin=0.8in]{geometry}
\usepackage[utf8]{inputenc}
\usepackage[T1]{fontenc}
\usepackage[english]{babel}
\usepackage{csquotes}

\usepackage[backend=biber, style=authoryear, maxcitenames=2, uniquelist=false, uniquename=init,
giveninits=true, doi=false, url=false, isbn=false, eprint=true]{biblatex}
\let\cite\parencite

\usepackage{amsmath}
\usepackage{amsthm}
\usepackage{amssymb}
\usepackage{mathrsfs}
\usepackage{mathtools}
\usepackage{booktabs}
\usepackage{array}
\usepackage{enumitem}
\usepackage{xcolor}
\usepackage{etoc}

\usepackage{graphicx}
\usepackage{float}
\graphicspath{{figures/}}

\usepackage{tikz}
\usetikzlibrary{arrows.meta,positioning,calc,decorations.pathreplacing}

\usepackage{hyperref}
\hypersetup{hidelinks,
  pdftitle={Universal Local Error and Realized Amplification for the First-Order EDM Predictor},
  pdfauthor={Nicolas Brosse, Arnak S. Dalalyan}}
\usepackage{xurl}
\usepackage{caption}
\usepackage{cleveref}

\newif\ifarxivlayout
\arxivlayouttrue
\newcommand{\coderepo}{\url{https://github.com/nbrosse/edm-error-propagation-code}}
\newcommand{\codeavailability}{available in the public code repository \coderepo}
\newcommand{\introfigwidth}{0.6\textwidth}
\newcommand{\reliabilitytablestyle}{\footnotesize\setlength{\tabcolsep}{3pt}}

\expandafter\ifx\csname ifarxivlayout\endcsname\relax
  \expandafter\newif\csname ifarxivlayout\endcsname
\fi

\newcommand{\R}{\mathbb{R}}
\newcommand{\E}{\mathbb{E}}
\newcommand{\Cov}{\operatorname{Cov}}
\newcommand{\Law}{\operatorname{Law}}
\newcommand{\Tr}{\operatorname{tr}}
\newcommand{\Id}{\operatorname{Id}}
\newcommand{\Idd}{\mathbf{I}_d}
\newcommand{\dd}{\mathrm{d}}
\newcommand{\norm}[1]{\left\lVert #1\right\rVert}
\newcommand{\ip}[2]{\left\langle #1,#2\right\rangle}
\newcommand{\Wtwo}{W_2}
\newcommand{\Rres}{\mathcal R}                        %
\newcommand{\Bd}[1][\rho]{B_{d,#1}}                   %
\newcommand{\gambar}{\overline{\gamma}}                 %
\newcommand{\budexp}{\omega}                          %
\NewDocumentCommand{\Lamos}{o m}{\Lambda^{\os\IfValueT{#1}{,#1}}_{#2}}
\newcommand{\Cres}{C_{\mathrm{res}}}                  %
\newcommand{\sighi}{\sigma_{\mathrm{hi}}}             %
\newcommand{\bhi}{b_{\mathrm{hi}}}                    %
\newcommand{\Khi}{K_{\mathrm{hi}}}                    %
\newcommand{\Lhi}{L_{\mathrm{hi}}}                    %
\newcommand{\thetahi}{\theta_{\mathrm{hi}}}           %
\newcommand{\thetalo}{\theta_{\mathrm{lo}}}           %
\newcommand{\kappahi}{\kappa_{\mathrm{hi}}}           %
\newcommand{\beff}{b_{\mathrm{eff}}}                  %
\newcommand{\Gsched}{\mathfrak G}                     %
\newcommand{\rhoEDM}{\rho_{\sf EDM}}                  %
\newcommand{\sigmad}{\tau_q}
\newcommand{\Dhat}{\widehat{D}}
\newcommand{\Phihat}{\widehat{\Phi}}
\newcommand{\vhat}{\widehat{v}}                       %
\newcommand{\auditDim}{3{,}072}                       %
\newcommand{\auditSigmaData}{0.5}                     %

\newcommand{\sync}{\mathrm{sync}}
\newcommand{\OT}{\mathrm{OT}}
\newcommand{\pts}{\mathrm{pts}}                       %
\newcommand{\os}{\mathrm{os}}                         %
\newcommand{\op}{\mathrm{op}}                         %
\newcommand{\cin}{c_{\mathrm{in}}}
\newcommand{\cskip}{c_{\mathrm{skip}}}
\newcommand{\cout}{c_{\mathrm{out}}}

\newcommand{\Fhat}{\widehat F}                        %

\NewDocumentCommand{\SFest}{o o}
  {\widetilde S_F
   \IfValueTF{#1}{\IfBlankTF{#1}
     {\IfValueT{#2}{^{(#2)}}}
     {^{[#1]\IfValueT{#2}{,(#2)}}}}
     {\IfValueT{#2}{^{(#2)}}}}
\NewDocumentCommand{\SFpest}{m o}{\widetilde S_{F,#1}\IfValueT{#2}{^{(#2)}}}

\newcommand{\qlev}[1]{^{[#1]}}                        %
\newcommand{\Quant}[3]{\operatorname{Quant}_{#1}\bigl\{#2\bigr\}_{#3}}

\NewDocumentCommand{\SF}{o}{S_F\IfValueT{#1}{\qlev{#1}}}
\NewDocumentCommand{\SFp}{m}{S_{F,#1}}
\NewDocumentCommand{\SPhi}{o m}{S_{\Phi,#2}\IfValueT{#1}{\qlev{#1}}}

\NewDocumentCommand{\bprobe}{o}{b_{\mathrm{probe}}\IfValueT{#1}{\qlev{#1}}}

\NewDocumentCommand{\zdisp}{m o}{\zeta^{#1}\IfValueT{#2}{_{#2}}}

\NewDocumentCommand{\Erate}{o m}{\mathcal E\IfValueT{#1}{^{#1}}_{#2}}
\NewDocumentCommand{\Epos}{o m}{\mathcal E\IfValueT{#1}{^{#1}}_{#2,+}}
\NewDocumentCommand{\Qsec}{o m}{Q\IfValueT{#1}{^{#1}}_{#2}}
\NewDocumentCommand{\Amp}{o m}{A\IfValueT{#1}{^{#1}}_{#2}}

\NewDocumentCommand{\LF}{m}{L_{F,#1}}                 %
\NewDocumentCommand{\Los}{o m}{\mathcal L^{\os\IfValueT{#1}{,#1}}_{#2}}

\theoremstyle{plain}
\newtheorem{theorem}{Theorem}
\newtheorem{proposition}[theorem]{Proposition}
\newtheorem{corollary}[theorem]{Corollary}
\newtheorem{lemma}[theorem]{Lemma}
\newtheorem{definition}[theorem]{Definition}
\newtheorem{assumption}[theorem]{Assumption}
\newtheorem{remark}[theorem]{Remark}

\crefname{assumption}{assumption}{assumptions}
\Crefname{assumption}{Assumption}{Assumptions}

\crefname{enumi}{assumption}{assumptions}
\Crefname{enumi}{Assumption}{Assumptions}

\newcommand{\auditBcrossQninety}{8.9}
\newcommand{\auditBsupercritQninety}{40.8}

\newcommand{\resKmin}{17}
\newcommand{\resKmax}{1088}

\newcommand{\resErrorOrderMin}{0.58}
\newcommand{\resErrorOrderMax}{0.60}

\newcommand{\resOmega}{0.41}
\newcommand{\resOmegaFitK}{136}

\newcommand{\resBoundExcessMin}{35}
\newcommand{\resBoundExcessMax}{40}

\newcommand{\resDirExcessMax}{2.2}
\newcommand{\resFieldRatioCoarse}{143}
\newcommand{\resFieldRatioFine}{11.1}

\newcommand{\resGaussInitMax}{6.2\times10^{-5}}
\newcommand{\resGaussDiscFine}{0.0136}
\newcommand{\resMisspecLambdaFine}{0.67}
\newcommand{\auditSegK}{68}
\newcommand{\auditSegSigmaLow}{0.004}
\newcommand{\auditSegRandLow}{-82}
\newcommand{\auditSegParLow}{0.49}
\newcommand{\auditSegWorstOverParLow}{424}

\newcommand{\auditSyncTrajectories}{1280}

\title{Universal Local Error and Realized Amplification\\ for the First-Order EDM Predictor}
\author{%
  Nicolas Brosse\\ CREST, ENSAE Paris, Institut Polytechnique de Paris\\
  \texttt{nicolas.brosse@ensae.fr}
  \and
  Arnak S. Dalalyan\\ CREST, ENSAE Paris, Institut Polytechnique de Paris\\
  \texttt{arnak.dalalyan@ensae.fr}}
\date{}

\begin{document}
\etocdepthtag.toc{main}

\maketitle

\begin{abstract}
  We analyze the first-order deterministic diffusion sampler of
  \textcite{karras_elucidating_2022}, termed EDM, in 2-Wasserstein distance by separating two
  sources of error: local discretization error and its amplification by subsequent learned steps.
  We prove that local error admits a universal bound: for any data distribution with finite second
  moment, the one-step discretization error is quadratic in the step size, with an explicit
  constant that does not depend on the data distribution. Error propagation, in contrast, depends
  on the learned network. At high noise levels, we exploit the network parametrization of EDM to
  derive an explicit contraction criterion. At low noise levels, we measure propagation through
  the amplification realized on the distributions transported by the sampler; this realized
  amplification can be arbitrarily smaller than the worst-case Lipschitz constant. This analysis
  yields an \(O(e^{\Lambda_K}/K)\) global discretization error for \(K\) sampling steps, where
  \(\Lambda_K\) is the low-noise log-amplification. Experiments on a one-dimensional Gaussian
  mixture show how measured amplification accounts for slower error decay on finite sampling
  grids. Diagnostics on a pretrained CIFAR–10 model illustrate related stability mechanisms
  without certifying the global assumptions.
\end{abstract}

\section{Introduction}
\label{sec:introduction}

A deterministic diffusion sampler generates samples by discretizing the probability flow from noise
to data using a trained denoising network. Its accuracy depends both on the errors introduced at
each step and on how subsequent steps propagate them. Controlling these two mechanisms is essential
for understanding accuracy at a limited number of network evaluations.

\begin{figure}[t]
  \centering
  \includegraphics[width=\introfigwidth]
  {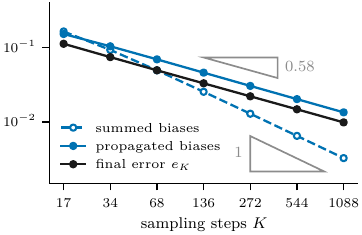}
  \vspace*{-4pt}
  \caption{%
  Local accuracy and accumulated error. For a four-Gaussian mixture with oracle denoising and exact
  initialization, summed local biases approach first-order decay, while final error decreases at
  orders \(\resErrorOrderMin\)--\(\resErrorOrderMax\) over the tested range. Propagating each bias
  through the measured amplification of the later steps (solid blue) recovers the slower decay of
  the error; see \Cref{subsec:evidence_law_level}.
  \vspace*{-14pt}
  }
  \label{fig:teaser_convergence}
\end{figure}
\Cref{fig:teaser_convergence} illustrates the distinction even with an exact denoiser. On a Gaussian
mixture, the sum of local discretization biases approaches first-order decay, while the final
Wasserstein error decreases at orders \(\resErrorOrderMin\)--\(\resErrorOrderMax\) over the tested
range. Propagating each bias through the measured amplification of the subsequent steps recovers
this slower decay. This rate is preasymptotic: it shows that local accuracy alone need not describe
the behavior on finite sampling grids.

Existing Wasserstein analyses control local error and propagation under various structural
assumptions. \textcite{beyler_convergence_2025} analyze deterministic and stochastic samplers on the
Gaussian smoothing path, using bounded support in their local estimates and a global Lipschitz
envelope of the learned map for propagation.
\textcite{gao_zhu_probability_flow_2025} obtain contraction under
strong log-concavity; other approaches use weaker geometric conditions
\cite{gentiloni_ocello_beyond_2025,bruno_sabanis_semiconvexity_2025}. We ask whether local error can
be bounded without such assumptions on the data law, and how much worst-case stability estimates can
overstate propagation.

We study the first-order deterministic EDM predictor of \cite{karras_elucidating_2022}, equivalently
the DDIM update of
\textcite{song_denoising_implicit_2021} in the
variance-exploding parametrization. Our main result is a universal local discretization bound. We
combine it with propagation measured on the exact and sampler laws, keeping separate what follows
from Gaussian smoothing and what requires control of the learned dynamics.

Our analysis has three parts.
\vspace*{-7pt}
\begin{enumerate}[leftmargin=1.6em]
  \itemsep=0pt
  \item \textbf{Universal local error.}
  We express the acceleration of the exact probability flow through conditional moments of Gaussian
  noise and bound its \(\mathbb L_2\) norm for every Borel data law. The bound has an explicit
  constant and a sharp noise exponent. For data with finite second moment, it yields a quadratic
  one-step discretization bound without support, smoothness, or log-concavity assumptions
  (\Cref{lem:local_euler_defect}). The key observation is that the relevant conditional moments are
  controlled by Gaussian noise moments, independently of the geometry of the data distribution.

  \item \textbf{Realized propagation.}
  EDM preconditioning provides an explicit sufficient condition for high-noise contraction. At lower
  noise, we retain the amplification of the Wasserstein distance between the two transported laws
  and derive directional and worst-case upper bounds. A smooth construction shows that realized
  amplification can remain nonexpansive while the worst-case majorants become arbitrarily large
  (\Cref{prop:strict_law_map_separation}). The directional bound measures expansion along an optimal
  coupling of the exact and sampler laws, avoiding a supremum over all spatial locations and
  directions.

  \item \textbf{Global error control.}
  We separate initialization, discretization, learning, and terminal smoothing errors
  (\Cref{thm:global_edm_wasserstein}). At fixed positive noise endpoints, the discretization
  contribution is \(O(K^{-1}e^{\Lambda_K})\), where \(\Lambda_K\) accumulates the positive
  logarithms of the realized low-noise amplification factors. A uniform bound on \(\Lambda_K\) gives
  first-order decay; logarithmic growth gives the rates in \Cref{cor:budget_refinement_rates}.
  High-noise contraction damps early errors, while EDM preconditioning expresses the learning
  contribution through the excess risk of a normalized regression problem.
\end{enumerate}

The propagation estimate is a posteriori: we do not derive an amplification envelope from training
or architecture. Experiments on the Gaussian mixture examine its explanatory value and the gap to
worst-case bounds. Diagnostics on a pretrained CIFAR--10 model illustrate related mechanisms without
certifying the global stability assumptions. The analysis covers the deterministic predictor; Heun
corrections and stochastic churn remain outside its scope. \Cref{app:related_work_extended} provides
detailed comparisons with prior work.

\section{Setting and One-Step Decomposition}
\label{sec:denoising_path_predictor}
\label{sec:one_step_error}

This section introduces deterministic diffusion sampling and the first-order EDM predictor studied
in this paper. We write \(\mathcal P_2(\R^d)\) for the set of Borel probability measures on \(\R^d\)
with finite second moment and \(f_\#\mu\) for the pushforward of \(\mu\) by a measurable map \(f\).
A hat marks a learned object. For a random vector \(Y\) and a measurable vector-valued function
\(f\), respectively, we use the notation \(\|Y\|_{\mathbb L_2} \coloneqq (\mathbb
E[\|Y\|^2])^{1/2}\) and \(\| f\|_{\mathbb L_2(\mu)} \coloneqq (\int \|f(x)\|^2 \,\mu (\mathrm d
x))^{1/2}\). For a more detailed introduction to the underlying mathematical framework, see
\textcite{holderrieth2025}; complements specific to our setting are provided in
\Cref{app:main_text_complements}.

Let \(q\in\mathcal P_2(\R^d)\) be the target distribution. For independent vectors \(X_0\sim q\) and
\(Z\sim\mathcal N(0,\Idd)\), set
\begin{equation*}
  X_\sigma=X_0+\sigma Z,\quad\text{ and} \quad \mu_\sigma =\Law(X_\sigma) .
\end{equation*}
The mapping \(\sigma \mapsto \mu_\sigma\), for \(\sigma\in(0, \infty)\) defines a smooth curve in
\(\mathcal P_2 (\R^d)\) starting from the target distribution \(q =\mu_0\). In addition, at high
noise levels \(\bar\sigma\), the normalized variable \(X_{\bar\sigma}/\bar\sigma\) has a
distribution close to \(\mathcal N(0,\Idd)\). This motivates using \(\mathcal
N(0,\bar\sigma^2\Idd)\) as the latent distribution for generative modeling. The same curve can be
realized by solving an ordinary differential equation based on the velocity field
\begin{equation}\label{eq:velocity}
  v_\sigma(x)=\big(x-D_\sigma(x)\big)/{\sigma} = -\sigma\nabla\log p_\sigma(x)
\end{equation}
where \(D_\sigma(x) = \mathbb E [\,X_0\, |\, X_\sigma = x\,]\) is the posterior-mean denoiser and
\(p_\sigma\) is the density of \(\mu_\sigma\). More precisely, for every compact interval \( I
\subset (0,\infty)\), and for every \( (\sigma_0,
x) \in I\times \R^d\), the differential equation
\begin{equation*}
  \frac{\dd x_r}{\dd r} = v_r(x_r), \quad r\in I, \quad x_{\sigma_0} = x
\end{equation*}
has a unique solution on the whole interval \(I\). Denote its solution map by \(
\Psi_{\sigma,\sigma_0}:\R^d\to\R^d \), so that
\begin{equation}\label{eq:Psi}
  \frac{\dd}{\dd\sigma}\Psi_{\sigma,\sigma_0}(x) = v_\sigma\bigl(\Psi_{\sigma,\sigma_0}(x)\bigr),
  \quad \Psi_{\sigma_0,\sigma_0}(x)=x .
\end{equation}
For every \(\sigma,\sigma_0\in I\), the map \(\Psi_{\sigma,\sigma_0}\) is a diffeomorphism of
\(\R^d\), with inverse \(\Psi_{\sigma_0,\sigma}\). In addition, for all \(\sigma_0,\sigma_1,\sigma_2
\in I\), \(\Psi_{\sigma_2,\sigma_0} = \Psi_{
\sigma_2, \sigma_1} \circ\Psi_{\sigma_1,\sigma_0}\).
This flow transports the heat-regularized laws:
\begin{equation}\label{eq:pushforward_sigma}
  \mu_\sigma = (\Psi_{\sigma,\sigma_0})_\#\mu_{\sigma_0}, \qquad \sigma,\sigma_0\in I .
\end{equation}
The main idea behind flow matching, DDIM, EDM and their multiple variants is to choose a small value
\(\underline{\sigma}>0\), a large value \(\overline{\sigma}\), to approximate
\(\Psi_{\underline{\sigma}, \overline{\sigma}}\) by some map \(\widehat\Psi\) based on the training
data, and then to define the generative model as \(\widehat\Psi_\#
\mathcal N(0,\overline{\sigma}^2\Idd) \).

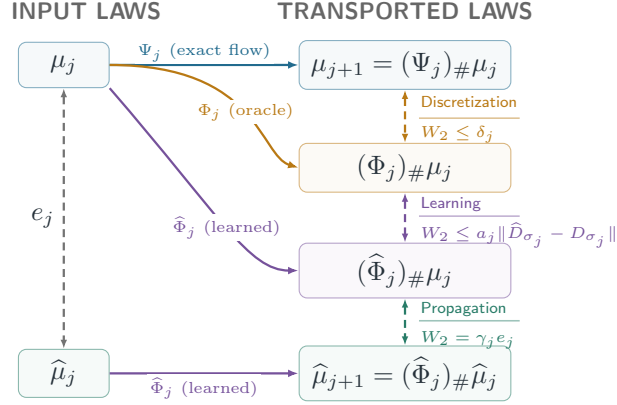
\begin{figure}[!t]
  \centering
  \begingroup
\definecolor{edmExact}{HTML}{226B8E}
\definecolor{edmDisc}{HTML}{B97914}
\definecolor{edmLearn}{HTML}{76529B}
\definecolor{edmProp}{HTML}{287C68}
\definecolor{edmInk}{HTML}{26313A}
\begin{tikzpicture}[
  x=0.74cm,y=0.8cm,
  font=\small,
  every node/.style={text=edmInk},
  law/.style={draw=black!22,fill=white,rounded corners=3pt,
    minimum height=6mm,inner sep=4pt,align=center},
  source/.style={law,minimum width=12mm},
  target/.style={law,minimum width=28mm},
  transport/.style={-{Latex[length=1.4mm]},line width=.85pt},
  edm distance/.style={{Latex[length=1.2mm]}-{Latex[length=1.2mm]},
    densely dashed,line width=.85pt},
  maplabel/.style={fill=white,inner sep=1pt},
  heading/.style={font=\sffamily\footnotesize\bfseries,text=black!60}
]
\node[heading] at (-.5,4.5) {INPUT LAWS};
\node[heading] at (5.2,4.5) {TRANSPORTED LAWS};

\node[source,draw=edmExact!55,fill=edmExact!4] (mu) at (-0.9,3.6)
  {$\mu_j$};
\node[source,draw=edmProp!55,fill=edmProp!4] (muhat) at (-0.9,-1.5)
  {$\widehat\mu_j$};

\node[target,draw=edmExact!55,fill=edmExact!4] (exact) at (5.2,3.6)
  {$\mu_{j+1}=(\Psi_j)_\#\mu_j$};
\node[target,draw=edmDisc!55,fill=edmDisc!4] (oracle) at (5.2,1.9)
  {$(\Phi_j)_\#\mu_j$};
\node[target,draw=edmLearn!55,fill=edmLearn!4] (learned) at (5.2,0.2)
  {$(\widehat\Phi_j)_\#\mu_j$};
\node[target,draw=edmProp!55,fill=edmProp!4] (out) at (5.2,-1.5)
  {$\widehat\mu_{j+1}=(\widehat\Phi_j)_\#\widehat\mu_j$};

\draw[transport,edmExact] (mu.east) --
  node[maplabel,above,text=edmExact] {\tiny$\Psi_j$ (exact flow)} (exact.west);
\draw[transport,edmDisc] (mu.east)
  .. controls (2.6,3.6) and (2.6,2.1) ..
  node[maplabel,pos=.5,text=edmDisc] {\tiny$\Phi_j$ (oracle)} (oracle.west);
\draw[transport,edmLearn] (mu.south east)
  .. controls (1.9,1.1) and (2.4,.2) ..
  node[maplabel,pos=.50,text=edmLearn] {\tiny$\widehat\Phi_j$ (learned)} (learned.west);
\draw[transport,edmLearn] (muhat.east) --
  node[maplabel,below,text=edmLearn] {\tiny$\widehat\Phi_j$ (learned)} (out.west);

\draw[edm distance,black!55] (mu.south) --
  node[maplabel,left=2pt] {$e_j$} (muhat.north);
\draw[edm distance,edmDisc] (exact.south) -- (oracle.north);
\draw[edm distance,edmLearn] (oracle.south) -- (learned.north);
\draw[edm distance,edmProp] (learned.south) -- (out.north);

\node[anchor=west,align=left,text=edmDisc] at (5.3,2.7)
  {\textsf{\tiny Discretization}\\[-1pt]
  \tiny $W_2\le\delta_j$};
\node[anchor=west,align=left,text=edmLearn] at (5.3,1)
  {\tiny\textsf{Learning}\\[-1pt]
   \tiny$W_2\le a_j\|\widehat D_{\sigma_j}-D_{\sigma_j}\|$};
\node[anchor=west,align=left,text=edmProp] at (5.3,-0.7)
  {\tiny\textsf{Propagation}\\[-1pt] 
  \tiny $W_2=\gamma_j e_j$};
\draw[edmDisc!45,line width=.6pt] (5.45,2.72) -- (7.25,2.72);
\draw[edmLearn!45,line width=.6pt] (5.45,1.1) -- (7.25,1.1);
\draw[edmProp!45,line width=.6pt] (5.45,-0.63) -- (7.25,-0.63);
\end{tikzpicture}
\endgroup
  \vspace*{-10pt}
  \caption{One-step Wasserstein error decomposition. Solid arrows correspond to
  pushforwards, whereas dashed links are \(W_2\) distances.
  \vspace*{-10pt}}
  \label{fig:one_step_decomposition}
\end{figure}

To be more precise, fix a schedule \(\overline{
\sigma}=\sigma_0 > \ldots > \sigma_K = \underline{
\sigma} >0\), and assume that we have a trained
version \(\widehat{D}_\sigma\) of the posterior mean denoiser \(D_\sigma\) at each noise level
\(\sigma\in\{\sigma_0,\ldots,\sigma_K\}\). Set \(\mu_j = \mu_{\sigma_j}\), \( a_j = ( \sigma_j-
\sigma_{j+1})/\sigma_j\in(0,1)\), and
\(\Psi_j = \Psi_{\sigma_{j+1}, \sigma_j}\), so that \(\mu_{j+1} = (\Psi_j)_\#\mu_j\). In
\cref{eq:Psi}, freezing the velocity at the starting noise level \(\sigma_j\) and using
\eqref{eq:velocity}, gives the one-step oracle and learned EDM predictors
\begin{align}
  \Phi_j(x) &= (1-a_j)x+a_jD_{\sigma_j}(x), \label{eq:oracle_edm_predictor} \\
  \Phihat_j(x) &= (1-a_j)x+a_j\Dhat_{\sigma_j}(x). \label{eq:learned_edm_predictor}
\end{align}
For \(\widehat\mu_0=\mathcal N(0,\sigma_0^2\Idd)\), the sampler laws and errors are
\begin{equation}\label{eq:learned_sampler_laws_errors}
  \widehat\mu_{j+1}=(\Phihat_j)_\#\widehat\mu_j, \quad e_j=\Wtwo(\widehat\mu_j,\mu_j).
\end{equation}
With these notations, the triangle inequality and the canonical coupling give
\begin{equation}\label{eq:edm_final_error_with_terminal_bias}
  \Wtwo(\widehat\mu_K,q)\le e_K+\sigma_K\sqrt d.
\end{equation}
Therefore, the main goal is to understand the behavior of the final error \(e_K\). The deployed
output \((\Dhat_{\sigma_K})_\#\widehat\mu_K\), obtained by one additional denoising step, is treated
in \Cref{subsec:deployed_output}.

We complete this section by providing a simple one-step recursion for the accumulated error \(
e_j\). Its proof follows straightforwardly from the triangle inequality. Let us define the local
discretization bias by
\begin{equation}\label{eq:def_local_defect_learned_map}
  \delta_j=\bigl\|\Psi_j(X_{\sigma_j})-\Phi_j(X_{\sigma_j})\bigr\|_{\mathbb L_2},
\end{equation}
and the realized amplification factor
\begin{equation}\label{eq:def_raw_amplification}
  \gamma_j
  = \frac{\Wtwo\bigl((\Phihat_j)_\#\mu_j,(\Phihat_j)_\#\widehat\mu_j\bigr)}{\Wtwo(\mu_j,\widehat\mu_j)},
\end{equation}
with the convention that \(\gamma_j=1\) if \(\Wtwo(\mu_j,\widehat\mu_j)=0\). For every
\(j=0,\ldots,K-1\),
\begin{equation}\label{eq:law_level_one_step}
  e_{j+1} \le \gamma_je_j+\delta_j
  + a_j \bigl\|\Dhat_{\sigma_j}-D_{\sigma_j} \bigr\|_{\mathbb L_2(\mu_j)},
\end{equation}
see \Cref{fig:one_step_decomposition} for a visual representation.

\section{Universal Local Error}
\label{sec:discretization_error}

The one-step decomposition in \cref{eq:law_level_one_step} separates error propagation,
discretization bias, and denoiser approximation. This section bounds the local discretization bias
\(\delta_j\). We show that its second-order behavior follows from a universal bound on the
acceleration of the exact flow. Proofs are provided in \Cref{app:local_error}.

To compare samplers with different numbers of steps, we keep the endpoints \(\underline\sigma\) and
\(\overline\sigma\) fixed and refine the grid between them. The EDM schedule places noise levels
uniformly in the transformed variable \(\eta=\sigma^{1/\rhoEDM}\), which we call the numerical
clock, see \Cref{fig:clock_grid1}. To ease notation, we write \(\rho=1/\rhoEDM\); thus the commonly
used EDM choice \(\rhoEDM = 7\) corresponds to \(\rho=1/7\). By the chain rule, the flow in the
clock \(\eta =
\sigma^\rho\) satisfies
\begin{equation}\label{eq:eta_probability_flow_ode}
  \frac{\dd x}{\dd\eta}=w_\eta(x),\quad w_\eta(x)=\frac{\eta^{\frac{1}{\rho}-1}
  }{\rho}\, v_{\eta^{1/\rho}}(x).
\end{equation}

\begin{definition}[Refinement family
  \(\Gsched\)]
  \label{def:schedule_family}
  Fix \(\rho>0\), \( 0 < \underline{
  \sigma} < \overline{\sigma}\), and a positive
  integer \(K_\circ\). Set \(\underline{\eta} = \underline{\sigma}^\rho\) and \(\overline{\eta} =
  \overline{\sigma}^\rho\). For \(K\in\mathscr K=\{K_\circ2^r: r\in\mathbb N\}\), set \(h_K =
  (\overline{\eta}-\underline{\eta})/{K}
  \) and let \(G_K: = \{\eta_j=\overline{\eta} -
  jh_K : 0\leqslant j\leqslant K\}\). The family of nested grids is \(\Gsched =
  \{G_K:K\in\mathscr K\}\).
\end{definition}

In the asymptotic results throughout this paper, when we assume \(K\) goes to infinity, the target
\(q\), the denoisers \((\Dhat_\sigma)_{
\sigma\in[\underline{\sigma},\overline{\sigma}]}\),
and the initial distribution \(\widehat\mu_0 =
\mathcal N(0,\overline{\sigma}^2\Idd)\) are held
fixed. These schedules share their endpoints and are nested: \(G_K\subset G_{K'}\) if \(K,K'\in
\mathscr K\) and \(K<K'\). Their relative clock
increments satisfy
\begin{equation}\label{eq:schedule_family_relative_mesh}
  {h_K}/{\eta_j}\leqslant{h_K}/{\underline{\eta}} \eqqcolon \theta_K\longrightarrow0,
\end{equation}
as \(K\to\infty\). We add the subscript \(K\) when comparing quantities on different schedules, for
example \(\delta_{j,K}\) and \(\gamma_{j,K}\). A bound is uniform under refinement if its constant
is independent of \(K\).

\paragraph{Acceleration in the numerical clock.}
\label{subsec:acceleration_numerical_clock}
Define
\begin{equation}\label{eq:def_A_eta}
  \mathcal A_\rho(\sigma,x) =\bigl[\partial_\eta w_\eta(x)
  +\nabla_xw_\eta(x)\,w_\eta(x)\bigr]_{ \eta=\sigma^\rho}.
\end{equation}
We refer to this quantity as the acceleration field, expressed in terms of the noise level, since
along an exact trajectory \(x_\eta\) solution to \eqref{eq:eta_probability_flow_ode}, the chain rule
yields \(\dd^2x_\eta/\dd\eta^2 =\mathcal A_\rho(\eta^{1/\rho},x_\eta)\). Controlling this
acceleration in \(\mathbb L_2\) therefore controls second-order remainders along the flow. The key
fact is that such a bound can be obtained independently of the target distribution.

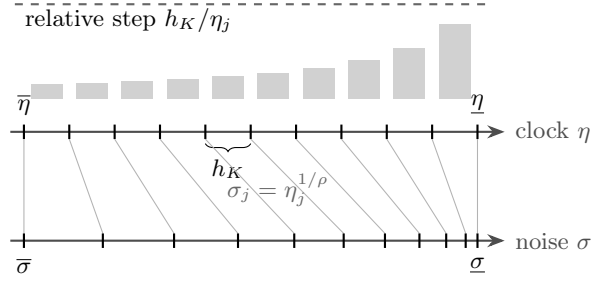
\begin{figure}[t]
  \centering
  \begin{tikzpicture}[x=0.6cm,y=0.8cm,>=Stealth]
  \def\etaticks{0,1,2,3,4,5,6,7,8,9,10}
  \foreach \x/\r in {0/0.090,1/0.099,2/0.110,3/0.123,4/0.141,5/0.164,6/0.196,7/0.243,8/0.321,9/0.474}{
    \fill[black!18] (\x+0.15,1.15) rectangle (\x+0.85,{1.15+2.6*\r});
  }
  \draw[dashed,thick,black!60] (-0.2,{1.15+2.6*0.6}) -- (10.2,{1.15+2.6*0.6})
    node[right,font=\footnotesize,align=left] {};
  \node[font=\footnotesize,anchor=west] at (-0.2,2.45) {relative step \(h_K/\eta_j\)};
  \draw[->,thick,black!70] (-0.3,0.6) -- (10.6,0.6) node[right,font=\footnotesize] {clock \(\eta\)};
  \foreach \x in \etaticks { \draw[thick] (\x,0.48) -- (\x,0.72); }
  \node[font=\footnotesize,above] at (0,0.72) {\(\overline\eta\)};
  \node[font=\footnotesize,above] at (10,0.72) {\(\underline\eta\)};
  \draw[decorate,decoration={brace,amplitude=3pt,mirror,raise=1pt}] (4,0.45) -- (5,0.45)
    node[midway,below=3pt,font=\footnotesize] {\(h_K\)};
  \draw[->,thick,black!70] (-0.3,-1.2) -- (10.6,-1.2) node[right,font=\footnotesize] {noise \(\sigma\)};
  \foreach \xe/\xs in {0/0.0,1/1.74,2/3.31,3/4.72,4/5.96,5/7.05,6/7.96,7/8.72,8/9.31,9/9.74,10/10.0}{
    \draw[black!30] (\xe,0.48) -- (\xs,-1.08);
    \draw[thick] (\xs,-1.32) -- (\xs,-1.08);
  }
  \node[font=\footnotesize,below] at (0,-1.32) {\(\overline\sigma\)};
  \node[font=\footnotesize,below] at (10,-1.32) {\(\underline\sigma\)};
  \node[font=\footnotesize,black!60] at (5.6,-0.35) {\(\sigma_j=\eta_j^{1/\rho}\)};
  \end{tikzpicture}
  \vspace*{-5pt}
  \caption[Clock grid and mesh regularity]{A member \(G_K\) of the refinement family
  with \(\rho=1/2\) and \(K=10\).
  \vspace*{-10pt}}
  \label{fig:clock_grid1}
\end{figure}

\begin{theorem}[Universal \(\mathbb L_2\) bound]
  \label{thm:global_curvature_bound}
  For every \(\rho>0\) and every Borel probability measure \(q\) on \(\R^d\),
  \begin{equation}\label{eq:Aeta_global_bound}
    \bigl\|\mathcal A_\rho(\sigma,X_\sigma)\bigr\|_{\mathbb L_2} \le B_{d,\rho}\,\sigma^{1-2\rho},
    \qquad\forall\sigma>0
  \end{equation}
  with \(B_{d,\rho} = \rho^{-2}\sqrt d
  \bigl[|1-\rho| + 4(d+3)\bigr]\).
\end{theorem}

The proof expresses the acceleration as \(\sigma^{1-2\rho}/\rho^2\) times a combination of
conditional moments of \(Z\) given \(X_\sigma = x\). Jensen's inequality then bounds the relevant
terms by Gaussian moments, independently of \(q\). The exponent \(1-2\rho\) is sharp over all target
distributions; see Prop.~\ref{prop:acceleration_exponent_sharp} in \Cref{subsec:gaussian_sharpness}.
Whether the dimension dependence \(O_\rho(d^{3/2})\) is sharp remains open.

\paragraph{Local bias of EDM.}

The numerical clock determines the noise levels at which the predictor is evaluated. The oracle
predictor is given by \cref{eq:oracle_edm_predictor}. If \(\rho\neq 1\), this is generally different
from an Euler step\footnote{See \Cref{rem:frozen_step_not_eta_euler} in
\Cref{app:eta_euler_vs_edm_predictor}.} for \(w_\eta\). We bound its bias using
\Cref{thm:global_curvature_bound} at \(\rho=1\).

\begin{lemma}
  \label{lem:local_euler_defect}
  Let \(q\in\mathcal P_2(\R^d)\) and let \(\sigma_0>\cdots>\sigma_K>0\) be any schedule. For every
  \(j=0,\ldots,K-1\), the bias defined in \cref{eq:def_local_defect_learned_map} satisfies
  \begin{equation}\label{eq:def_Delta_k_curvature_bound}
    \delta_j\leqslant\Delta_j = B_{d,1}\Bigl[ (\sigma_j-\sigma_{j+1})
    -\sigma_{j+1} \log\frac{\sigma_j}{\sigma_{j+1}}\Bigr].
  \end{equation}
\end{lemma}

The proof is based on a second-order expansion. Indeed, the integral remainder along the exact flow,
together with \Cref{thm:global_curvature_bound}, implies
\begin{equation*}
  \delta_j \le B_{d,1}\int_{\sigma_{j+1}}^{\sigma_j} \frac{s-\sigma_{j+1}}{s}\,\dd s =\Delta_j.
\end{equation*}
This bound holds for every noise schedule and depends on the target distribution only through the
dimension.

On the refinement family \(\Gsched\), expanding in the relative clock increment \(h_K/\eta_j\) and
setting \(\sigma_j=\eta_j^{1/\rho}\) yields
\begin{equation}\label{eq:local_defect_clock_expansion}
  \Delta_{j,K}
  =\frac{B_{d,1}}{2\rho^2}\,h_K^2\sigma_j^{1-2\rho}
  \bigg[1+O_\rho\bigg(\frac{h_K}{\eta_j}\bigg)\bigg].
\end{equation}
Thus the bound on the local bias is quadratic in the clock step, with a noise dependence
\(\sigma_j^{1-2\rho}\). In particular, \(\rho = 1/2\), or \(\rhoEDM=2\), makes the leading bound
independent of the noise level. This choice balances the universal bound across noise scales; it
need not minimize the actual local bias on practical grids. The mixture experiment in
\Cref{sec:numerical_evidence}, for example, ranks the clocks in the reverse order when measured
biases are compared (\Cref{tab:control_clock}). For the commonly used choice \(\rho=1/7\), or
\(\rhoEDM=7\), the leading bound instead contains \(\sigma_j^{5/7}\) and is %
largest at high noise. The contraction mechanism studied in \Cref{sec:two_regime_stability} can
reduce the contribution of these early biases to the final error.

\section{EDM Preconditioning}
\label{sec:preconditioning_residual}

We now examine how the learned network enters the other two terms of \cref{eq:law_level_one_step}.
The EDM parametrization expresses the denoiser approximation error as a normalized regression error
and provides an explicit sufficient condition for contraction of the learned predictor. Proofs and
complements are provided in \Cref{app:preconditioning}. We first state the standing assumption under
which the learned sampler laws and Wasserstein errors introduced in~\Cref{sec:one_step_error} are
well defined.
\begin{assumption}
  \label{ass:wellposed}
  On every schedule in \(\Gsched\), each \(\Phihat_j\) is locally Lipschitz and has at most linear
  growth: for some \(A_j<\infty\), \(\|\Phihat_j(x)\|\le A_j(1+\|x\|)\) for all \(x\in\R^d\).
\end{assumption}
Since the initial distribution has finite second moment, this assumption ensures that
\(\widehat\mu_j\in\mathcal P_2(\R^d)\) and \(e_j<\infty\).

\paragraph{EDM parametrization.}
For the remainder of the analysis, we assume that the target distribution is centered, \(\mathbb
E[X_0]=0\), and write \(\sigmad^2 =\mathbb E\|X_0\|^2/d>0\). Set
\(\alpha(\sigma)={\sigmad^2}/({\sigma^2 +
\sigmad^2})\) and
\begin{equation}\label{eq:edm_coefficients}
  \beta(\sigma)={\sigma\sigmad} \cin(\sigma), \ \ \cin(\sigma)=(\sigma^2
  + \sigmad^2)^{-1/2}\!\!\!\!.
\end{equation}
These are the coefficients \(\cskip\), \(\cout\), and \(\cin\) of
\cite[Sec.~5]{karras_elucidating_2022}. Writing \(\alpha_j=\alpha(\sigma_j)\), \(\beta_j =
\beta(\sigma_j)\), and \(c_j=\cin(
\sigma_j)\), the learned denoiser is
\begin{equation}\label{eq:edm_learned_denoiser_step}
  \Dhat_{\sigma_j}(x) =\alpha_j x +\beta_j \,\Fhat_j\bigl(c_j x\bigr),
\end{equation}
where \(\Fhat_j=\Fhat_{\sigma_j}\) is the network evaluated at noise level \(\sigma_j\).
Substituting this expression into \cref{eq:learned_edm_predictor} gives
\begin{equation}\label{eq:edm_learned_map_explicit}
  \Phihat_j(x) = \bigl(1 - a_j(1 - \alpha_j)\bigr) x + a_j\beta_j \,\Fhat_j\bigl(c_j x\bigr).
\end{equation}
Our analysis uses this form without imposing a particular network architecture.

\paragraph{Normalized regression error.}
The exact denoiser can be written in the same form, which defines its normalized residual function
\(F_\sigma\):
\begin{equation*}
  D_\sigma(x) = \alpha(\sigma)x + \beta(\sigma) \, F_\sigma\bigl(\cin(\sigma) x\bigr).
\end{equation*}
We write \(F_j=F_{\sigma_j}\). Define the normalized input and regression target by \(\widetilde
X_\sigma=\cin (\sigma) X_\sigma\) and \(\mathcal T_\sigma = \big(X_0 - \alpha(\sigma) X_\sigma
\big)/
{\beta(\sigma)}\). Then \(F_\sigma(\widetilde
X_\sigma) = \mathbb E[\mathcal T_\sigma\mid
\widetilde X_\sigma]\).
Thus \(F_\sigma\) is the optimal predictor of \(\mathcal T_\sigma\) from \(\widetilde X_\sigma\)
under squared loss. Let
\begin{equation*}
  \Rres(\sigma) =\|\Fhat_\sigma(\widetilde X_\sigma) -F_\sigma(\widetilde X_\sigma)\|_{\mathbb L_2},
  \quad \Rres_j=\Rres(\sigma_j).
\end{equation*}
The orthogonality of conditional expectation yields
\begin{equation}\label{eq:edm_regr_pyth}
  \Rres_j^2 = \|\Fhat_j(\widetilde X_{\sigma_j}) -\mathcal T_{\sigma_j}\|^2_{\mathbb L_2}
  - \|F_j(\widetilde X_{\sigma_j}) - \mathcal T_{\sigma_j} \|^2_{\mathbb L_2}.
\end{equation}
Consequently, \(\Rres_j^2\) is the excess population risk of the normalized regression problem.
Moreover,
\begin{equation}\label{eq:edm_den_error_residual}
  \|\Dhat_{\sigma_j}-D_{\sigma_j}\|_{\mathbb L_2(\mu_j)} =\beta_j\Rres_j,
\end{equation}
so the learning term in \cref{eq:law_level_one_step} is exactly \(a_j\beta_j\Rres_j\). The choice of
coefficients gives \(\mathbb E\|\mathcal T_\sigma\|^2/d=1\), which motivates measuring the
regression error relative to \(\sqrt d\).
\begin{assumption}
  \label{ass:uniform_residual_accuracy}
  There is \(\Cres\geqslant0\) such that \(\Rres(\sigma)\leqslant\sqrt d\,\Cres\), for all \(\sigma
  \in [\underline \sigma,
  \overline\sigma]\).
\end{assumption}
The first term on the right-hand side of \eqref{eq:edm_regr_pyth} can be estimated from noisy
validation pairs. The second term, the Bayes risk, is generally unknown, so the excess risk cannot
be recovered from the prediction loss. The prediction loss upper-bounds \(\mathcal R_j^2\), but
checking that \(\Cres\) is small over all noise levels remains difficult in practice.

\paragraph{EDM stability margin.}
The same parametrization yields a sufficient condition for contraction. Suppose \(\Fhat_j\) is
globally \(\LF{j}\)-Lipschitz with respect to its normalized input, and define
\begin{equation}\label{eq:def_b_k_edm}
  b_j =1-\alpha_j-\beta_jc_j\LF{j} =\ c_j^2(\sigma_j^2-\sigma_j\sigmad\LF{j}).
\end{equation}
Since \(1-a_j(1-\alpha_j)>0\), \cref{eq:edm_learned_map_explicit} implies
\begin{equation}\label{eq:edm_learned_map_lip_margin}
  \operatorname{Lip}(\Phihat_j)\leqslant1-a_jb_j.
\end{equation}
Thus \(b_j>0\) is equivalent to \(\LF{j}<\sigma_j/\sigmad\) and certifies strict contraction. High
noise, therefore, permits larger network Lipschitz constants while retaining this contraction
guarantee. When \(b_j\leqslant0\), the bound does not certify strict contraction; the learned
predictor may nevertheless contract the laws it transports.

\paragraph{Assembled recursion.}
Combining the bound from \Cref{lem:local_euler_defect} with \eqref{eq:edm_den_error_residual}, we
obtain
\begin{equation}\label{eq:one_step_learned_map_bound_integrated_curvature}
  e_{j+1} \leqslant\gamma_je_j+\Delta_j+a_j\beta_j\Rres_j.
\end{equation}
Under \Cref{ass:uniform_residual_accuracy}, the learning contribution is at most \(a_j\beta_j\sqrt
d\,\Cres\).

To summarize this section, the residual \(\Rres_j\) controls the learning term, while a positive
margin \(b_j\) certifies contraction. The recursion itself requires no contraction assumption. The
next section combines this contraction certificate at high noise with control of the amplification
accumulated over the remaining steps.

\section{High-Noise Contraction \ifarxivlayout and\else \&\fi{} Realized Amplification}
\label{sec:two_regime_stability}

We now control the propagation term in \cref{eq:one_step_learned_map_bound_integrated_curvature}. At
high noise, the EDM margin provides a sufficient condition for contraction. Over the remaining
steps, we retain the realized amplification. Its cumulative value is the quantity used in the global
bound; directional and worst-case estimates provide upper bounds on it. Proofs and complements are
provided in \Cref{app:propagation}.

\paragraph{High-noise contraction.}
Fix \(\sighi>0\) and define
\begin{equation}\label{eq:def_split_index}
  \Khi=\max\{j\in\{0,\ldots,K\}: \sigma_j\geqslant\sighi\},
\end{equation}
with \(\Khi=0\) if the set is empty. Steps \(j<\Khi\) form the high-noise block; the remaining steps
form the low-noise block.

\begin{assumption}[High-noise transport contraction]
  \label{ass:high_noise_contraction}
  There is \(\bhi\in(0,1)\), independent of \(K\), such that on every \(G_K\in\Gsched\), for
  \(j<\Khi\),
  \begin{equation}\label{eq:high_noise_transport_contraction}
    \Wtwo\bigl((\Phihat_j)_\#\mu_j, (\Phihat_j)_\#\widehat\mu_j\bigr) \leqslant(1-\bhi a_j)e_j.
  \end{equation}
\end{assumption}

\begin{proposition}
  \label{prop:high_noise_stability_certificate}
  If, on every schedule, each high-noise network \(\Fhat_j\) is globally \(\LF{j}\)-Lipschitz in its
  normalized input and \(b_j\geqslant\bhi\) for some \(\bhi\in(0,1)\) independent of \(K\), then
  \Cref{ass:high_noise_contraction} holds.
\end{proposition}

This follows by applying \eqref{eq:edm_learned_map_lip_margin} to an optimal coupling. For the
oracle with \(\operatorname{supp}q\subseteq B(0,R)\), the transport assumption holds above
\(\sighi=R/\sqrt{1-\bhi}\) (\Cref{subsec:oracle_contraction}).

\paragraph{Realized amplification.}
On the remaining steps, we impose no contraction assumption. We now accumulate the logarithms of the
realized factors \(\gamma_j\).

\begin{definition}[Low-noise log-amplification]
  \label{def:low_noise_budget}
  Set \(\gambar_j=\max\{1,\gamma_j\}\). On \(G_K\in\Gsched\), define
  \(\Lambda_K=\sum_{j=\Khi}^{K-1}\log\gambar_j\).
\end{definition}

Under \Cref{ass:wellposed}, \(\Lambda_K<\infty\) on each schedule, and
\(\prod_{i=r}^{s}\gambar_i\leqslant e^{\Lambda_K}\) whenever \(\Khi\leqslant r\leqslant s<K\). Thus
\(e^{\Lambda_K}\) bounds propagation over any part of the low-noise block. Its growth under
refinement enters the global error bound. The estimates below provide alternative ways to control
this same quantity.

\paragraph{Directional bound.}
\label{subsec:low_noise_certificate}

Write \(\ell_j=\sigma_j-\sigma_{j+1}=a_j\sigma_j\) and
\(\vhat_\sigma(x)=(x-\Dhat_\sigma(x))/\sigma\). For a low-noise step, let \((U_j,V_j)\) optimally
couple \((\mu_j,\widehat\mu_j)\), and set \(\mathsf v_j=\vhat_{\sigma_j}(U_j)
-\vhat_{\sigma_j}(V_j)\). Define the positive directional expansion rate and the squared secant
ratio by
\begin{equation}\label{eq:directional_step_constants}
  \Epos[\OT]{j} =\frac{(\mathbb E\ip{V_j-U_j}{ \mathsf v_j})_+}{e_j^2}, \quad \Qsec[\OT]{j}
  =\frac{\mathbb E\|\mathsf v_j\|^2}{e_j^2}.
\end{equation}
Set both to zero when \(e_j=0\). The first quantity measures expansion along the displacement in the
optimal coupling; the second controls the quadratic correction from the finite step.

\begin{proposition}
  \label{prop:law_level_sufficiency}
  For every low-noise step,
  \begin{align}
    \gamma_j^2 &\leqslant1+2\ell_j\Epos[\OT]{j}
    +\ell_j^2\Qsec[\OT]{j}, \label{eq:directional_gamma_chain} \\
    \log\gambar_j &\leqslant\ell_j\Epos[\OT]{j}
    +\tfrac12\ell_j^2\Qsec[\OT]{j}. \label{eq:law_level_stepwise_log}
  \end{align}
  Consequently, for any choice of optimal couplings,
  \begin{equation}\label{eq:law_level_summed_condition}
    \Lambda_K\leqslant\Lambda_K^{\mathrm{dir}} \coloneqq \sum_{j=\Khi}^{K-1} \Bigl(\ell_j\Epos[\OT]{j}
    +\tfrac12\ell_j^2\Qsec[\OT]{j}\Bigr).
  \end{equation}
\end{proposition}

The proof pushes the optimal input coupling through \(\Phihat_j\) and expands
\(\|\Phihat_j(U_j)-\Phihat_j(V_j)\|^2\) using \(\Phihat_j(x)=x-\ell_j\vhat_{\sigma_j}(x)\). Unlike
\(\Lambda_K\), the directional bound uses the transported input coupling, which need not be optimal
for the output laws. Therefore, its cost only upper-bounds the Wasserstein distance defining
\(\gamma_j\).

\paragraph{Worst-case bounds.}
A direct upper bound replaces each \(\gamma_j\) by the Lipschitz constant of the map:
\begin{equation}\label{eq:map_gamma_bound}
  \Lambda_K
  \leqslant \Lambda_K^{\mathrm{map}} \coloneqq \sum_{j=\Khi}^{K-1}
  \log\max\{1,\operatorname{Lip}(\Phihat_j)\}.
\end{equation}
A field-level bound instead uses the global one-sided expansion rate
\begin{equation}\label{eq:global_one_sided_expansion_rate}
  \Los{j} =\sup_{x\ne y} \frac{-\ip{x-y}{ \vhat_{\sigma_j}(x)-\vhat_{\sigma_j}(y)}} {\|x-y\|^2}.
\end{equation}
For continuously differentiable fields, this equals \(\sup_x\lambda_{\max}(
-\operatorname{Sym}\nabla\vhat_{\sigma_j}(x))\). Replacing the directional quantities by their
global bounds gives
\begin{equation}\label{eq:field_level_summed_condition}
  \Lambda_K^{\mathrm{field}} \coloneqq \sum_{j=\Khi}^{K-1} \Bigl(\ell_j(\Los{j})_+
  +\tfrac12\ell_j^2 \operatorname{Lip}(\vhat_{\sigma_j})^2\Bigr).
\end{equation}
Both worst-case bounds maximize over space, rather than depend on the particular laws
\((\mu_j,\widehat\mu_j)\). If every low-noise field is globally Lipschitz, then
\begin{equation}\label{eq:law_map_field_budget_hierarchy}
  \Lambda_K \leqslant\Lambda_K^{\mathrm{dir}} \leqslant\Lambda_K^{\mathrm{field}},\qquad \Lambda_K
  \leqslant\Lambda_K^{\mathrm{map}} \leqslant\Lambda_K^{\mathrm{field}},
\end{equation}
see \Cref{cor:field_level_certificate}. The directional and map majorants are not ordered in
general: each one can serve to bound \(\Lambda_K\) uniformly under refinement. A sufficient
condition, through the field majorant, is given in \Cref{rem:integrable_rate}.

\paragraph{Strict separation from worst-case stability.}

The worst-case bounds can be arbitrarily larger than the realized amplification. For the following
proposition only, we allow an alternative initialization \(\widehat\mu_0=\nu_0^\varepsilon\), fixed
across refinements for each \(\varepsilon\), instead of the Gaussian initialization in
\Cref{sec:one_step_error}.

\begin{proposition}
  \label{prop:strict_law_map_separation}
  Let \(d=1\), \(q=\mathcal N(0,\sigmad^2)\) with \(\sigmad>0\), and \(\sighi>\overline\sigma\).
  There is \(C\), depending only on \(\sigmad,\underline\sigma,\overline\sigma\), such that for each
  \(\varepsilon\in(0,2\sigmad)\) there are a smooth denoiser family \(\Dhat^\varepsilon\) satisfying
  \Cref{ass:wellposed} and \cref{eq:edm_learned_denoiser_step}, and an initialization
  \(\nu_0^\varepsilon\), with the following properties:
  \begin{enumerate}[
    label=\textup{(\alph*)}, leftmargin=1.6em, itemsep=0pt, topsep=0pt]
    \item
    \(\Wtwo(\nu_0^\varepsilon,\mu_{\overline\sigma})
    \leqslant C\varepsilon\) and
    \(\Rres_{j,K}^\varepsilon\leqslant C\varepsilon\);
    \item
    \(\gamma_{j,K}^\varepsilon\leqslant1\), hence \(\Lambda_K^\varepsilon=0\) for every \(K\);
    \item as \(K\) tends to infinity
    along \(\mathscr K\), we have \(\lim_K \Lambda_K^{
    \mathrm{map},\varepsilon} = \lim_K
    \Lambda_K^{\mathrm{field},\varepsilon}
    =\frac{\overline\sigma - \underline\sigma}
    {\varepsilon}+O(1)\), where \(O(1)\) is independent
    of \(\varepsilon\).
  \end{enumerate}
\end{proposition}

The construction combines an \(O(\varepsilon)\) oscillation with slope \(O(\varepsilon^{-1})\) and
an \(O(\varepsilon)\) translation of the initial law. The learned steps do not amplify the realized
discrepancy, although their worst-case majorants diverge as \(\varepsilon\to0\). This is a
separation across denoisers: for each fixed \(\varepsilon\), both worst-case majorants have finite
limits under refinement. The role of initialization is essential to this example;
\Cref{subsec:evidence_hierarchy} examines the default Gaussian one.

\section{Global Error Bound}
\label{sec:two_regime_propagation}
\label{sec:global_bound}

We now combine the local error bound with the propagation controls of
\Cref{sec:two_regime_stability}. Proofs and explicit constants are provided in
\Cref{app:global_bound}. In the high-noise block, the kernel \(\Pi_{r,k}=\prod_{i=r}^{k-1}(1-\bhi
a_i)\) damps errors introduced at earlier steps. In the low-noise block, every partial product of
amplification factors is bounded by \(e^{\Lambda_K}\). Empty products are set to one.

To express the accumulated discretization bound, define, for \(0<s_1\leqslant s_0\) and
\(b\geqslant0\), the scale function
\begin{equation}\label{eq:def_Xi}
  \Xi_\rho(b;s_1,s_0) =\int_{s_1}^{s_0} \Bigl(\frac{s_1}{s}\Bigr)^b s^{-\rho}\,\dd s.
\end{equation}
The factor \((s_1/s)^b\) represents damping, while \(s^{-\rho}\) comes from summing the local biases
on the clock-uniform grid. Thus \(\Xi_\rho\) is nonincreasing in \(b\).

\begin{theorem}[Global bound]
  \label{thm:global_edm_wasserstein}
  Let \Cref{ass:wellposed,ass:uniform_residual_accuracy,ass:high_noise_contraction} hold and
  \(\rho\in(0,1]\). Fix \(\thetahi,\thetalo\in(0,1)\). There are \(C_\star<\infty\), depending only
  on \(\rho\) and \(\max(\thetahi,\thetalo)\), and an effective margin \(\beff\in(0,\bhi)\),
  depending only on \(\rho,\thetahi,\bhi\), with \(\beff\to\bhi\) as \(\thetahi\to0\), such that for
  every \(G_K\in\Gsched\) with \(h_K\leqslant
  \min\{\thetahi\sighi^\rho,\thetalo
  \underline{\sigma}^\rho\}\),
  \ifarxivlayout
  \begin{equation}
    \Wtwo(\widehat\mu_K,q) \leqslant e^{\Lambda_K} \bigl[\Pi_{0,\Khi}e_0 +\mathrm{Disc}_K
    +\mathrm{Learn}_K\bigr] +\sigma_K\sqrt d, \label{eq:transport_global_bound}
  \end{equation}
  where
  \begin{align*}
    \mathrm{Disc}_K &=C_\star\Bd[1]\,h_K \bigl[\Xi_\rho(0;\sigma_K,\sigma_{\Khi})
    +\Xi_\rho(\beff; \sigma_{\Khi}, \sigma_0)\bigr], \\
    \mathrm{Learn}_K &=\sqrt d\,\Cres\,\sigmad \Bigl[\frac{1-\Pi_{0,\Khi}}{\bhi}
    +\log\Bigl(1+\frac{2\sigma_{\Khi}}{\sigmad}\Bigr)\Bigr].
  \end{align*}
  \else
  \begin{align}
    \Wtwo(\widehat\mu_K,q) \leqslant{} e^{\Lambda_K} \bigl[\Pi_{0,\Khi}e_0
    +\mathrm{Disc}_K&+\mathrm{Learn}_K\bigr] \nonumber \\
    &+\sigma_K\sqrt d, \label{eq:transport_global_bound}
  \end{align}
  where
  \begin{align*}
    \mathrm{Disc}_K &=C_\star\Bd[1]\,h_K \bigl[\Xi_\rho(0;\sigma_K,\sigma_{\Khi}) \\
    &\qquad\qquad\qquad\qquad\qquad +\Xi_\rho(\beff; \sigma_{\Khi}, \sigma_0)\bigr], \\
    \mathrm{Learn}_K &=\sqrt d\,\Cres\,\sigmad \Bigl[\frac{1-\Pi_{0,\Khi}}{\bhi}
    +\log\Bigl(1+\frac{2\sigma_{\Khi}}{\sigmad}\Bigr)\Bigr].
  \end{align*}
  \fi
\end{theorem}

The four contributions are initialization, discretization (undamped low-noise and damped high-noise
biases), learning, and terminal smoothing bias.

\paragraph{Role of high-noise contraction.}
At fixed endpoints, both scale functions are uniformly bounded under refinement. The discretization
contribution is therefore \(O(h_Ke^{\Lambda_K})\). High-noise contraction reduces its prefactor by
damping early biases. For \(\beff>1-\rho\),
\begin{equation*}
  \Xi_\rho(\beff;\sigma,\sigma_0) \leqslant \frac{\sigma^{1-\rho}}{\beff+\rho-1},
\end{equation*}
so the high-noise scale function is bounded independently of \(\sigma_0\) at fixed
\(\sigma_{\Khi}\). For \(\rho=1/7\), this requires \(\beff>6/7\), which follows from \(\bhi>6/7\) on
sufficiently fine meshes; see \Cref{rem:three_regimes}.

Since \(0<\Pi_{0,\Khi}\leqslant1\), \(\mathrm{Learn}_K\) also admits an upper bound independent of
\(\sigma_0\) at fixed \(\sigma_{\Khi}\). Under the light-tail assumptions, the canonical
initialization error is \(O(\sigma_0^{-1})\); see \Cref{prop:initialization_error}.

\paragraph{Rates under refinement.}
With fixed endpoints and clock, \(h_K=O(K^{-1})\). The resulting discretization rate therefore
depends on the growth of \(\Lambda_K\).

\begin{corollary}
  \label{cor:budget_refinement_rates}
  In the setting of \Cref{thm:global_edm_wasserstein}, suppose
  \begin{equation}\label{eq:logarithmic_budget}
    \exists \budexp\in[0,1)\text{ s.t. } \sup_{G_K\in\Gsched}
    \bigl(\Lambda_K-\budexp\log K\bigr)<\infty.
  \end{equation}
  Then the discretization term in \eqref{eq:transport_global_bound} is \(O(K^{-(1-\budexp)})\). For
  the oracle denoiser with exact initialization \(\widehat\mu_0 = \mu_{\overline\sigma}\), we have
  \(\Wtwo(\widehat\mu_K,q) \leqslant \underline\sigma\sqrt d + O(K^{-(1-\budexp)})\).
\end{corollary}

A uniform bound on \(\Lambda_K\) gives \(\budexp=0\) and first-order decay of the discretization
contribution. Refinement alone does not remove learning error, initialization mismatch, or the
terminal smoothing bias. The envelope \eqref{eq:logarithmic_budget} is an asymptotic assumption, not
a consequence of a finite-range fit. Even when \(\Lambda_K\) is uniformly bounded, first-order
behavior may emerge only beyond practical step counts (\Cref{subsec:evidence_law_level}).

\section{Numerical Experiments}
\label{sec:numerical_evidence}

We examine whether realized amplification explains the observed discretization error, how closely
its majorants track it, and whether related stability mechanisms appear in an image model. On a
one-dimensional Gaussian mixture, we evaluate the law-level quantities by quadrature. On CIFAR--10,
we study local diagnostics of the learned field. Additional results and numerical checks are in
\Cref{app:numerics}.

\begin{figure*}[t]
  \centering
  \includegraphics[width=\textwidth]{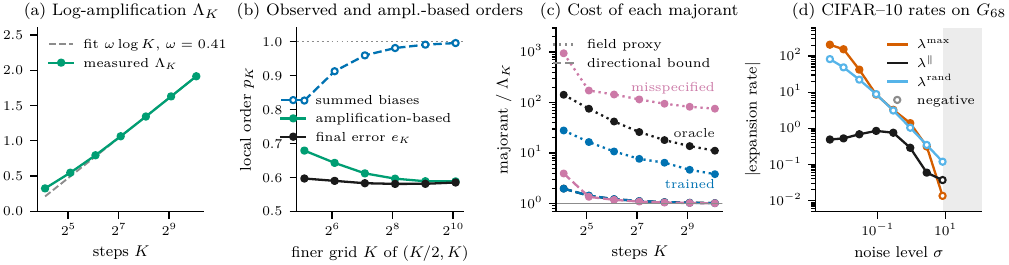}
  \caption{
  \small
  1D mixture (a--c) and pretrained CIFAR--10 EDM model (d). \textbf{(a)} Oracle with exact
  initialization: measured \(\Lambda_K\) and its fit \(\budexp\log K\). \textbf{(b)} Local orders
  \eqref{eq:empirical_order_main} of the final error \(e_K\) and of the summed local biases, against
  the amplification-based order \(1-(\Lambda_K-\Lambda_{K/2})/\log2\) computed from (a): the biases
  converge at first order, the error does not. \textbf{(c)} Directional majorant (dashed) and
  finite-domain field proxy (dotted) divided by \(\Lambda_K\), from the canonical Gaussian
  initialization; trained networks as the median over three seeds. \textbf{(d)} Energy-weighted
  rates \eqref{eq:segment_rate_main}, maximized over nine nodes \(t\), on the coarse--fine segments
  of \(G_{\auditSegK}\): worst direction \(\lambda^{\max}\), displacement \(\lambda^{\parallel}\)
  and random direction \(\lambda^{\mathrm{rand}}\), in absolute value; open markers are negative
  rates, and shading marks where the 90th-percentile empirical margin \eqref{eq:probe_margin_main}
  is positive.
  \vspace*{-14pt}
  }
  \label{fig:numerical_summary_row}
\end{figure*}

\subsection{Amplification and the Finite-Grid Rate}
\label{subsec:evidence_law_level}

We use a centered mixture of four narrow Gaussians (\(\sigmad=0.5\)), with the oracle denoiser and
exact initialization \(\mu_{\overline\sigma}\). The error relative to the exact terminal law is
therefore due only to discretization. Each law is represented by \(M=65{,}536\) quantiles. Matching
quantiles gives the optimal coupling in one dimension, allowing us to evaluate \(e_{j,K}\),
\(\gamma_{j,K}\) and \(\Lambda_K\) on grids ranging from \(K=\resKmin\) to \(K=\resKmax\).

\paragraph{Local accuracy and final error.}
For a quantity \(C_K\), define its order between successive grids by
\begin{equation}\label{eq:empirical_order_main}
  \textstyle p_K(C) =\frac{\log(C_{K/2}/C_K)}{\log 2}.
\end{equation}
The summed local biases approach first-order decay, whereas the final error has orders between
\(\resErrorOrderMin\) and \(\resErrorOrderMax\). Over the same range, \(\Lambda_K\) grows
approximately logarithmically, with fitted slope \(\budexp\simeq\resOmega\)
(\Cref{fig:numerical_summary_row}a).

An error proportional to \(K^{-1}e^{\Lambda_K}\) would have order \(1-(\Lambda_K-\Lambda_{K/2})/\log
2\). This amplification-based order agrees with the observed one on the two finest grid pairs
(\Cref{fig:numerical_summary_row}b). Thus measured amplification accounts for the slower decay on these
grids.

\paragraph{Interpretation and scope.}
This agreement is a posteriori: \(\Lambda_K\) is computed from the transported laws, rather than
predicted from the denoiser. Moreover, the slower decay is preasymptotic. For this oracle,
\(\Lambda_K\) is uniformly bounded under refinement (\Cref{rem:toy_uniform_budget}), so
\Cref{cor:budget_refinement_rates} gives an eventual first-order bound.

The measured local biases also provide a quantitative check of the propagation analysis. With exact
initialization and oracle denoising, \cref{eq:law_level_one_step} unrolls into
\begin{equation}\label{eq:recursion_bound_main}
  e_{K,K} \leqslant \sum_{j=0}^{K-1}\widetilde\delta_{j,K} \prod_{k=j+1}^{K-1}\gamma_{k,K},
\end{equation}
where the local biases \(\widetilde\delta_{j,K}\) are evaluated by quadrature. The resulting bound
overestimates the error by \(\resBoundExcessMin\)--\(\resBoundExcessMax\)\%. Additional comparisons
with the universal local bound and the high-noise initialization control are reported in
\Cref{app:numerics}.

\subsection{Directional Bounds Track Realized
Amplification}
\label{subsec:evidence_hierarchy}

We compare \(\Lambda_K\) with \(\Lambda_K^{\mathrm{dir}}\) and a finite-domain proxy of
\(\Lambda_K^{\mathrm{field}}\). We use the canonical Gaussian initialization and five denoisers: the
oracle, three trained EDM-preconditioned multilayer perceptrons, and a misspecified posterior mean.
We do not evaluate \(\Lambda_K^{\mathrm{map}}\).

At \(K=\resKmax\), the directional majorant exceeds \(\Lambda_K\) by at most \(\resDirExcessMax\%\)
for every denoiser. For the oracle, the field proxy is much larger: \(\resFieldRatioCoarse\) times
\(\Lambda_K\) on the coarsest grid and \(\resFieldRatioFine\) times on the finest
(\Cref{fig:numerical_summary_row}c). These differences become more pronounced after exponentiation in
the global error bound.

The oracle illustrates why the directional bound is tighter. Its velocity expands in narrow bands
between the modes, while the discrepancy between the laws lies mostly where the predictor contracts.
The field supremum includes these expanding regions regardless of where the discrepancy lies; the
directional bound weights them through the actual coupling.

\subsection{Stability Mechanisms in an Image Model}

We use the class-conditional CIFAR--10 EDM checkpoint of \textcite{karras_elucidating_2022} on
nested grids \(G_K\), with \(\sigma\in[0.002,80]\), \(\rhoEDM=7\), and \(\auditSyncTrajectories\)
latent variables. The diagnostics use Jacobian--vector products and the checkpoint's \(\sigma_{\rm
data}=\auditSigmaData\) in place of \(\sigmad\) (\Cref{subsec:numerical_definitions}).

\paragraph{High noise.}
\label{subsec:evidence_damping}

We probe the EDM contraction margin by replacing the global constant \(\LF{j}\) in
\cref{eq:def_b_k_edm} with an empirical Jacobian norm. Let \(\SFest[p](\sigma)\) be the
\(p\)-quantile of power-iteration estimates of \(\|\nabla\Fhat_\sigma\|_\op\) at \(128\) noised test
images, and set
\begin{equation}\label{eq:probe_margin_main}
  \bprobe[p](\sigma) =1-\alpha(\sigma) -\beta(\sigma)\cin(\sigma)\SFest[p](\sigma).
\end{equation}
The empirical margin \(\bprobe[0.9]\) becomes positive at \(\sigma=\auditBcrossQninety\) and exceeds
\(6/7\), the supercritical threshold in \Cref{rem:three_regimes}, above
\(\sigma=\auditBsupercritQninety\) (\Cref{fig:numerical_summary_row}d). These probes suggest that the
EDM certificate is most informative near the high-noise end of the schedule. They do not bound the
global Lipschitz constant or establish a valid contraction threshold.

\paragraph{Low noise.}
\label{subsec:evidence_directional_cifar}

To examine directional expansion, we pair trajectories on \(G_K\) and \(G_{2K}\) from the same
latent variable. At a common noise level, let \(x\) and \(y\) be the coarse and fine states and set
\(\zeta=y-x\). For a unit direction \(u\), define
\begin{equation}\label{eq:segment_rate_main}
  \lambda^u(t) =-u^\top\operatorname{Sym} \nabla\vhat_{\sigma_j}(x+t\zeta)\,u, \qquad t\in[0,1].
\end{equation}
We compare rates along the displacement (\(\lambda^\parallel\)), in the most expanding direction
(\(\lambda^{\max}\)), and in a random direction (\(\lambda^{\rm rand}\)); aggregation details are in
\Cref{app:numerics}.

At \(\sigma=\auditSegSigmaLow\), the random direction contracts (\(\lambda^{\rm
rand}=\auditSegRandLow\)), while the displacement expands (\(\lambda^\parallel=+\auditSegParLow\)).
The worst-direction rate is \(\auditSegWorstOverParLow\) times the displacement rate
(\Cref{fig:numerical_summary_row}d). Thus random directions can miss expansion along the discrepancy,
while worst directions can substantially overstate it.

These are local diagnostics along synchronously coupled coarse and fine trajectories. They
illustrate the importance of direction, but do not estimate \(\Lambda_K\), which compares the exact
and sampler laws through optimal transport. Refinement results and sensitivity to individual
trajectories are given in \Cref{subsec:analytic_controls}.

\section{Discussion and Limitations}

Gaussian smoothing gives a universal local discretization bound for the first-order EDM predictor,
while the global estimate separates this error from its amplification by the learned dynamics. The
mixture experiment shows why this distinction matters on finite grids: accumulated local biases and
final error can decrease at different rates.

The remaining difficulty is to control realized amplification from accessible properties of the
trained model. Our bound is a posteriori, the CIFAR--10 diagnostics do not certify global stability,
and the \(O(d^{3/2})\) local constant limits quantitative use in high dimension. A central next step
is a computable propagation certificate that retains the advantage of distribution-dependent
control.

\section*{Acknowledgments}
This project has received funding from the European Research Council (ERC) under the European
Union’s Horizon Europe research and innovation programme (grant agreement No.~101201229) and from
ANR through the Hi!~PARIS Cluster 2030 project (ANR-23-IACL-0005) under the France 2030 investment
plan.

\section*{AI Use Statement}
We made use of generative AI tools (large language models and coding assistants) throughout this
project: to develop and check proofs, to write and test the code of the numerical experiments, and
to draft and edit the text. The research questions, the modelling choices and assumptions, and the
overall direction of the work are our own. We carefully read every result and proof, corrected or
rewrote them where needed. The figures, the generated tables and the measured values quoted in the
main text are generated automatically from the saved experiment outputs available in the public
code repository, rather than typed by hand. We take full responsibility for the content of this
paper.

\printbibliography

\clearpage
\appendix
\raggedbottom %
\etocdepthtag.toc{appendix}

\begingroup
\renewcommand{\contentsname}{Appendix Contents}
\etocsetnexttocdepth{subsection}
\etocsettagdepth{main}{none}
\etocsettagdepth{appendix}{subsection}
\tableofcontents
\endgroup

\clearpage

\Cref{app:notation} collects the notation. \Crefrange{app:setting}{app:numerics} follow
\Crefrange{sec:denoising_path_predictor}{sec:numerical_evidence} and contain all proofs, with the
comparisons with prior work placed next to the results they concern;
\Cref{app:related_work_extended} compares our analysis in detail with the closest Wasserstein
analyses.

\section{Notation Index}
\label{app:notation}

This appendix collects the general notation and the recurring problem-specific symbols; each
problem-specific symbol is defined precisely at the place indicated in the last column.

\paragraph{General notation.}
The symbols \(\E\), \(\Law(U)\), \(\Cov(U\mid V)\), \(\mathcal N(m,\Sigma)\) and
\(\operatorname{supp}\lambda\) denote expectation, the law of \(U\), conditional covariance, a
Gaussian law and the support of a measure; \(\Id\) is the identity map and \(\Idd\) the \(d\times
d\) identity matrix, \(\lambda_{\max}(A)\) the largest eigenvalue of a symmetric matrix and
\(\preceq\) the Loewner order. We write \(\ip{x}{y}\) and \(\norm{x}\) for the Euclidean inner
product and norm, \(\norm{A}_\op\) for the operator norm of a matrix,
\(\operatorname{Sym}(A)=\tfrac12(A+A^\top)\) for its symmetric part, \(B(m,r)\) for the Euclidean
ball of center \(m\) and radius \(r\), and \(r_+=\max\{r,0\}\) for the positive part of a scalar.
For a vector-valued map \(f\), \(\nabla f\) is its Jacobian, \(\nabla\cdot f\) its divergence and
\(\operatorname{Lip}(f)\) its global Lipschitz constant. For \(1\le p<\infty\) and a measure
\(\lambda\), \(\norm{f}_{\mathbb L_p(\lambda)}=(\int\norm{f(x)}^p\,\lambda(\dd x))^{1/p}\); the
measure is omitted for a random variable on the underlying probability space, so that
\(\norm{U}_{\mathbb L_p}=(\E\norm{U}^p)^{1/p}\). A subscript on a Landau symbol records allowed
dependence of its implicit constant, as in \(O_\rho(\cdot)\).

Finally, \(\mathcal P_2(\R^d)\) is the set of Borel probability measures with finite second moment,
\(f_\#\lambda\) the pushforward of \(\lambda\) by a measurable map \(f\), and
\begin{equation*}
  \Wtwo^2(\lambda,\lambda')
  = \inf_{\pi\in\mathscr C(\lambda,\lambda')} \int\norm{x-y}^2\,\pi(\dd x,\dd y)
\end{equation*}
the quadratic Wasserstein distance, where \(\mathscr C(\lambda,\lambda')\) is the set of couplings
and \(\mathscr C_{\mathrm{opt}}(\lambda,\lambda')\) the set of optimal ones. \(\mathbb
N=\{0,1,\ldots\}\), and \(\delta_m\) is the Dirac mass at \(m\). For a scalar-valued map \(f\),
\(\nabla f\) and \(\nabla^2f\) are its gradient and Hessian. For matrices, \(A^\top\), \(\Tr A\) and
\(\norm{A}_F\) are the transpose, trace and Frobenius norm, and \(\lambda_{\min}(A)\) the smallest
eigenvalue of a symmetric matrix. The set \(\mathbb S_+^d\) is the cone of symmetric
positive-semidefinite \(d\times d\) matrices, \(A\preceq B\) means \(B-A\in\mathbb S_+^d\), and
\(A>0\) means that a symmetric matrix is positive definite.

\paragraph{Decorations and indices.}
A hat marks a learned object (\(\Dhat\), \(\Phihat\), \(\vhat\), \(\Fhat\), \(\widehat\mu_j\)), and
the exact counterpart is undecorated. A tilde marks the value returned by a numerical routine or
measured by quadrature for the object written without it (\Cref{app:numerical_registry}). The
superscripts \(\OT\) and \(\sync\) name the coupling along which a quantity is evaluated: an optimal
coupling of the exact and sampler laws, or the synchronous coupling of two nested grids driven by
the same latent variable. The index \(j\) denotes a schedule level or the step starting there, and
\(K\) the number of steps; the index \(K\) is written only when several grids are compared.

\begin{center}
  \small
  \begin{tabular}{@{}>{$}l<{$}>{\raggedright}p{0.55\linewidth}>{\raggedright\arraybackslash}p{0.16\linewidth}@{}}
\toprule
\text{symbol} & \text{meaning} & \text{defined in} \\
\midrule
q,\ X_0,\ Z,\ X_\sigma,\ \mu_\sigma
  & data law, clean variable, standard Gaussian noise, noisy variable, and noisy law
  & \Cref{sec:denoising_path_predictor} \\
p_\sigma,\ \pi_{x,\sigma},\ D_\sigma,\ s_\sigma,\ v_\sigma,\ \Psi_{\sigma,\sigma_0}
  & noisy density, posterior kernel, exact denoiser, score, probability-flow velocity, and exact
    flow map
  & \cref{eq:velocity,eq:Psi} \\
\underline\sigma,\ \overline\sigma
  & terminal and initial noise levels, fixed under refinement
  & \Cref{sec:denoising_path_predictor} \\
\sigma_j,\ a_j,\ \Psi_j,\ \Phi_j,\ \Phihat_j
  & noise level, relaxation coefficient, exact flow step, oracle predictor, and learned
    predictor
  & \cref{eq:oracle_edm_predictor,eq:learned_edm_predictor} \\
\ell_j
  & noise step \(\sigma_j-\sigma_{j+1}=a_j\sigma_j\)
  & \Cref{subsec:low_noise_certificate} \\
\mu_j,\ \widehat\mu_0,\ \widehat\mu_j,\ e_j
  & exact law \(\mu_{\sigma_j}\), initial sampling law \(\mathcal N(0,\overline\sigma^2\Idd)\),
    sampler law, and its \(\Wtwo\)-error against \(\mu_j\)
  & \cref{eq:learned_sampler_laws_errors} \\
\eta,\ \rho,\ \rhoEDM,\ w_\eta
  & numerical clock \(\eta=\sigma^\rho\), clock exponent, EDM exponent \(\rhoEDM=1/\rho\), and
    clock velocity
  & \Cref{sec:discretization_error} \\
\underline\eta,\ \overline\eta,\ h_K,\ G_K,\ \Gsched
  & clock endpoints, clock mesh, schedule, and refinement family
  & \Cref{def:schedule_family} \\
\theta_K
  & relative mesh \(h_K/\underline\eta\) of \(G_K\)
  & \cref{eq:schedule_family_relative_mesh} \\
\delta_j,\ \Delta_j
  & local discretization bias and its universal bound
  & \cref{eq:def_local_defect_learned_map,eq:def_Delta_k_curvature_bound} \\
\mathcal A_\rho,\ B_{d,\rho}
  & clock acceleration and acceleration constant; \(B_{d,1}\) is the constant at \(\rho=1\)
    that enters \(\Delta_j\)
  & \cref{eq:def_A_eta}, \Cref{thm:global_curvature_bound} \\
\gamma_j,\ \gambar_j,\ \Lambda_K
  & realized amplification factor, its truncation from below at one, and the cumulative low-noise
    log-amplification
  & \cref{eq:def_raw_amplification}, \Cref{def:low_noise_budget} \\
\budexp
  & growth exponent of a logarithmic envelope \(\Lambda_K\le\budexp\log K+O(1)\)
  & \cref{eq:logarithmic_budget} \\
\sigmad
  & root-mean-square scale of the centered target, \(\sigmad^2=\E\norm{X_0}^2/d\)
  & \Cref{sec:preconditioning_residual} \\
\alpha_j,\ \beta_j,\ c_j
  & EDM coefficients \(\alpha(\sigma_j)\), \(\beta(\sigma_j)\), \(\cin(\sigma_j)\), that is
    \(\cskip,\cout,\cin\)
  & \cref{eq:edm_coefficients} \\
F_\sigma,\ \Fhat_j,\ \Rres_j,\ \Cres
  & exact and learned normalized residual branches, residual error, and its uniform constant
  & \Cref{sec:preconditioning_residual} \\
\sighi,\ \Khi
  & high-noise threshold and split index
  & \cref{eq:def_split_index} \\
\LF{j},\ \Lhi
  & Lipschitz constant of \(\Fhat_j\) in normalized input, and its envelope on the high-noise block
  & \Cref{sec:preconditioning_residual}, \Cref{app:global_bound} \\
b_j,\ \bhi
  & EDM stability margin and its uniform lower bound
  & \cref{eq:def_b_k_edm}, \Cref{ass:high_noise_contraction} \\
\Pi_{r,k},\ \Xi_\rho
  & high-noise contraction kernel \(\prod_{i=r}^{k-1}(1-\bhi a_i)\) and scale function
  & \Cref{sec:global_bound}, \cref{eq:def_Xi} \\
\kappa,\ \thetahi,\ \beff
  & finite-mesh damping factor, relative-mesh bound on the high-noise block, and effective
    exponent \(\bhi\kappa(\thetahi)\)
  & \cref{eq:def_kappa_mesh_edm_predictor}, \Cref{prop:high_noise_grid} \\
\thetalo,\ C_{\rho,\theta},\ C_\star,\ \Upsilon
  & relative-mesh bound on the low-noise block, constants of the bias sums, and learning scale
  & \Cref{thm:global_edm_wasserstein,prop:terminal_transport_grid_bound,lem:damped_defect_sum} \\
\mathrm{Disc}_K,\ \mathrm{Learn}_K
  & discretization and learning terms of the global bound
  & \Cref{thm:global_edm_wasserstein} \\
\zdisp{\pi},\ \Erate[\pi]{j},\ \Epos[\pi]{j},\ \Qsec[\pi]{j}
  & displacement, signed and positive directional rates, and squared secant ratio along a coupling
    \(\pi\)
  & \cref{eq:directional_step_constants} \\
\Amp[\pi]{j,K}
  & root-mean-square amplification of the predictor along a coupling \(\pi\)
  & \cref{eq:toy_rms_amplification,eq:transport_sync_secant_identity} \\
\Los{j}
  & global one-sided expansion rate of \(\vhat_{\sigma_j}\)
  & \cref{eq:global_one_sided_expansion_rate} \\
\Lambda_K^{\mathrm{dir}},\ \Lambda_K^{\mathrm{map}},\ \Lambda_K^{\mathrm{field}}
  & transport-directional, map-level and field-level majorants of \(\Lambda_K\)
  & \cref{eq:law_level_summed_condition,eq:map_gamma_bound,eq:field_level_summed_condition} \\
\ell_{\max,K}
  & largest noise step of \(G_K\)
  & \Cref{rem:integrable_rate} \\
\bottomrule
  \end{tabular}
\end{center}

Notation specific to a numerical estimator is introduced only in
\Cref{sec:numerical_evidence,app:numerical_registry}.

\section{Setting and One-Step Decomposition}
\label{app:setting}
\label{app:main_text_complements}

This appendix complements \Cref{sec:denoising_path_predictor}. It constructs the exact probability
flow for every Borel data law, records the clock convention and the refinement family, relates the
EDM predictor to explicit Euler in the numerical clock, and proves the one-step decomposition, which
it compares with the two other possible orderings.

\subsection{Posterior Identities and the Exact \texorpdfstring{\(\sigma\)-Flow}{sigma-Flow}}

We record here the posterior identities and the exact probability-flow construction used in
\Cref{sec:denoising_path_predictor}, and prove the flow properties stated there. All statements are
made at positive noise, and all flow statements on compact subintervals of \((0,\infty)\); no
regularity at \(\sigma=0\) and no moment of \(q\) is assumed. The material is classical in
substance, but the paper needs it in a form that holds for every Borel data law and for integrands
that are unbounded in \(X_0\), which is why it is recorded with proofs. It is organized in three
blocks.
\begin{itemize}[nosep,leftmargin=1.4em]
  \item \Cref{subsec:app_posterior_kernel}: Gaussian smoothing and the posterior kernel. The identity
  \eqref{eq:posterior_kernel_identity_joint} for unbounded, jointly dependent integrands and the
  finiteness of all posterior moments are what \Cref{lem:gaussian_posterior_differentiation} uses to
  differentiate posterior averages in \Cref{app:curvature}, and what makes the pointwise denoiser
  \(D_\sigma\) well defined without a moment of \(q\).
  \item \Cref{subsec:app_tweedie}: Tweedie's formula, which identifies the probability-flow velocity
  with the denoiser residual and the score, and the linear growth of the denoiser on positive-noise
  intervals.
  \item \Cref{subsec:app_exact_flow}: the weak transport identity and the construction, by
  characteristics, of the exact flow maps \(\Psi_{\sigma,\sigma_0}\) of \cref{eq:Psi} that push
  \(\mu_{\sigma_0}\) to \(\mu_\sigma\), used in the one-step decomposition.
\end{itemize}

\paragraph{The probability flow.}
The velocity \cref{eq:velocity} drives the increasing-noise curve \((\mu_\sigma)_{\sigma>0}\)
through the continuity equation \(\partial_\sigma\mu_\sigma+\nabla\cdot(v_\sigma\mu_\sigma)=0\) on
\((0,\infty)\times\R^d\), with \(q\) as its narrow initial trace (\Cref{rem:initial_trace_sigma}).
The pointwise posterior construction, the Tweedie identity (\Cref{lem:tweedie}), the weak form of
the continuity equation (\Cref{prop:continuity_sigma_clock}), and the flow properties stated in
\Cref{sec:denoising_path_predictor} (\Cref{prop:exact_flow_sigma_clock}) are established below.

\subsubsection{Gaussian Smoothing and Posterior Moment Identities}
\label{subsec:app_posterior_kernel}
Let \(q\) be a Borel probability measure on \(\R^d\). Let \(X_0\sim q\), let \(Z\sim\mathcal
N(0,\Idd)\), and assume that \(X_0\) and \(Z\) are independent. For \(\sigma>0\), set
\begin{equation*}
  X_\sigma=X_0+\sigma Z, \qquad \mu_\sigma=\Law(X_\sigma).
\end{equation*}
Then \(\mu_\sigma\) is absolutely continuous with respect to Lebesgue measure, with density
\begin{equation}\label{eq:app_density_sigma}
  p_\sigma(x) = \int_{\R^d}\phi_\sigma(x-y)\,q(\dd y), \qquad \phi_\sigma(x)
  = (2\pi\sigma^2)^{-d/2} \exp\left(-\frac{\norm{x}^2}{2\sigma^2}\right).
\end{equation}

\begin{lemma}[Gaussian smoothing and posterior kernel]
  \label{lem:gaussian_smoothing}
  In this setting, for every \(\sigma>0\), the density \(p_\sigma\) is strictly positive and
  \begin{equation*}
    (x,\sigma)\longmapsto p_\sigma(x)
  \end{equation*}
  belongs to \(C^\infty(\R^d\times(0,\infty))\). Moreover, for every \(x\in\R^d\) and \(\sigma>0\),
  \begin{equation*}
    \pi_{x,\sigma}(\dd y) = \frac{\phi_\sigma(x-y)}{p_\sigma(x)}\,q(\dd y)
  \end{equation*}
  defines a probability kernel, and this kernel is a regular conditional distribution of \(X_0\)
  given \(X_\sigma\). That is, for every bounded measurable \(f:\R^d\to\R\),
  \begin{equation}\label{eq:posterior_kernel_identity}
    \int_{\R^d} f(y)\,\pi_{x,\sigma}(\dd y) = \E[f(X_0)\mid X_\sigma=x]
  \end{equation}
  for \(\mu_\sigma\)-almost every \(x\). Consequently, writing
  \begin{equation}\label{eq:app_def_posterior_average}
    \langle \varphi\rangle_\sigma(x) = \int_{\R^d}\varphi(x,y)\,\pi_{x,\sigma}(\dd y)
  \end{equation}
  for measurable \(\varphi:\R^d\times\R^d\to\R\), one has
  \begin{equation}\label{eq:posterior_kernel_identity_joint}
    \langle \varphi\rangle_\sigma(X_\sigma) = \E[\varphi(X_\sigma,X_0)\mid X_\sigma]
    \qquad\text{almost surely}
  \end{equation}
  whenever \(\varphi\) is either nonnegative or satisfies \(\varphi(X_\sigma,X_0)\in \mathbb L_1\);
  the integrand may therefore be unbounded and may depend on the conditioning point as well as on
  \(X_0\). Taking \(\varphi(x,y)=f(y)\) shows in particular that \cref{eq:posterior_kernel_identity}
  holds for every measurable \(f\) that is nonnegative or satisfies \(f(X_0)\in \mathbb L_1\). Since
  \(\pi_{x,\sigma}\) is a genuine conditional distribution, the conditional Jensen inequality
  applies to it in the usual way, under the same integrability proviso. Finally, for each fixed
  \((x,\sigma)\), all polynomial moments under \(\pi_{x,\sigma}\) are finite.
\end{lemma}

\begin{proof}
Strict positivity follows from \(\phi_\sigma(x-y)>0\) for all \(x,y\in\R^d\) and \(q(\R^d)=1\).

Let \(\mathcal Q\subset\R^d\times(0,\infty)\) be compact. On \(\mathcal Q\), every derivative
\begin{equation*}
  \partial_\sigma^m\partial_x^\alpha\phi_\sigma(x-y)
\end{equation*}
is a polynomial in \(x-y\) and \(\sigma^{-1}\), multiplied by \(\phi_\sigma(x-y)\). Since \(\sigma\)
is bounded away from zero on \(\mathcal Q\), this derivative is uniformly bounded in
\((x,\sigma,y)\in\mathcal Q\times\R^d\). Differentiation under the integral sign is therefore
justified by dominated convergence, and \(p_\sigma\) is smooth.

The definition of \(\pi_{x,\sigma}\) gives a probability measure because \(p_\sigma(x)\) is the
normalizing constant. To identify the conditional law, let \(f,g\) be bounded measurable. Then
\begin{equation*}
  \E[f(X_0)g(X_\sigma)] = \int_{\R^d}\int_{\R^d} f(y)g(x)\phi_\sigma(x-y)\,\dd x\,q(\dd y).
\end{equation*}
Equivalently,
\begin{equation*}
  \E[f(X_0)g(X_\sigma)]
  = \int_{\R^d} g(x) \left( \int_{\R^d} f(y)\,\pi_{x,\sigma}(\dd y) \right) p_\sigma(x)\,\dd x .
\end{equation*}
This proves the conditional-law identity.

We extend it to \cref{eq:posterior_kernel_identity_joint}. For \(\varphi(x,y)=g(x)h(y)\) with
\(g,h\) bounded measurable, the identity just proved and pulling out the \(X_\sigma\)-measurable
factor \(g(X_\sigma)\) give \(\langle\varphi\rangle_\sigma(X_\sigma)=\E[\varphi(X_\sigma,X_0)\mid
X_\sigma]\) almost surely. Such products generate the product \(\sigma\)-algebra, so the functional
monotone class theorem extends the identity to all bounded measurable \(\varphi\); monotone
convergence applied to \(\varphi\wedge n\) extends it to \(\varphi\ge0\), and splitting
\(\varphi=\varphi^+-\varphi^-\) to \(\varphi(X_\sigma,X_0)\in \mathbb L_1\). The Jensen claim is the
conditional Jensen inequality for the probability measure \(\pi_{x,\sigma}\).

Finally, for every integer \(k\ge0\), the function
\begin{equation*}
  y\longmapsto \norm{y}^k\phi_\sigma(x-y)
\end{equation*}
is bounded on \(\R^d\). Hence its integral against the finite measure \(q\) is finite, which gives
all polynomial moments of \(\pi_{x,\sigma}\).
\end{proof}

The kernel \(\pi_{x,\sigma}\) is defined for every \(x\), not only \(\mu_\sigma\)-almost everywhere.
Using it as the version of the conditional law of \(X_0\) given \(X_\sigma=x\), define for every
\(x\in\R^d\) and \(\sigma>0\)
\begin{equation}\label{eq:app_def_D_N_score}
  D_\sigma(x) = \int_{\R^d} y\,\pi_{x,\sigma}(\dd y), \qquad N_\sigma(x)
  = \int_{\R^d}\frac{x-y}{\sigma}\,\pi_{x,\sigma}(\dd y), \qquad s_\sigma(x)=\nabla\log p_\sigma(x).
\end{equation}
Thus \(D_\sigma\) is the posterior mean denoiser, \(N_\sigma\) is the posterior mean of the injected
noise, and \(s_\sigma\) is the score of the noisy law. When \(X_0\) is integrable, these definitions
coincide with the usual conditional expectations
\begin{equation*}
  D_\sigma(x)=\E[X_0\mid X_\sigma=x], \qquad N_\sigma(x)=\E[Z\mid X_\sigma=x],
\end{equation*}
for the chosen version: applying \cref{eq:posterior_kernel_identity} to \(f(y)=y\) componentwise
makes \(D_\sigma\) a version of \(\E[X_0\mid X_\sigma=\cdot]\), and since
\(N_\sigma(x)=(x-D_\sigma(x))/\sigma\) by \cref{eq:app_def_D_N_score} while
\(Z=(X_\sigma-X_0)/\sigma\), this makes \(N_\sigma(X_\sigma)\) a version of \(\E[Z\mid X_\sigma]\).
The pointwise definitions in \cref{eq:app_def_D_N_score} are useful because they are meaningful for
every Borel input law \(q\) after Gaussian smoothing.

The coupled path \(\sigma\mapsto X_\sigma=X_0+\sigma Z\) has pathwise velocity \(Z\). We define the
corresponding Eulerian velocity by conditioning this pathwise velocity on the current position:
\begin{equation}\label{eq:def_v_sigma}
  v_\sigma(x) = \E[Z\mid X_\sigma=x] = N_\sigma(x).
\end{equation}
This velocity field is canonical for the Gaussian coupling \(X_\sigma=X_0+\sigma Z\). It is not
meant to be the unique velocity field generating the marginal curve \((\mu_\sigma)_{\sigma>0}\).

Conditioning the identity \(X_\sigma=X_0+\sigma Z\) on \(X_\sigma=x\) gives
\begin{equation}\label{eq:D_N_relation_sigma}
  x=D_\sigma(x)+\sigma N_\sigma(x).
\end{equation}
Consequently,
\begin{equation}\label{eq:v_sigma_denoiser_form}
  v_\sigma(x) = N_\sigma(x) = \frac{x-D_\sigma(x)}{\sigma}.
\end{equation}

\subsubsection{Tweedie's Formula and Linear Growth of the Denoiser}
\label{subsec:app_tweedie}

\begin{lemma}[Tweedie identity]
  \label{lem:tweedie}
  Let \(q\) be a Borel probability measure on \(\R^d\), and let the Gaussian smoothing, posterior
  denoiser, score and velocity be defined as in \Cref{lem:gaussian_smoothing} and
  \cref{eq:def_v_sigma}. For every \(x\in\R^d\) and every \(\sigma>0\),
  \begin{equation*}
    s_\sigma(x) = \frac{D_\sigma(x)-x}{\sigma^2} = -\frac{1}{\sigma}N_\sigma(x).
  \end{equation*}
  Consequently, by \cref{eq:v_sigma_denoiser_form}, \(v_\sigma(x)=-\sigma s_\sigma(x)\).
\end{lemma}

\begin{proof}
Differentiating \eqref{eq:app_density_sigma} with respect to \(x\) gives
\begin{equation*}
  \nabla_x p_\sigma(x) = \int_{\R^d} \frac{y-x}{\sigma^2}\phi_\sigma(x-y)\,q(\dd y).
\end{equation*}
Dividing by \(p_\sigma(x)\), we obtain
\begin{equation*}
  s_\sigma(x) = \frac{1}{\sigma^2} \left( \int_{\R^d} y\,\pi_{x,\sigma}(\dd y)-x \right)
  = \frac{D_\sigma(x)-x}{\sigma^2}.
\end{equation*}
The identity \(s_\sigma(x)=-\sigma^{-1}N_\sigma(x)\) follows from \cref{eq:D_N_relation_sigma}, and
the velocity representation then follows from \cref{eq:v_sigma_denoiser_form}.
\end{proof}

The identities above show that a single population vector field can be written in four equivalent
coordinates: as the conditional velocity of the Gaussian coupling, as the posterior noise predictor,
as the residual of the posterior mean denoiser, and as the score. The conditional-velocity
representation is conceptually important: it shows that the denoising oracle is not only a pointwise
posterior regression object, but also the Eulerian velocity of the heat-regularized curve
\((\mu_\sigma)_{\sigma>0}\) associated with the canonical coupling \(X_\sigma=X_0+\sigma Z\), which
\Cref{subsec:app_exact_flow} makes precise.

The exact-flow construction of \Cref{subsec:app_exact_flow} is carried out only on compact
positive-noise intervals. The following estimate is the only global-in-space input it needs to rule
out finite-time explosion of characteristics.

\begin{lemma}[Linear growth of the denoiser and the velocity on positive-noise intervals]
  \label{lem:posterior_linear_growth}
  Let \(q\) be a Borel probability measure on \(\R^d\) and let \(0<\sigma_-<\sigma_+<\infty\). There
  exists \(C<\infty\), depending only on \(q,\sigma_-,\sigma_+\), such that
  \begin{equation}\label{eq:D_sigma_linear_growth_lemma}
    \norm{D_\sigma(x)} \leqslant C(1+\norm{x}), \qquad x\in\R^d,\quad \sigma\in[\sigma_-,\sigma_+] .
  \end{equation}
  Consequently,
  \begin{equation}\label{eq:v_sigma_linear_growth_lemma}
    \norm{v_\sigma(x)} \leqslant C(1+\norm{x}), \qquad x\in\R^d,\quad \sigma\in[\sigma_-,\sigma_+] .
  \end{equation}
\end{lemma}

\begin{proof}
The Gaussian prefactor cancels in the posterior mean, so
\begin{equation*}
  D_\sigma(x)
  = \frac{ \int_{\R^d} y \exp\left(-\norm{x-y}^2/(2\sigma^2)\right)\,q(\dd y) }{ \int_{\R^d}
  \exp\left(-\norm{x-y}^2/(2\sigma^2)\right)\,q(\dd y) }.
\end{equation*}
Choose \(r_0>0\) such that \(c_0=q(B(0,r_0))>0\), and denote the denominator, an unnormalized
version of \(p_\sigma\), by \(\bar
p_\sigma(x)=\int_{\R^d}\exp\left(-\norm{x-y}^2/(2\sigma^2)\right)\,q(\dd y)\). For \(y\in
B(0,r_0)\), one has \(\norm{x-y}\leqslant\norm{x}+r_0\). Therefore
\begin{equation}\label{eq:denominator_lower_bound_growth}
  \bar p_\sigma(x) \ge c_0 \exp\left(-\frac{(\norm{x}+r_0)^2}{2\sigma^2}\right).
\end{equation}

Let \(r_x=2\norm{x}+r_0\). On the region \(\{\norm{y}\leqslant r_x\}\), the contribution to the
numerator is bounded in norm by \(r_x\bar p_\sigma(x)\). On the complementary region, write
\(u=\norm{y}-r_x\ge0\), so that \(\norm{x-y}\ge\norm{y}-\norm{x}=\norm{x}+r_0+u\). Hence
\begin{equation*}
  \norm{y} \exp\left(-\frac{\norm{x-y}^2}{2\sigma^2}\right)
  \leqslant (r_x+u) \exp\left(-\frac{(\norm{x}+r_0)^2}{2\sigma^2}\right)
  \exp\left(-\frac{u^2}{2\sigma^2}\right).
\end{equation*}
Since \(u\exp(-u^2/(2\sigma^2))\leqslant\sigma\leqslant\sigma_+\), integration over
\(\{\norm{y}>r_x\}\) gives
\begin{equation*}
  \int_{\{\norm{y}>r_x\}} \norm{y} \exp\left(-\frac{\norm{x-y}^2}{2\sigma^2}\right) q(\dd y)
  \leqslant (r_x+\sigma_+) \exp\left(-\frac{(\norm{x}+r_0)^2}{2\sigma^2}\right).
\end{equation*}
Combining this estimate with \eqref{eq:denominator_lower_bound_growth}, we obtain
\begin{equation*}
  \norm{D_\sigma(x)} \leqslant r_x+\frac{r_x+\sigma_+}{c_0} \leqslant C(1+\norm{x}),
  \qquad x\in\R^d,\quad \sigma\in[\sigma_-,\sigma_+] .
\end{equation*}
This proves \eqref{eq:D_sigma_linear_growth_lemma}. Since \(v_\sigma(x)=(x-D_\sigma(x))/\sigma\) and
\(\sigma\ge\sigma_-\), the velocity bound \eqref{eq:v_sigma_linear_growth_lemma} follows.
\end{proof}

\subsubsection{The Continuity Equation and the Exact Flow Maps}
\label{subsec:app_exact_flow}

The velocity \(v_\sigma\) of \Cref{subsec:app_tweedie} drives the heat-regularized curve in two
senses: weakly, through the continuity equation, and pathwise, through the flow of its
characteristics. The second statement produces the exact maps \(\Psi_{\sigma,\sigma_0}\) of
\cref{eq:Psi}.

\begin{proposition}[Weak transport identity]
  \label{prop:continuity_sigma_clock}
  Let \(q\) be a Borel probability measure on \(\R^d\), let \(X_0\sim q\), and let
  \(X_\sigma=X_0+\sigma Z\) with \(Z\sim\mathcal N(0,\Idd)\) independent of \(X_0\). For every
  \(f\in C_c^\infty(\R^d)\), the map \( \sigma\longmapsto \int_{\R^d} f\,\dd\mu_\sigma \) is
  differentiable on \((0,\infty)\), and
  \begin{equation}\label{eq:weak_transport_sigma}
    \frac{\dd}{\dd\sigma} \int_{\R^d} f(x)\,\mu_\sigma(\dd x)
    = \int_{\R^d} \nabla f(x)\cdot v_\sigma(x)\,\mu_\sigma(\dd x).
  \end{equation}
  Equivalently, in density form,
  \begin{equation}\label{eq:continuity_sigma}
    \partial_\sigma p_\sigma + \nabla_x\cdot(p_\sigma v_\sigma) = 0
    \qquad\text{on }(0,\infty)\times\R^d .
  \end{equation}
\end{proposition}

\begin{proof}
Since \(X_\sigma=X_0+\sigma Z\), differentiation under the expectation gives
\begin{equation*}
  \frac{\dd}{\dd\sigma}\E[f(X_\sigma)] = \E[\nabla f(X_\sigma)\cdot Z].
\end{equation*}
The differentiation is justified because \(f\in C_c^\infty(\R^d)\), \(\nabla f\) is bounded, and
\(Z\) is integrable. Since \(\nabla f(X_\sigma)\) is \(X_\sigma\)-measurable,
\begin{equation*}
  \E[\nabla f(X_\sigma)\cdot Z] = \E\left[ \nabla f(X_\sigma)\cdot \E[Z\mid X_\sigma] \right].
\end{equation*}
Using the definition of \(v_\sigma\), we obtain \cref{eq:weak_transport_sigma}.

The density \(p_\sigma\) is smooth by \Cref{lem:gaussian_smoothing}, and
\begin{equation*}
  p_\sigma v_\sigma = -\sigma p_\sigma s_\sigma = -\sigma\nabla p_\sigma
\end{equation*}
is smooth as well. Hence the weak identity is equivalent to the pointwise continuity equation
\eqref{eq:continuity_sigma}.
\end{proof}

\begin{remark}[Initial trace]
  \label{rem:initial_trace_sigma}
  The curve \((\mu_\sigma)_{\sigma>0}\) starts from the data law: \(\mu_\sigma\to q\) narrowly as
  \(\sigma\downarrow0\). Indeed, \(X_\sigma=X_0+\sigma Z\to X_0\) almost surely, so
  \(\E[f(X_\sigma)]\to\E[f(X_0)]\) for every bounded continuous \(f\) by dominated convergence.
\end{remark}

We now prove that the velocity field generates a global flow on every compact positive-noise
interval and that this flow transports the heat-regularized laws.

\begin{proposition}[Exact probability flow on the \texorpdfstring{\(\sigma\)-clock}{sigma-clock}]
  \label{prop:exact_flow_sigma_clock}
  Let \(q\) be a Borel probability measure on \(\R^d\), let the posterior velocity \(v_\sigma\) be
  defined from its Gaussian smoothing as above, and let \(I=[\sigma_-,\sigma_+]\Subset(0,\infty)\).
  The vector field \((\sigma,x)\mapsto v_\sigma(x)\) is the restriction to \(I\times\R^d\) of a
  \(C^\infty\) vector field on \((0,\infty)\times\R^d\). Moreover, for every compact set \(\mathcal
  Q\subset\R^d\),
  \begin{equation}\label{eq:local_lipschitz_v_sigma}
    \sup_{\sigma\in I}\sup_{x\in\mathcal Q} \Bigl( \norm{v_\sigma(x)}
    + \norm{\nabla_x v_\sigma(x)}_\op \Bigr) <\infty ,
  \end{equation}
  and, by \Cref{lem:posterior_linear_growth}, there exists \(C_I<\infty\) such that
  \begin{equation}\label{eq:linear_growth_v_sigma_flow}
    \norm{v_\sigma(x)} \leqslant C_I(1+\norm{x}), \qquad (\sigma,x)\in I\times\R^d .
  \end{equation}

  Consequently, for every \(\sigma_0\in I\) and every \(x\in\R^d\), the ODE
  \begin{equation*}
    \frac{\dd x_r}{\dd r} = v_r(x_r), \qquad r\in I,
  \end{equation*}
  has a unique solution on the whole interval \(I\). Denote its solution map by \(
  \Psi_{\sigma,\sigma_0}:\R^d\to\R^d \), so that
  \begin{equation*}
    \frac{\dd}{\dd\sigma}\Psi_{\sigma,\sigma_0}(x) = v_\sigma\bigl(\Psi_{\sigma,\sigma_0}(x)\bigr),
    \qquad \Psi_{\sigma_0,\sigma_0}(x)=x .
  \end{equation*}
  For every \(\sigma,\sigma_0\in I\), the map \(\Psi_{\sigma,\sigma_0}\) is a
  \(C^\infty\)-diffeomorphism of \(\R^d\), with inverse \(\Psi_{\sigma_0,\sigma}\). In particular,
  for all \(\sigma_0,\sigma_1,\sigma_2\in I\),
  \begin{equation}\label{eq:exact_flow_group_property_sigma}
    \Psi_{\sigma_2,\sigma_0} = \Psi_{\sigma_2,\sigma_1}\circ\Psi_{\sigma_1,\sigma_0},
    \qquad \Psi_{\sigma_0,\sigma_0}=\Id .
  \end{equation}

  The heat-regularized laws are transported by this flow, that is, \cref{eq:pushforward_sigma}
  holds. If \(q\in\mathcal P_2(\R^d)\), then \(\mu_\sigma\in\mathcal P_2(\R^d)\) for every
  \(\sigma>0\), and \((\Psi_{\sigma,\sigma_0})_\#\) maps \(\mathcal P_2(\R^d)\) into itself.
\end{proposition}

\begin{proof}
By \Cref{lem:gaussian_smoothing}, the map \( (\sigma,x)\longmapsto p_\sigma(x)\) is \(C^\infty\) on
\((0,\infty)\times\R^d\), and \(p_\sigma(x)>0\) for every \(\sigma>0\) and every \(x\in\R^d\). Hence
\(s_\sigma(x)=\nabla\log p_\sigma(x)\) is \(C^\infty\) on \((0,\infty)\times\R^d\). By the Tweedie
identity (\Cref{lem:tweedie}), \(v_\sigma(x)=-\sigma s_\sigma(x)\) is \(C^\infty\) on
\((0,\infty)\times\R^d\). Since \(I\times\mathcal Q\) is compact for every compact \(\mathcal
Q\subset\R^d\), continuity of \(v\) and \(\nabla_xv\) gives \eqref{eq:local_lipschitz_v_sigma}. The
global linear-growth estimate \eqref{eq:linear_growth_v_sigma_flow} is
\eqref{eq:v_sigma_linear_growth_lemma} of \Cref{lem:posterior_linear_growth}.

We now construct the flow. Since \(v\) is continuous in \(\sigma\) and locally Lipschitz in \(x\),
uniformly on compact subsets of \(I\times\R^d\), the Picard--Lindelöf theorem gives a unique maximal
solution through every \((\sigma_0,x)\in I\times\R^d\); see \cite[Ch.~II, Sec.~1,
Thm.~1.1]{hartman_ode_1982}. It remains to show that the maximal solution is defined on all of
\(I\).

Let \(y_r\) be such a maximal trajectory, with \(y_{\sigma_0}=x\). By
\eqref{eq:linear_growth_v_sigma_flow},
\begin{equation*}
  \norm{\dot y_r} \leqslant C_I(1+\norm{y_r})
\end{equation*}
as long as the trajectory is defined. Therefore, for every \(r\) in the maximal interval of
existence,
\begin{equation*}
  1+\norm{y_r} \leqslant 1+\norm{x} + C_I \left| \int_{\sigma_0}^{r} (1+\norm{y_u})\,\dd u \right|.
\end{equation*}
Gr\"onwall's lemma gives
\begin{equation*}
  1+\norm{y_r} \leqslant \exp(C_I|r-\sigma_0|) (1+\norm{x})
\end{equation*}
as long as the solution is defined. Hence the trajectory remains bounded on the finite interval
\(I\). By the continuation criterion for ODEs, a maximal solution in \(I\times\R^d\) can fail to
extend only by leaving every compact set; see \cite[Ch.~II, Sec.~3,
Cor.~3.1--3.2]{hartman_ode_1982}. Thus the solution is defined on all of \(I\).

We denote the solution map by \(\Psi_{\sigma,\sigma_0}\). The flow identities
\eqref{eq:exact_flow_group_property_sigma} follow from uniqueness. Since \(v\) is \(C^\infty\), the
classical smooth-dependence theorem for ODEs implies that \(x\mapsto\Psi_{\sigma,\sigma_0}(x)\) is
\(C^\infty\); see \cite[Ch.~V, Sec.~3--4]{hartman_ode_1982}. Moreover,
\begin{equation*}
  \Psi_{\sigma,\sigma_0}\circ\Psi_{\sigma_0,\sigma} = \Psi_{\sigma,\sigma} = \Id,
\end{equation*}
and the same identity with \(\sigma\) and \(\sigma_0\) exchanged gives
\begin{equation*}
  \Psi_{\sigma_0,\sigma}\circ\Psi_{\sigma,\sigma_0} = \Id .
\end{equation*}
Thus \(\Psi_{\sigma,\sigma_0}\) is a \(C^\infty\)-diffeomorphism of \(\R^d\), with inverse
\(\Psi_{\sigma_0,\sigma}\).

It remains to identify the transported laws. The curve \(\sigma\longmapsto\mu_\sigma\) is narrowly
continuous, since \(\mu_\sigma=\Law(X_0+\sigma Z)\) and bounded continuous test functions are
handled by dominated convergence, as in \Cref{rem:initial_trace_sigma}. By
\Cref{prop:continuity_sigma_clock}, the curve \((\mu_\sigma)_{\sigma>0}\) solves
\begin{equation*}
  \partial_\sigma\mu_\sigma+\nabla\cdot(v_\sigma\mu_\sigma)=0
\end{equation*}
in the weak sense.

We verify the hypotheses of the characteristic representation theorem for the continuity equation on
an arbitrary compact subinterval \(J\subset I\). The local regularity condition is satisfied
because, for every compact \(\mathcal Q\subset\R^d\),
\begin{equation*}
  \int_J \sup_{x\in\mathcal Q} \Bigl( \norm{v_\sigma(x)}
  + \norm{\nabla_xv_\sigma(x)}_\op \Bigr) \,\dd\sigma <\infty .
\end{equation*}
Moreover, the velocity is integrable against the curve of laws:
\begin{equation*}
  \int_{\R^d}\norm{v_\sigma(x)}\,\mu_\sigma(\dd x) = \E\left[ \norm{\E[Z\mid X_\sigma]} \right]
  \leqslant \E\norm{Z}.
\end{equation*}
Hence
\begin{equation*}
  \int_J\int_{\R^d} \norm{v_\sigma(x)}\,\mu_\sigma(\dd x)\,\dd\sigma <\infty .
\end{equation*}
The characteristic representation formula for the continuity equation therefore applies; see
\cite[Prop.~8.1.8]{ambrosio_gigli_savare_2008}. It gives, for \(\sigma\ge\sigma_0\),
\(\mu_\sigma=(\Psi_{\sigma,\sigma_0})_\#\mu_{\sigma_0}\). If \(\sigma<\sigma_0\), applying the same
result on \([\sigma,\sigma_0]\) gives \(\mu_{\sigma_0}=(\Psi_{\sigma_0,\sigma})_\#\mu_\sigma\).
Pushing forward by the inverse map \(\Psi_{\sigma,\sigma_0}\) yields again
\(\mu_\sigma=(\Psi_{\sigma,\sigma_0})_\#\mu_{\sigma_0}\). Thus \eqref{eq:pushforward_sigma} holds
for all \(\sigma,\sigma_0\in I\).

Finally, assume \(q\in\mathcal P_2(\R^d)\). Then \(X_\sigma=X_0+\sigma Z\) has finite second moment
for every \(\sigma>0\), so \(\mu_\sigma\in\mathcal P_2(\R^d)\). The Gr\"onwall estimate also gives,
with \(C'_I=\exp\bigl(C_I(\sigma_+-\sigma_-)\bigr)\),
\begin{equation*}
  \norm{\Psi_{\sigma,\sigma_0}(x)} \leqslant C'_I(1+\norm{x}), \qquad \sigma,\sigma_0\in I .
\end{equation*}
Therefore, if \(\lambda\in\mathcal P_2(\R^d)\), then
\begin{equation*}
  \int_{\R^d} \norm{\Psi_{\sigma,\sigma_0}(x)}^2\,\lambda(\dd x)
  \leqslant 2{C'_I}^2 \int_{\R^d} \bigl(1+\norm{x}^2\bigr)\,\lambda(\dd x) <\infty .
\end{equation*}
Thus \((\Psi_{\sigma,\sigma_0})_\#\lambda\in\mathcal P_2(\R^d)\) for every \(\lambda\in\mathcal
P_2(\R^d)\), which is the final assertion.
\end{proof}

\subsection{The Refinement Family}

The family \(\Gsched\) must be infinite: several hypotheses ask for a constant independent of the
mesh, and on any finite collection of grids such a constant exists trivially. Fixed endpoints matter
twice: they make the terminal bias \(\sigma_K\sqrt d\) of
\cref{eq:edm_final_error_with_terminal_bias} the same for every member, and they let a
schedule-uniform per-step constant be summed into a schedule-uniform total. Since the grid is
uniform in \(\eta\), the relative step \(h_K/\eta_j\) is largest where \(\eta_j\) is smallest
(\Cref{fig:clock_grid1}), and the single bound \eqref{eq:schedule_family_relative_mesh} covers every
block of the schedule, with \(\underline\eta=\sighi^\rho\) on the high-noise block and
\(\underline\eta=\underline\sigma^\rho\) on the whole schedule.

\subsection{\texorpdfstring{From \(\eta\)-Euler to the EDM Predictor}{From eta-Euler to the EDM Predictor}}
\label{app:eta_euler_vs_edm_predictor}

\begin{remark}[The EDM predictor is not \(\eta\)-Euler]
  \label{rem:frozen_step_not_eta_euler}
  The frozen \(\sigma\)-step \eqref{eq:oracle_edm_predictor} is the map analyzed throughout. It is
  not the explicit Euler step of the \(\eta\)-clock ODE \cref{eq:eta_probability_flow_ode}: the
  exact flow is invariant under \(\eta=\sigma^\rho\) but explicit Euler is not, and the two use
  different relaxation coefficients, agreeing only to first order in \(h_j/\eta_j\). On the
  power-law clock the EDM predictor is the integrating-factor form of that \(\eta\)-step, whatever
  \(\rho\); the rest of this subsection records the derivation and the exact one-step discrepancy.
\end{remark}

The EDM predictor \cref{eq:oracle_edm_predictor,eq:learned_edm_predictor} is the finite-step sampler
analyzed in this paper. This choice should be separated from the clock convention itself. The exact
probability flow is invariant under the change of variables \(\eta=\sigma^\rho\), but explicit Euler
is not invariant under a nonlinear reparameterization of time. Consequently, an explicit Euler step
in the \(\eta\)-clock is not, at finite step size, the same update as the first-order predictor used
in EDM on the same positive noise levels \cite[Sec.~3]{karras_elucidating_2022}. This appendix makes
the difference precise in two ways: it gives the exact one-step discrepancy between the two updates,
and it reads the EDM predictor as an integrating-factor variant of \(\eta\)-Euler. The upshot is
that the EDM predictor, unlike \(\eta\)-Euler, is the same map on every power-law clock, so the
clock exponent changes only where the levels are placed, not the update. The full EDM sampler may
also include stochastic churn and a Heun correction \cite[Sec.~3]{karras_elucidating_2022}. The
comparison here concerns only the underlying deterministic first-order predictor, because this is
the map directly comparable with explicit \(\eta\)-Euler.

We work on a generic decreasing clock grid \(\eta_0>\eta_1>\cdots>\eta_K>0\), with local steps
\(h_j=\eta_j-\eta_{j+1}\) and noise levels \(\sigma_j=\eta_j^{1/\rho}\); on the refinement family
\(\Gsched\) all local steps are equal, \(h_j\equiv h_K\). The comparison is governed by the relative
step
\begin{equation*}
  \bar h_j = \frac{h_j}{\eta_j} = 1-\Bigl(\frac{\sigma_{j+1}}{\sigma_j}\Bigr)^{\rho} \in(0,1).
\end{equation*}

To make the distinction explicit, suppose that a learned denoiser \(\Dhat_{\sigma_j}\) induces the
learned \(\eta\)-clock velocity
\begin{equation*}
  \widehat w_j(x) = \frac{1}{\rho\eta_j} \bigl(x-\Dhat_{\sigma_j}(x)\bigr),
  \qquad \eta_j=\sigma_j^\rho.
\end{equation*}
The corresponding explicit \(\eta\)-Euler denoiser relaxation is
\begin{equation*}
  \Phihat_j^{\mathrm{Eu}}(x) = x-h_j\widehat w_j(x)
  = (1-a_j^{\mathrm{Eu}})x+a_j^{\mathrm{Eu}}\Dhat_{\sigma_j}(x), \qquad a_j^{\mathrm{Eu}}
  = \frac{h_j}{\rho\eta_j} = \frac{\bar h_j}{\rho}.
\end{equation*}
By contrast, the EDM predictor \cref{eq:learned_edm_predictor} has the same form with the relaxation
coefficient
\begin{equation*}
  a_j = 1-\frac{\sigma_{j+1}}{\sigma_j} = 1-(1-\bar h_j)^{1/\rho}.
\end{equation*}

Thus both updates use the same denoiser and the same pair of positive noise levels, but they use
different relaxation coefficients, \(a_j^{\mathrm{Eu}}=\bar h_j/\rho\) and \(a_j=1-(1-\bar
h_j)^{1/\rho}\). For fixed \(\rho>0\), Taylor expansion gives
\begin{equation*}
  a_j = \frac{\bar h_j}{\rho} - \frac{1-\rho}{2\rho^2}\bar h_j^2 + O_\rho(\bar h_j^3)
  = a_j^{\mathrm{Eu}} - \frac{1-\rho}{2\rho^2}\bar h_j^2 + O_\rho(\bar h_j^3).
\end{equation*}
The two schemes are therefore first-order equivalent under the relative mesh condition \(\max_j\bar
h_j\to0\), but they need not be close on a coarse terminal grid, where \(\eta_j\) is small and
\(\bar h_j\) may be of order one. The one-step discrepancy has the exact algebraic form
\begin{equation*}
  \Phihat_j^{\mathrm{Eu}}(x)-\Phihat_j(x)
  = \bigl(a_j^{\mathrm{Eu}}-a_j\bigr) \bigl(\Dhat_{\sigma_j}(x)-x\bigr).
\end{equation*}
Consequently,
\begin{equation*}
  \norm{\Phihat_j^{\mathrm{Eu}}(x)-\Phihat_j(x)}
  \leqslant \left|a_j^{\mathrm{Eu}}-a_j\right| \norm{x-\Dhat_{\sigma_j}(x)}.
\end{equation*}
This discrepancy identity is the basic reconciliation between the \(\eta\)-clock Euler sampler and
the EDM finite-step predictor: they discretize the same continuous probability flow and may use the
same denoiser, but they are distinct finite-step integrators.

There is also a useful way to read the EDM coefficient from the \(\eta\)-clock. If, over one step,
the denoiser is frozen at its left-endpoint value \(\Dhat_{\sigma_j}(x_j)\), then the linear
equation
\begin{equation*}
  \frac{\dd x}{\dd\eta} = \frac{1}{\rho\eta} \bigl(x-\Dhat_{\sigma_j}(x_j)\bigr)
\end{equation*}
can be integrated exactly from \(\eta_j\) to \(\eta_{j+1}\). This gives
\begin{equation*}
  x_{j+1} = \frac{\sigma_{j+1}}{\sigma_j}x_j
  + \left(1-\frac{\sigma_{j+1}}{\sigma_j}\right) \Dhat_{\sigma_j}(x_j),
\end{equation*}
which is precisely the EDM predictor. In this sense, on the power-law clock, the EDM predictor may
be viewed as a \(\sigma\)-adapted or integrating-factor variant of the \(\eta\)-Euler step. It
treats the leading scalar factor \((\rho\eta)^{-1}\) exactly while freezing the denoiser over the
step. The result does not depend on \(\rho\), which is the invariance announced above.

\paragraph{Relation to prior work: identification of the predictor.}
The predictor analyzed here is the deterministic DDIM update \cite{song_denoising_implicit_2021} in
the variance-exploding parametrization, which is also the first-order DPM-Solver step
\cite{lu_dpm_solver_2022}.

\subsection{Proof of the One-Step Decomposition}

\begin{proof}[Proof of \cref{eq:law_level_one_step}]
By the triangle inequality in \(\Wtwo\), applied to \(\mu_{j+1}=(\Psi_j)_\#\mu_j\) and
\(\widehat\mu_{j+1}=(\Phihat_j)_\#\widehat\mu_j\),
\begin{equation*}
  e_{j+1} \leqslant \Wtwo\!\left((\Psi_j)_\#\mu_j,(\Phi_j)_\#\mu_j\right)
  + \Wtwo\!\left((\Phi_j)_\#\mu_j,(\Phihat_j)_\#\mu_j\right)
  + \Wtwo\!\left((\Phihat_j)_\#\mu_j,(\Phihat_j)_\#\widehat\mu_j\right).
\end{equation*}
All four laws lie in \(\mathcal P_2(\R^d)\) by \Cref{ass:wellposed}, so every term is finite. The
first two terms are bounded by the cost of the diagonal coupling induced by \(\mu_j\): the first
gives \(\delta_j\) by \cref{eq:def_local_defect_learned_map}, and the second gives the learning term
through the pointwise identity \(\Phi_j-\Phihat_j=a_j(D_{\sigma_j}-\Dhat_{\sigma_j})\) evaluated at
\(X_{\sigma_j}\sim\mu_j\). The third term equals \(\gamma_je_j\) when \(e_j>0\) by
\cref{eq:def_raw_amplification}, and vanishes when \(e_j=0\), since then \(\mu_j=\widehat\mu_j\) and
the two pushforwards coincide.
\end{proof}

\subsection{Three One-Step Decompositions}
\label{app:alt_decomposition}

This subsection justifies the order of the one-step decomposition \eqref{eq:law_level_one_step}: of
the three possible decompositions compared below, it is the only one that keeps both local errors it
creates, discretization and learning, on the exact law \(\mu_j\).

Let \((Y_j,\widehat X_j)\) be an optimal coupling of \((\mu_j,\widehat\mu_j)\), and set
\(Y_{j+1}=\Psi_j(Y_j)\sim\mu_{j+1}\) and \(\widehat X_{j+1}=\Phihat_j(\widehat
X_j)\sim\widehat\mu_{j+1}\). Any three-term bound on \(\norm{Y_{j+1}-\widehat X_{j+1}}_{\mathbb
L_2}\) walks along the chain \(\Psi_j\to\Phi_j\to\Phihat_j\), whose arrows are the discretization
bias and the learning error, and changes the argument from \(Y_j\) to \(\widehat X_j\) exactly once.
Where the change happens fixes which map is transported, and under which law the two local errors
are measured: upstream of the change they sit on the exact law \(\mu_j\), downstream on the sampler
law \(\widehat\mu_j\). Writing
\(\varepsilon_j(\lambda)=a_j\norm{\Dhat_{\sigma_j}-D_{\sigma_j}}_{\mathbb L_2(\lambda)}\) for the
learning error under a law \(\lambda\), so that \(\varepsilon_j(\mu_j)=a_j\beta_j\Rres_j\) by
\cref{eq:edm_den_error_residual}, the three options are
\begin{align*}
  \text{last:}\quad e_{j+1} \leqslant\norm{\Phihat_j(Y_j)-\Phihat_j(\widehat X_j)}_{\mathbb L_2}
  +\norm{\Psi_j-\Phi_j}_{\mathbb L_2(\mu_j)}+\varepsilon_j(\mu_j), \\
  \text{middle:}\quad e_{j+1} \leqslant\norm{\Phi_j(Y_j)-\Phi_j(\widehat X_j)}_{\mathbb L_2}
  +\norm{\Psi_j-\Phi_j}_{\mathbb L_2(\mu_j)}+\varepsilon_j(\widehat\mu_j), \\
  \text{first:}\quad e_{j+1} \leqslant\norm{\Psi_j(Y_j)-\Psi_j(\widehat X_j)}_{\mathbb L_2}
  +\norm{\Psi_j-\Phi_j}_{\mathbb L_2(\widehat\mu_j)}+\varepsilon_j(\widehat\mu_j).
\end{align*}

\begin{center}
  \begin{tabular}{llll}
\toprule
Argument change & Transported map & Discretization on & Learning on \\
\midrule
after \(\Phihat_j\) & learned \(\Phihat_j\) & \(\mu_j\) & \(\mu_j\) \\
after \(\Phi_j\) & oracle \(\Phi_j\) & \(\mu_j\) & \(\widehat\mu_j\) \\
after \(\Psi_j\) & exact flow \(\Psi_j\) & \(\widehat\mu_j\) & \(\widehat\mu_j\) \\
\bottomrule
  \end{tabular}
\end{center}

The choice trades the stability of one map against the law under which the local errors live. The
first option transports the exact flow, whose stability is a property of the data alone; this is the
route of the log-concave analysis compared in \Cref{app:related_work_extended}. It pays by moving
both local errors onto \(\widehat\mu_j\), where neither the universal acceleration bound nor the
identity \(\varepsilon_j(\mu_j)=a_j\beta_j\Rres_j\) applies without a change-of-measure estimate,
and \(\Wtwo\)-closeness of \(\widehat\mu_j\) to \(\mu_j\) alone does not supply one. The middle
option keeps the acceleration bound but still moves the learning error off the exact law, and
transports the oracle predictor, with
\(\operatorname{Lip}(\Phi_j)\leqslant1-a_j+a_j\operatorname{Lip}(D_{\sigma_j})\) set by the score
geometry of the path. Without structural assumptions on \(q\), \(\operatorname{Lip}(D_{\sigma_j})\)
is not controlled at low noise, so this transported term is a cost rather than a gain. The last
option, used in \cite{beyler_convergence_2025}, is a coupling-level form of the decomposition
\eqref{eq:law_level_one_step}. It keeps both local errors on \(\mu_j\), where
\Cref{thm:global_curvature_bound} and \cref{eq:edm_den_error_residual} apply directly, and places
the whole dependence on the network in the stability of \(\Phihat_j\). Its transported term is the
cost of one particular coupling of \((\Phihat_j)_\#\mu_j\) and \((\Phihat_j)_\#\widehat\mu_j\), and
so dominates the intrinsic term \(\gamma_je_j\) of \cref{eq:law_level_one_step}, which measures the
optimal one.

\Cref{app:related_work_extended} compares these routes, as they are used in the closest Wasserstein
analyses, with the present one.

\section{Universal Local Error}
\label{app:local_error}
\label{app:curvature}

This appendix supports \Cref{sec:discretization_error}. It proves the universal acceleration bound
\Cref{thm:global_curvature_bound} and the bound \Cref{lem:local_euler_defect} on the local
discretization bias (\Cref{subsec:proof_local_defect}), states the formal extrapolation to
higher-order integrators (\Cref{rem:higher_order_extrapolation}), and closes with exact formulas for
Gaussian input laws and the sharpness of the acceleration exponent. The acceleration \(\mathcal
A_\rho(\sigma,x)\) and the constant \(B_{d,\rho}\) are defined in
\Cref{subsec:acceleration_numerical_clock}; \(\mathcal A_1\) denotes the native \(\sigma\)-clock
case. The argument rewrites the acceleration in a sequence of increasingly convenient
coordinates---first in score--Hessian form, then in terms of conditional moments of the Gaussian
noise. From the noise-moment form the universal \(\mathbb L_2\) bound follows by elementary
inequalities, after which the same identity yields exact formulas in the Gaussian case.

\subsection{Clock Change and the Score--Hessian Form}

Before turning to the quantitative estimate, it is worth seeing where the curvature of the clock,
its second derivative \(\dd^2\sigma/\dd\eta^2\), enters. Under a smooth reparametrization
\(\sigma=\sigma(\eta)\) the velocity transforms as
\(w_\eta(x)=\tfrac{\dd\sigma}{\dd\eta}v_\sigma(x)\), and the chain rule gives
\begin{equation*}
  \partial_\eta w_\eta(x)+\nabla w_\eta(x)\,w_\eta(x)
  = \left(\frac{\dd\sigma}{\dd\eta}\right)^2 \left( \partial_\sigma v_\sigma(x)
  + \nabla v_\sigma(x)\,v_\sigma(x) \right) + \frac{\dd^2\sigma}{\dd\eta^2}v_\sigma(x).
\end{equation*}
For the power-law clock \(\eta=\sigma^\rho\) with \(\rho>0\) one has
\(\tfrac{\dd\sigma}{\dd\eta}=\tfrac{1}{\rho}\sigma^{1-\rho}\) and
\(\tfrac{\dd^2\sigma}{\dd\eta^2}=\tfrac{1-\rho}{\rho^2}\sigma^{1-2\rho}\), so that
\begin{equation*}
  \mathcal A_\rho(\sigma,x) = \frac{1}{\rho^2}\sigma^{2-2\rho}\mathcal A_1(\sigma,x)
  + \frac{1-\rho}{\rho^2}\sigma^{1-2\rho}v_\sigma(x).
\end{equation*}
The acceleration in the numerical clock is thus a rescaling of the acceleration in the noise clock,
plus a multiple of the velocity \(v_\sigma\) itself, with coefficient proportional to \(1-\rho\): a
genuinely curved clock (\(\rho\neq1\)) injects this velocity term, which a linear clock (\(\rho=1\))
does not. It reappears as the term \((1-\rho)\Idd\,N_\sigma\) of the noise-moment identity
\eqref{eq:Aeta_noise_moment_identity}, since \(N_\sigma=v_\sigma\) by Tweedie's identity.

For the quantitative estimate it is more convenient to write the clock acceleration in
score--Hessian coordinates. Besides the score \(s_\sigma=\nabla\log p_\sigma\), we write
\begin{equation*}
  H_\sigma=\nabla^2\log p_\sigma=\nabla s_\sigma
\end{equation*}
for the Hessian of the log-density of \(\mu_\sigma\); both are smooth on \(\R^d\) because
\(p_\sigma\) is smooth and positive.

\begin{lemma}[Score--Hessian formula for the clock acceleration]
  Let \(\rho>0\) and let \(q\) be a Borel probability measure on \(\R^d\). For every \(x\in\R^d\)
  and every \(\sigma>0\),
  \begin{equation}\label{eq:Aeta_score_hessian_formula}
    \mathcal A_\rho(\sigma,x) = \frac{1}{\rho^2}\left( -\sigma^{4-2\rho}\nabla\Tr H_\sigma(x)
    -\sigma^{4-2\rho}H_\sigma(x)s_\sigma(x) +(\rho-2)\sigma^{2-2\rho}s_\sigma(x) \right).
  \end{equation}
\end{lemma}

\begin{proof}
The heat equation for the Gaussian smoothing in the \(\sigma\)-clock is
\begin{equation*}
  \partial_\sigma p_\sigma(x) = \sigma\Tr\nabla^2 p_\sigma(x).
\end{equation*}
Since \(\nabla^2p_\sigma/p_\sigma=H_\sigma+s_\sigma s_\sigma^\top\), dividing by \(p_\sigma\) gives
\begin{equation*}
  \partial_\sigma \log p_\sigma = \sigma\left(\Tr H_\sigma+\norm{s_\sigma}^2\right).
\end{equation*}
Differentiating in \(x\) gives
\begin{equation}\label{eq:partial_sigma_score_curvature}
  \partial_\sigma s_\sigma = \sigma\left(\nabla\Tr H_\sigma+2H_\sigma s_\sigma\right).
\end{equation}
Using
\begin{equation*}
  w_\eta(x) = -\frac{1}{\rho}\sigma^{2-\rho}s_\sigma(x), \qquad \frac{\dd\sigma}{\dd\eta}
  = \frac{1}{\rho}\sigma^{1-\rho},
\end{equation*}
we compute
\begin{equation*}
  \partial_\eta w_\eta(x) = -\frac{2-\rho}{\rho^2}\sigma^{2-2\rho}s_\sigma(x)
  - \frac{1}{\rho^2}\sigma^{3-2\rho}\partial_\sigma s_\sigma(x),
\end{equation*}
and
\begin{equation*}
  \nabla w_\eta(x)\,w_\eta(x) = \frac{1}{\rho^2}\sigma^{4-2\rho}H_\sigma(x)s_\sigma(x).
\end{equation*}
Therefore
\begin{equation*}
  \mathcal A_\rho(\sigma,x) = \frac{1}{\rho^2} \left( (\rho-2)\sigma^{2-2\rho}s_\sigma(x)
  -\sigma^{3-2\rho}\partial_\sigma s_\sigma(x) +\sigma^{4-2\rho}H_\sigma(x)s_\sigma(x) \right).
\end{equation*}
Substituting \cref{eq:partial_sigma_score_curvature} yields \cref{eq:Aeta_score_hessian_formula}.
\end{proof}

\subsection{The Noise-Moment Identity}

We now rewrite the same acceleration in terms of conditional moments of the Gaussian noise. This is
the form from which the universal \(\mathbb L_2\) bound follows.

Throughout the rest of this appendix, for fixed \(x\in\R^d\) and \(\sigma>0\), the notations
\(\E[\,\cdot\mid X_\sigma=x]\) and \(\Cov(\,\cdot\mid X_\sigma=x)\) denote integrals in \(X_0\)
against the posterior kernel \(\pi_{x,\sigma}\), with \(Z=(x-X_0)/\sigma\) under this conditioning.
By \cref{eq:posterior_kernel_identity_joint}, evaluating these pointwise representatives at
\(x=X_\sigma\) realizes the corresponding conditional expectations given \(X_\sigma\) almost surely,
including for the unbounded integrands, depending on the conditioning point, considered here.

With this convention, recall that
\begin{equation*}
  N_\sigma(x) = \E[Z \mid X_\sigma = x]
\end{equation*}
is the conditional mean of the noise given \(X_\sigma = x\). Define the conditional noise covariance
by
\begin{equation*}
  \mathcal C_\sigma(x) = \Cov(Z\mid X_\sigma=x),
\end{equation*}
and the contracted third centered conditional noise moment by
\begin{equation*}
  \mathfrak m_{3,\sigma}(x)
  = \E\left[ \norm{Z-N_\sigma(x)}^2 \bigl(Z-N_\sigma(x)\bigr) \mid X_\sigma=x \right].
\end{equation*}

For the relation of these identities to higher-order Tweedie formulas, see
\Cref{app:related_work_extended}. The posterior differentiation identity used below is the Gaussian
specialization of the meta-derivative identity of Dytso--Poor--Shamai
\cite[Thm.~1]{dytso_poor_shamai_gaussian_derivative_2023}; see also the observation-dependent
formulation of Dytso--Cardone--Zieder \cite{dytso_cardone_zieder_meta_derivative_2023}. We include
the short proof in our notation because \(q\) is only assumed to be a Borel probability measure, and
because the function to which the identity is applied depends explicitly on the conditioning point.

\begin{lemma}[Gaussian posterior differentiation rule]
  \label{lem:gaussian_posterior_differentiation}
  Let \(q\) be a Borel probability measure on \(\R^d\) and fix \(\sigma>0\). For a measurable
  function \(F:\R^d\times\R^d\to\R\), let \(\langle F\rangle_\sigma\) be the posterior average
  defined in \cref{eq:app_def_posterior_average}, so that, with the convention above,
  \begin{equation*}
    \E[F(x,X_0)\mid X_\sigma=x] = \langle F\rangle_\sigma(x)
    = \int_{\R^d}F(x,y)\,\pi_{x,\sigma}(\dd y).
  \end{equation*}
  Assume that \(F\) is continuously differentiable in its first variable and that, locally uniformly
  in \(x\), both \(F(x,y)\) and \(\nabla_xF(x,y)\) have at most polynomial growth in \(y\). Then
  \(x\mapsto \langle F\rangle_\sigma(x)\) is \(C^1\), and
  \begin{equation}\label{eq:gaussian_posterior_differentiation_rule}
    \nabla_x \E[F(x,X_0)\mid X_\sigma=x] = \E[\nabla_xF(x,X_0)\mid X_\sigma=x]
    + \Cov\left( F(x,X_0), \frac{X_0-x}{\sigma^2} \,\middle|\, X_\sigma=x \right).
  \end{equation}
  Here the covariance between a scalar and a vector is understood componentwise. The same formula
  holds componentwise for vector-valued or matrix-valued \(F\).
\end{lemma}

\begin{proof}
With the Gaussian density \(\phi_\sigma\) of \cref{eq:app_density_sigma}, the definition of the
posterior kernel gives
\begin{equation*}
  \langle F\rangle_\sigma(x) = \frac{I_F(x)}{I_1(x)},
  \qquad I_F(x)=\int_{\R^d}F(x,y)\phi_\sigma(x-y)\,q(\dd y),
  \qquad I_1(x)=\int_{\R^d}\phi_\sigma(x-y)\,q(\dd y)=p_\sigma(x).
\end{equation*}
For \(x\) in a compact set, the Gaussian factor dominates every polynomial in \(y\), uniformly in
\(x\). Therefore the assumed polynomial-growth bounds justify differentiation under the integral
sign. Since
\begin{equation*}
  \nabla_x \phi_\sigma(x-y) = \frac{y-x}{\sigma^2}\phi_\sigma(x-y),
\end{equation*}
we obtain
\begin{equation*}
  \nabla I_F(x) = \int_{\R^d} \left( \nabla_xF(x,y)
  + F(x,y)\frac{y-x}{\sigma^2} \right) \phi_\sigma(x-y)\,q(\dd y),
\end{equation*}
and
\begin{equation*}
  \nabla I_1(x) = \int_{\R^d} \frac{y-x}{\sigma^2} \phi_\sigma(x-y)\,q(\dd y).
\end{equation*}
The quotient rule gives
\begin{equation*}
  \begin{split} \nabla_x\langle F\rangle_\sigma(x) ={} &\E[\nabla_xF(x,X_0)\mid X_\sigma=x]
    + \E\left[ F(x,X_0)\frac{X_0-x}{\sigma^2} \,\middle|\, X_\sigma=x \right] \\ &-
    \E[F(x,X_0)\mid X_\sigma=x]\, \E\left[ \frac{X_0-x}{\sigma^2} \,\middle|\, X_\sigma=x \right],
  \end{split}
\end{equation*}
which is exactly \cref{eq:gaussian_posterior_differentiation_rule}.
\end{proof}

With the posterior differentiation rule in hand, we can eliminate the score and Hessian in favour of
conditional moments of the noise alone. The score is handled by Tweedie's identity, the Hessian by
\Cref{lem:hessian_covariance_identity}, and the gradient of its trace by
\Cref{lem:trace_gradient_identity}, which is where the posterior differentiation rule enters;
substituting the three into \cref{eq:Aeta_score_hessian_formula} gives the noise-moment identity
\Cref{prop:noise_moment_identity}.

\begin{lemma}[Hessian--covariance identity]
  \label{lem:hessian_covariance_identity}
  Let \(q\) be a Borel probability measure on \(\R^d\). For every \(x\in\R^d\) and every
  \(\sigma>0\),
  \begin{equation}\label{eq:H_noise_covariance_identity}
    s_\sigma(x) = -\sigma^{-1}N_\sigma(x), \qquad H_\sigma(x)
    = \sigma^{-2} \bigl( \mathcal C_\sigma(x)-\Idd \bigr).
  \end{equation}
\end{lemma}

\begin{proof}
The first identity is Tweedie's identity. Moreover, differentiating the Gaussian kernel twice gives
\begin{equation*}
  \frac{\nabla^2 p_\sigma(x)}{p_\sigma(x)}
  = \frac{1}{\sigma^2} \left( \E[ZZ^\top\mid X_\sigma=x]-\Idd \right).
\end{equation*}
Since
\begin{equation*}
  H_\sigma = \nabla^2\log p_\sigma = \frac{\nabla^2p_\sigma}{p_\sigma} - s_\sigma s_\sigma^\top,
\end{equation*}
and \(s_\sigma=-\sigma^{-1}N_\sigma\), we obtain \(H_\sigma(x)=\sigma^{-2}\bigl(\E[ZZ^\top\mid
  X_\sigma=x]-N_\sigma(x)N_\sigma(x)^\top-\Idd\bigr) =\sigma^{-2}(\mathcal C_\sigma(x)-\Idd)\).
\end{proof}

\begin{lemma}[Gradient of the conditional noise variance]
  \label{lem:trace_gradient_identity}
  Let \(q\) be a Borel probability measure on \(\R^d\). For every \(x\in\R^d\) and every
  \(\sigma>0\),
  \begin{equation}\label{eq:grad_trace_C_identity}
    \nabla_x\Tr\mathcal C_\sigma(x) = -\frac1\sigma\mathfrak m_{3,\sigma}(x),
    \qquad\text{hence}\qquad \nabla\Tr H_\sigma(x) = -\sigma^{-3}\mathfrak m_{3,\sigma}(x).
  \end{equation}
\end{lemma}

\begin{proof}
Write
\begin{equation*}
  N=N_\sigma(x), \qquad W=Z-N .
\end{equation*}
Then
\begin{equation*}
  \mathcal C_\sigma(x) = \E[WW^\top\mid X_\sigma=x], \qquad \Tr\mathcal C_\sigma(x)
  = \E[\norm{W}^2\mid X_\sigma=x].
\end{equation*}
We apply \Cref{lem:gaussian_posterior_differentiation} to
\begin{equation*}
  F(x,X_0) = \norm{W}^2 = \left\| \frac{x-X_0}{\sigma} - N_\sigma(x) \right\|^2 .
\end{equation*}
This application is legitimate because \( N_\sigma(x) = -\sigma s_\sigma(x), \) and \(s_\sigma\) is
smooth in \(x\). Hence, on compact \(x\)-sets, \(N_\sigma\) and \(\nabla_xN_\sigma\) are bounded.
Therefore \(F\) and \(\nabla_xF\) have polynomial growth in \(X_0\), and the posterior Gaussian
factor provides all required integrability. The differentiation rule gives
\begin{equation}\label{eq:grad_trace_C_split}
  \nabla_x\Tr\mathcal C_\sigma(x) = \E[\nabla_x\norm{W}^2\mid X_\sigma=x]
  + \Cov\left( \norm{W}^2, \frac{X_0-x}{\sigma^2} \,\middle|\, X_\sigma=x \right).
\end{equation}
We compute the two terms separately. For each coordinate \(i=1,\ldots,d\),
\begin{equation*}
  \partial_{x_i}W = \frac{e_i}{\sigma} - \partial_{x_i}N_\sigma(x),
\end{equation*}
and therefore
\begin{equation*}
  \partial_{x_i}\norm{W}^2 = 2W^\top \left( \frac{e_i}{\sigma} - \partial_{x_i}N_\sigma(x) \right).
\end{equation*}
The vector in parentheses is deterministic once \(x\) is fixed, while
\begin{equation*}
  \E[W\mid X_\sigma=x]=0.
\end{equation*}
Consequently,
\begin{equation}\label{eq:explicit_derivative_term_zero}
  \E[\nabla_x\norm{W}^2\mid X_\sigma=x]=0.
\end{equation}
For the covariance term, observe that
\begin{equation*}
  \frac{X_0-x}{\sigma^2} = -\frac{Z}{\sigma}.
\end{equation*}
Hence
\begin{equation*}
  \Cov\left( \norm{W}^2, \frac{X_0-x}{\sigma^2} \,\middle|\, X_\sigma=x \right) =
  -\frac1\sigma \Cov\left( \norm{W}^2, Z \,\middle|\, X_\sigma=x \right).
\end{equation*}
Since \(Z=N+W\), and \(N=N_\sigma(x)\) is deterministic under the conditioning,
\begin{equation*}
  \Cov\left( \norm{W}^2, Z \,\middle|\, X_\sigma=x \right)
  = \Cov\left( \norm{W}^2, W \,\middle|\, X_\sigma=x \right).
\end{equation*}
Using again \(\E[W\mid X_\sigma=x]=0\), this last covariance is
\begin{equation*}
  \E\left[ \norm{W}^2W \mid X_\sigma=x \right] = \mathfrak m_{3,\sigma}(x).
\end{equation*}
Combining this with \cref{eq:grad_trace_C_split,eq:explicit_derivative_term_zero} gives the first
identity in \cref{eq:grad_trace_C_identity}. By \cref{eq:H_noise_covariance_identity},
\begin{equation*}
  \Tr H_\sigma(x) = \sigma^{-2} \bigl( \Tr\mathcal C_\sigma(x)-d \bigr),
\end{equation*}
and since \(\sigma\) is fixed when differentiating with respect to \(x\), the second identity
follows.
\end{proof}

The resulting identity is the pivot of the whole appendix: the universal bound below follows from it
by crude moment inequalities, while its Gaussian specialization is exact.

\begin{proposition}[Noise-moment formula for the clock acceleration]
  \label{prop:noise_moment_identity}
  Let \(\rho>0\) and let \(q\) be a Borel probability measure on \(\R^d\). For every \(x\in\R^d\)
  and every \(\sigma>0\),
  \begin{equation}\label{eq:Aeta_noise_moment_identity}
    \mathcal A_\rho(\sigma,x)
    = \frac{\sigma^{1-2\rho}}{\rho^2} \left(
    \bigl((1-\rho)\Idd+\mathcal C_\sigma(x)\bigr)N_\sigma(x) + \mathfrak m_{3,\sigma}(x) \right).
  \end{equation}
\end{proposition}

\begin{proof}
We substitute
\begin{equation*}
  s_\sigma(x)=-\sigma^{-1}N_\sigma(x), \qquad H_\sigma(x)=\sigma^{-2}(\mathcal C_\sigma(x)-\Idd),
  \qquad \nabla\Tr H_\sigma(x)=-\sigma^{-3}\mathfrak m_{3,\sigma}(x),
\end{equation*}
from \Cref{lem:hessian_covariance_identity,lem:trace_gradient_identity}, into the score--Hessian
acceleration formula \eqref{eq:Aeta_score_hessian_formula}. The three terms become
\begin{equation*}
  -\sigma^{4-2\rho}\nabla\Tr H_\sigma = \sigma^{1-2\rho}\mathfrak m_{3,\sigma},
\end{equation*}
\begin{equation*}
  -\sigma^{4-2\rho}H_\sigma s_\sigma = \sigma^{1-2\rho} (\mathcal C_\sigma-\Idd)N_\sigma,
\end{equation*}
and
\begin{equation*}
  (\rho-2)\sigma^{2-2\rho}s_\sigma = (2-\rho)\sigma^{1-2\rho}N_\sigma.
\end{equation*}
Therefore
\begin{equation*}
  \mathcal A_\rho(\sigma,x) = \frac{\sigma^{1-2\rho}}{\rho^2} \left[ \mathfrak m_{3,\sigma}(x)
  + (\mathcal C_\sigma(x)-\Idd)N_\sigma(x) + (2-\rho)N_\sigma(x) \right].
\end{equation*}
Since \( -\Idd+(2-\rho)\Idd = (1-\rho)\Idd, \) this gives
\begin{equation*}
  \mathcal A_\rho(\sigma,x)
  = \frac{\sigma^{1-2\rho}}{\rho^2} \left( \bigl((1-\rho)\Idd+\mathcal C_\sigma(x)\bigr)N_\sigma(x)
  + \mathfrak m_{3,\sigma}(x) \right).
\end{equation*}
This proves \cref{eq:Aeta_noise_moment_identity}.
\end{proof}

\subsection{Proof of the Universal Bound}

We can now prove the universal acceleration bound. The noise-moment identity
\eqref{eq:Aeta_noise_moment_identity} factors \(\mathcal A_\rho(\sigma,\cdot)\) into the explicit
scale prefactor \(\sigma^{1-2\rho}/\rho^2\) and a bracketed moment vector; the entire task is to
bound that vector in \(\mathbb L_2\) uniformly over all input laws \(q\), which reduces to
controlling the conditional mean, covariance, and third moment of a standard Gaussian noise.

\begin{proof}[Proof of \Cref{thm:global_curvature_bound}]
Throughout the proof all conditional quantities are evaluated at \(X_\sigma\). We write
\begin{equation*}
  N=N_\sigma(X_\sigma) = \E[Z\mid X_\sigma],
\end{equation*}
\begin{equation*}
  W=Z-N,
\end{equation*}
\begin{equation*}
  \mathcal C=\mathcal C_\sigma(X_\sigma) = \E[WW^\top\mid X_\sigma],
\end{equation*}
and
\begin{equation*}
  \mathfrak m_3=\mathfrak m_{3,\sigma}(X_\sigma) = \E[\norm{W}^2W\mid X_\sigma].
\end{equation*}
By the noise-moment formula for the clock acceleration \cref{eq:Aeta_noise_moment_identity},
\begin{equation*}
  \mathcal A_\rho(\sigma,X_\sigma)
  = \frac{\sigma^{1-2\rho}}{\rho^2} \left( ((1-\rho)\Idd+\mathcal C)N+\mathfrak m_3 \right).
\end{equation*}
Thus it is enough to bound, uniformly in \(q\) and \(\sigma\),
\begin{equation*}
  S_\sigma = ((1-\rho)\Idd+\mathcal C)N+\mathfrak m_3
\end{equation*}
in \(\mathbb L_2\). First, since conditional expectation is a contraction in \(\mathbb L_2\),
\begin{equation}\label{eq:N_L2_improved}
  \norm{N}_{\mathbb L_2} = \norm{\E[Z\mid X_\sigma]}_{\mathbb L_2} \le \norm{Z}_{\mathbb L_2}
  = \sqrt d .
\end{equation}
The covariance and third-moment terms are bounded together. Since \(N+W=Z\),
\begin{equation*}
  \mathcal CN+\mathfrak m_3 = \E\bigl[W\bigl(\ip{W}{N}+\norm{W}^2\bigr)\mid X_\sigma\bigr]
  = \E\bigl[W\ip{W}{Z}\mid X_\sigma\bigr].
\end{equation*}
By the Cauchy--Schwarz inequality and conditional Jensen's inequality,
\begin{equation*}
  \norm{\mathcal CN+\mathfrak m_3}^2 \le \left( \E[\norm{W}^2\norm{Z}\mid X_\sigma] \right)^2
  \le \E[\norm{W}^4\norm{Z}^2\mid X_\sigma].
\end{equation*}
Taking expectations and applying H\"older's inequality with exponents \(3/2\) and \(3\),
\begin{equation*}
  \norm{\mathcal CN+\mathfrak m_3}_{\mathbb L_2}^2
  \le \left(\E\norm{W}^6\right)^{2/3} \left(\E\norm{Z}^6\right)^{1/3}.
\end{equation*}
Since \(W=Z-\E[Z\mid X_\sigma]\), the triangle inequality, followed by the contraction property of
conditional expectation in \(\mathbb L_6\), gives
\begin{equation*}
  \norm{W}_{\mathbb L_6} \le \norm{Z}_{\mathbb L_6} + \norm{\E[Z\mid X_\sigma]}_{\mathbb L_6}
  \le 2\norm{Z}_{\mathbb L_6},
\end{equation*}
so that \(\E\norm{W}^6\le2^6\E\norm{Z}^6\). Since \(Z\sim\mathcal N(0,\Idd)\),
\begin{equation*}
  \E\norm{Z}^6 = d(d+2)(d+4).
\end{equation*}
Consequently,
\begin{equation}\label{eq:merged_L2}
  \norm{\mathcal CN+\mathfrak m_3}_{\mathbb L_2} \le 4\sqrt{d(d+2)(d+4)} .
\end{equation}
Combining \cref{eq:N_L2_improved,eq:merged_L2}, we obtain
\begin{equation*}
  \norm{S_\sigma}_{\mathbb L_2} \le |1-\rho|\sqrt d + 4\sqrt{d(d+2)(d+4)}
  \le \sqrt d\,\bigl[|1-\rho|+4(d+3)\bigr],
\end{equation*}
where the last step uses \((d+2)(d+4)=(d+3)^2-1\le(d+3)^2\). Substituting this estimate into the
representation
\begin{equation*}
  \mathcal A_\rho(\sigma,X_\sigma) = \frac{\sigma^{1-2\rho}}{\rho^2}S_\sigma
\end{equation*}
proves \cref{eq:Aeta_global_bound} with \(B_{d,\rho}=\rho^{-2}\sqrt d\,[|1-\rho|+4(d+3)]\), as
stated in \Cref{thm:global_curvature_bound}.
\end{proof}

\subsection{Proof of the Local Bias Bound}
\label{subsec:proof_local_defect}

The local bias of \Cref{lem:local_euler_defect} compares one exact flow step with the oracle
predictor, which freezes the \(\sigma\)-clock velocity at the left endpoint. Their difference is the
Taylor remainder of the flow in the native clock, so the case \(\rho=1\) of
\Cref{thm:global_curvature_bound} controls it.

\begin{proof}[Proof of \Cref{lem:local_euler_defect}]
Write \(\ell_j=\sigma_j-\sigma_{j+1}=a_j\sigma_j\) for the step in the native \(\sigma\)-clock, as
in \Cref{subsec:low_noise_certificate}, so that \(\Phi_j(x)=x-\ell_jv_{\sigma_j}(x)\) by
\cref{eq:oracle_edm_predictor,eq:velocity}. Let \(Y_j\sim\mu_j\) and
\(y(r)=\Psi_{\sigma_j-r,\sigma_j}(Y_j)\) for \(0\le r\le\ell_j\). The exact flow transports the
smoothed laws, so \(y(r)\sim\mu_{\sigma_j-r}\), with \(\dd y/\dd r=-v_{\sigma_j-r}(y(r))\) and
\(\dd^2y/\dd r^2=\mathcal A_1(\sigma_j-r,y(r))\). Since
\(\Phi_j(y(0))=y(0)-\ell_jv_{\sigma_j}(y(0))\) and \(y(0)=Y_j\), Taylor's formula with integral
remainder gives
\begin{equation*}
  \Psi_j(Y_j)-\Phi_j(Y_j) = \int_0^{\ell_j}(\ell_j-r)\,\mathcal A_1(\sigma_j-r,y(r))\,\dd r,
\end{equation*}
and Minkowski's integral inequality yields \(\delta_j\le\int_0^{\ell_j}(\ell_j-r)\norm{\mathcal
A_1(\sigma_j-r,y(r))}_{\mathbb L_2}\,\dd r\); it applies because the integrand is jointly measurable
in \((r,\omega)\), as \(\mathcal A_1\) is continuous on \(\R^d\times(0,\infty)\) and \(r\mapsto
y(r)\) is continuous on \([0,\ell_j]\), where \(\sigma_j-r\ge\sigma_{j+1}>0\), and because its
\(\mathbb L_2\) norm is bounded uniformly in \(r\) by the estimate below. Taking \(\rho=1\) in
\Cref{thm:global_curvature_bound} gives \(\norm{\mathcal A_1(\sigma,X_\sigma)}_{\mathbb L_2}\le
B_{d,1}\sigma^{-1}\). With \(y(r)\sim\mu_{\sigma_j-r}\),
\begin{equation*}
  \delta_j \le B_{d,1}\int_0^{\ell_j}\frac{\ell_j-r}{\sigma_j-r}\,\dd r = \Delta_j,
\end{equation*}
since \(\sigma_j-\ell_j=\sigma_{j+1}\).
\end{proof}

\subsection{The Leading-Order Local Bias and the Clock}

The bound \eqref{eq:def_Delta_k_curvature_bound} has the compact form
\(\Delta_j=B_{d,1}\sigma_j\,g(a_j)\) with \(g(a)=a+(1-a)\log(1-a)=\tfrac12a^2+O(a^3)\). On a member
\(G_K\) of the refinement family \(\Gsched\), uniform in the clock \(\eta=\sigma^\rho\) with step
\(h_K\), one has \(\sigma_{j+1}/\sigma_j=(1-h_K/\eta_j)^{1/\rho}\), hence
\(a_j=1-(1-h_K/\eta_j)^{1/\rho}=\frac{h_K}{\rho\eta_j}\bigl[1+O_\rho(h_K/\eta_j)\bigr]\). Since
\(\eta_j^{-2}=\sigma_j^{-2\rho}\), this gives the expansion \eqref{eq:local_defect_clock_expansion},
\begin{equation*}
  \Delta_{j,K} = B_{d,1}\sigma_j\,\frac{a_j^2}{2}\bigl[1+O(a_j)\bigr]
  = \frac{B_{d,1}}{2\rho^2}\,h_K^2\sigma_j^{1-2\rho}
  \bigg[1+O_\rho\bigg(\frac{h_K}{\eta_j}\bigg)\bigg],
\end{equation*}
whose leading term carries the same noise exponent \(\sigma_j^{1-2\rho}\) as the pointwise
acceleration bound \eqref{eq:Aeta_global_bound}. Thus \(\rho=1/2\) is exactly the clock for which
the leading local bias bound does not depend on the noise level. This is the sense in which we call
it balanced: it balances the universal bias bound of the first-order predictor across noise scales.
In the EDM convention \(\rho=1/\rhoEDM\) (\Cref{sec:discretization_error}) this is \(\rhoEDM=2\),
the analogue for the analyzed first-order predictor of the schedule-smoothness calculation in
\cite[Appendix~D.1]{karras_elucidating_2022}. It is not a claim that \(\rho=1/2\) is optimal among
all schedules, nor that it optimizes the quality of the final samples.

The practical clock \(\rhoEDM=7\), that is \(\rho=1/7\), leaves a factor \(\sigma_j^{5/7}\) that
grows with the noise level: compared with the balanced clock, it refines the low-noise end of the
schedule and leaves the largest per-step biases at high noise. Within this analysis the departure is
harmless when the high-noise contraction is strong enough: in the supercritical regime of
\Cref{rem:three_regimes}, where the effective high-noise margin \(\beff\) exceeds \(1-\rho=6/7\),
the kernel discounts those biases and their sum is set by the exit scale \(\sigma_{\Khi}\), not by
\(\sigma_0\); \Cref{subsec:evidence_damping} measures where a pretrained network reaches this
regime. The finer low-noise steps serve sample quality, a criterion outside the scope of this
analysis. Why the balanced value is tied to first-order integrators is discussed in
\Cref{rem:higher_order_extrapolation}.

\begin{remark}[Formal extrapolation to higher-order integrators]
  \label{rem:higher_order_extrapolation}
  The balanced clock \(\rho=1/2\), i.e.\ \(\rhoEDM=2\), is tied to the first-order predictor:
  \(\mathcal A_\rho\) is the second \(\eta\)-derivative of the flow, and it is the second derivative
  that governs the local truncation error of a first-order step. Formally, the \(m\)-th
  \(\eta\)-derivative scales like \(\sigma^{1-m\rho}\), so a \(p\)-th order integrator on a uniform
  \(\eta\)-grid would have leading local error \(\propto h^{p+1}\sigma^{1-(p+1)\rho}\) and would be
  equalized by \(\rhoEDM=p+1\). This last statement is a formal extrapolation and is not proved
  here: only the case \(p=1\) is established, by \Cref{thm:global_curvature_bound}, which gives
  \(\rhoEDM=2\). In particular the value \(\rhoEDM=3\) that the extrapolation attaches to the
  second-order Heun corrector of the full EDM sampler---the value reported to nearly equalize the
  per-step truncation error \cite[Sec.~3]{karras_elucidating_2022}---is heuristic here, since we do
  not analyze the Heun corrector.
\end{remark}

\paragraph{Relation to prior work: schedule design.}
Schedule-oriented studies analyze backward SDEs, interpolation drifts, or the broader EDM design
space
\cite{wang_he_tao_design_space_2024,strasman_noise_schedule_2024,chen_lipschitz_schedules_2025,
kahouli_disentangling_2025}, or optimize the schedule directly for a given model or data law
\cite{sabour_align_steps_2024,williams_score_optimal_2024}.
\textcite{wainwright_denoising_growth_2026} derives local KL error bounds for Euler discretization of
a stochastic innovations representation, expressed through increments of denoising growth complexity
along the Gaussian heat flow; these bounds lead to optimized schedules and data-certified
guarantees. Our balanced clock \(\rho=1/2\) plays a different role: it equalizes across noise scales
a local-bias bound that holds uniformly over data laws, and is not meant to minimize the error for a
particular model or data law. The clock experiment of \Cref{tab:control_clock} illustrates the
distinction: the clock that balances the universal bound does not have the smallest measured biases
or final error.

\subsection{Gaussian Laws and Sharpness}
\label{subsec:gaussian_sharpness}

The universal bound is worst-case: it holds for every Borel input law, at the price of a dimensional
factor of order \(d^{3/2}\). For Gaussian input laws the same noise-moment identity can instead be
evaluated in closed form, since the third conditional moment vanishes and the conditional covariance
is explicit. This exhibits the exact dependence of the acceleration on the spectrum of the data
covariance, and supplies the witnesses used below to show that the exponent of the universal bound
is sharp and to separate the sharp steps of its proof from the lossy one.

\begin{proposition}[Exact clock acceleration for Gaussian input laws]
  Let \(\rho>0\) and assume that
  \begin{equation*}
    q=\mathcal N(m,\Sigma_0), \qquad \Sigma_0\in\mathbb S_+^d .
  \end{equation*}
  Let \(\lambda_1,\ldots,\lambda_d\) be the eigenvalues of \(\Sigma_0\). Then, for every
  \(\sigma>0\),
  \begin{equation}\label{eq:gaussian_exact_acceleration_formula}
    \left\| \mathcal A_\rho(\sigma,X_\sigma) \right\|_{\mathbb L_2}^2
    = \frac{\sigma^{2-4\rho}}{\rho^4} \sum_{i=1}^d \left(
    1-\rho+\frac{\lambda_i}{\lambda_i+\sigma^2} \right)^2 \frac{\sigma^2}{\lambda_i+\sigma^2}.
  \end{equation}
  In particular, within the Gaussian class,
  \begin{equation}\label{eq:gaussian_sqrt_d_bound}
    \left\| \mathcal A_\rho(\sigma,X_\sigma) \right\|_{\mathbb L_2}
    \le \frac{\sqrt d\,\Gamma_\rho}{\rho^2}\sigma^{1-2\rho}, \qquad \Gamma_\rho
    = \max_{0\le u\le 1}\sqrt u\,|2-\rho-u|.
  \end{equation}
  The \(\sqrt d\) dependence in \cref{eq:gaussian_sqrt_d_bound} is sharp inside the Gaussian family.
\end{proposition}

\begin{proof}
Let \(V_\sigma=\Sigma_0+\sigma^2\Idd\). Since \((Z,X_\sigma)\) is jointly Gaussian,
\begin{align*}
  N_\sigma(x) &= \E[Z\mid X_\sigma=x] = \sigma V_\sigma^{-1}(x-m), \\
  \mathcal C_\sigma(x) &= \Cov(Z\mid X_\sigma=x) = \Idd-\sigma^2V_\sigma^{-1}
  = \Sigma_0V_\sigma^{-1}.
\end{align*}
The conditional law of \(Z\) given \(X_\sigma=x\) is Gaussian with covariance independent of \(x\),
so its centered third conditional moment vanishes, \(\mathfrak m_{3,\sigma}(x)=0\). Using the
noise-moment formula \eqref{eq:Aeta_noise_moment_identity},
\begin{equation*}
  \mathcal A_\rho(\sigma,x)
  = \frac{\sigma^{1-2\rho}}{\rho^2} \bigl((1-\rho)\Idd+\mathcal C_\sigma\bigr)N_\sigma(x).
\end{equation*}
Diagonalizing \(\Sigma_0\), the \(i\)-th coordinate of \(N_\sigma(X_\sigma)\) has variance
\(\sigma^2/(\lambda_i+\sigma^2)\), while the corresponding eigenvalue of \(\mathcal C_\sigma\) is
\(\lambda_i/(\lambda_i+\sigma^2)\). Hence
\begin{equation*}
  \left\| \mathcal A_\rho(\sigma,X_\sigma) \right\|_{\mathbb L_2}^2
  = \frac{\sigma^{2-4\rho}}{\rho^4} \sum_{i=1}^d \left( 1-\rho+\frac{\lambda_i}{\lambda_i+\sigma^2}
  \right)^2 \frac{\sigma^2}{\lambda_i+\sigma^2},
\end{equation*}
which proves \cref{eq:gaussian_exact_acceleration_formula}. Writing
\(u_i=\sigma^2/(\lambda_i+\sigma^2)\in(0,1]\), so that \(\lambda_i/(\lambda_i+\sigma^2)=1-u_i\), the
\(i\)-th summand equals \(u_i(2-\rho-u_i)^2\), and the bound \eqref{eq:gaussian_sqrt_d_bound}
follows immediately by maximizing the one-dimensional factor \(\sqrt u\,|2-\rho-u|\) over
\(u\in[0,1]\). It is sharp within the Gaussian family by choosing all eigenvalues \(\lambda_i\) so
that the corresponding \(u_i\)'s attain the maximizer of \(u\mapsto
\sqrt u\,|2-\rho-u|\).
\end{proof}

\begin{corollary}[Isotropic Gaussian specialization]
  Let \(\rho>0\). If
  \begin{equation*}
    q=\mathcal N(m,\varsigma^2\Idd), \qquad \varsigma\ge0,
  \end{equation*}
  then
  \begin{equation*}
    \left\| \mathcal A_\rho(\sigma,X_\sigma) \right\|_{\mathbb L_2}
    = \frac{\sqrt d}{\rho^2} \sigma^{2-2\rho} \frac{ \left| (2-\rho)\varsigma^2+(1-\rho)\sigma^2
    \right| }{ (\varsigma^2+\sigma^2)^{3/2} } .
  \end{equation*}
  In particular, for the native \(\sigma\)-clock \(\rho=1\),
  \begin{equation}\label{eq:isotropic_gaussian_rho_one}
    \left\| \mathcal A_1(\sigma,X_\sigma) \right\|_{\mathbb L_2}
    = \sqrt d\, \frac{\varsigma^2}{(\varsigma^2+\sigma^2)^{3/2}} .
  \end{equation}
\end{corollary}

The estimate of \Cref{thm:global_curvature_bound} is sharp in its dependence on the noise scale
\(\sigma\), and lossy in its dependence on the dimension \(d\). We take the two axes in turn.

\begin{proposition}[Sharpness of the acceleration exponent]
  \label{prop:acceleration_exponent_sharp}
  Let \(\rho>0\) and \(\sigma>0\). If \(\rho\ne1\) and \(q=\delta_m\) for some \(m\in\R^d\), then
  \begin{equation*}
    \left\| \mathcal A_\rho(\sigma,X_\sigma) \right\|_{\mathbb L_2}
    = \frac{|1-\rho|}{\rho^2}\sqrt d\,\sigma^{1-2\rho}.
  \end{equation*}
  If \(\rho=1\) and \(q_\sigma=\mathcal N(0,2\sigma^2\Idd)\), then
  \begin{equation}\label{eq:gaussian_rho_one_witness}
    \left\| \mathcal A_1(\sigma,X_\sigma) \right\|_{\mathbb L_2}
    = \frac{2}{3\sqrt3}\sqrt d\,\sigma^{-1}.
  \end{equation}
  Consequently, for every \(\rho>0\), no estimate with a better power of \(\sigma\) than
  \(\sigma^{1-2\rho}\) in \cref{eq:Aeta_global_bound} can hold uniformly over the class of all Borel
  probability measures \(q\).
\end{proposition}

\begin{proof}
First suppose \(\rho\ne1\), and take \(q=\delta_m\). Then
\begin{equation*}
  X_\sigma=m+\sigma Z,
\end{equation*}
and conditioning on \(X_\sigma\) reveals \(Z\) exactly. Hence
\begin{equation*}
  N_\sigma(X_\sigma)=Z, \qquad \mathcal C_\sigma(X_\sigma)=0,
  \qquad \mathfrak m_{3,\sigma}(X_\sigma)=0,
\end{equation*}
so that, by \cref{eq:Aeta_noise_moment_identity},
\begin{equation*}
  \mathcal A_\rho(\sigma,X_\sigma) = \frac{1-\rho}{\rho^2}\sigma^{1-2\rho}Z,
\end{equation*}
and the first claim follows from \(\norm{Z}_{\mathbb L_2}=\sqrt d\). Thus, whenever \(\rho\ne1\),
the power \(\sigma^{1-2\rho}\) is attained, with equality, by a Dirac law.

For \(\rho=1\), the Dirac law gives zero acceleration because the path \(m+\sigma Z\) is affine in
the native \(\sigma\)-clock. Nevertheless the same scale is still sharp uniformly over \(q\).
Indeed, for each \(\sigma>0\) and \(t>0\), take
\begin{equation*}
  q_\sigma=\mathcal N(0,t\sigma^2\Idd).
\end{equation*}
Then \cref{eq:isotropic_gaussian_rho_one} gives
\begin{equation*}
  \left\| \mathcal A_1(\sigma,X_\sigma) \right\|_{\mathbb L_2}
  = \sqrt d\,\sigma^{-1} \frac{t}{(1+t)^{3/2}} ,
\end{equation*}
and the factor \(t(1+t)^{-3/2}\) is maximized at \(t=2\), yielding
\cref{eq:gaussian_rho_one_witness}. In both cases the lower bound is a positive multiple of \(\sqrt
d\,\sigma^{1-2\rho}\) at every \(\sigma>0\), which proves the last claim.
\end{proof}

\paragraph{Minimax reading.}
This sharpness statement is minimax in \(q\). It should not be read as saying that every fixed
regular input law saturates the universal scale: for a fixed nondegenerate full-dimensional Gaussian
\(q=\mathcal N(m,\Sigma_0)\), with \(\Sigma_0>0\), the formula
\eqref{eq:gaussian_exact_acceleration_formula} gives \(\norm{\mathcal
A_\rho(\sigma,X_\sigma)}_{\mathbb L_2}=O(\sigma^{2-2\rho})\) as \(\sigma\to0\), one power of
\(\sigma\) better than the worst-case bound. The universal exponent is driven by singular laws, or
by laws whose intrinsic scale is comparable to the current noise level \(\sigma\).

\paragraph{Where the factor \(d\) is lost.}
The proof of \Cref{thm:global_curvature_bound} bounds the covariance and third-moment terms jointly,
with the constant \(4\) of \cref{eq:merged_L2}. To see where the dimension is lost, we isolate an
auxiliary estimate of the covariance term alone, which is the whole acceleration for the Gaussian
witness \cref{eq:gaussian_rho_one_witness}: there \(\mathfrak m_{3,\sigma}=0\), so at \(\rho=1\) the
noise-moment identity \eqref{eq:Aeta_noise_moment_identity} reduces to \(\mathcal
A_1(\sigma,\cdot)=\sigma^{-1}\mathcal C_\sigma N_\sigma\). With the notation of the proof of the
theorem, let \(R=\E[\norm{Z}^2\mid X_\sigma]\). By Jensen's inequality \(\norm{N}^2\le R\), and
\(\Tr\mathcal C=\E[\norm{W}^2\mid X_\sigma]=R-\norm{N}^2\). Since \(\mathcal C\) is positive
semidefinite, \(\norm{\mathcal CN}\le\norm{\mathcal C}_\op\norm{N}\le(\Tr\mathcal C)\norm{N}\), so
that \(\norm{\mathcal CN}^2\le(R-\norm{N}^2)^2\norm{N}^2\). For \(0\le u\le R\) the cubic
\(u\mapsto(R-u)^2u\) is maximized at \(u=R/3\), with value \(4R^3/27\); with \(\E
R^3\le\E\norm{Z}^6=d(d+2)(d+4)\), this gives
\begin{equation}\label{eq:CN_L2_improved}
  \norm{\mathcal C N}_{\mathbb L_2} \le \frac{2}{3\sqrt3} \sqrt{d(d+2)(d+4)} .
\end{equation}
The witness saturates every step of this estimate but one. It sits at the maximizer of the cubic:
for \(t=2\), \(\mathcal C_\sigma=\tfrac23\Idd\) and each coordinate of \(N_\sigma\) has variance
\(\tfrac13\), so \(\norm{N_\sigma}^2\) and \(R\) concentrate at \(d/3\) and \(d\), which is why the
constant \(\tfrac{2}{3\sqrt3}\) is the same; and the moment step is tight to leading order, both
sides being of order \(d^3\). The one lossy step is \(\norm{\mathcal C_\sigma}_\op\le\Tr\mathcal
C_\sigma\), which at \(t=2\) replaces \(\tfrac23\) by \(\tfrac{2d}3\): carrying the operator norm
through the chain returns the exact value \(\tfrac{2}{3\sqrt3}\sqrt d\,\sigma^{-1}\) of
\cref{eq:gaussian_rho_one_witness}, whereas the trace returns
\(\tfrac{2}{3\sqrt3}d^{3/2}\sigma^{-1}\), the estimate \eqref{eq:CN_L2_improved}.

\paragraph{General laws and an open question.}
This diagnosis is specific to the Gaussian family, where \(\mathfrak m_{3,\sigma}=0\) and the
third-moment term is absent; on its own it does not lower the universal rate. For general laws the
proof controls the covariance and third-moment terms together, through the contracted moment
\(\E[\norm{W}^2\norm{Z}\mid X_\sigma]\) of \cref{eq:merged_L2}, whose Cauchy--Schwarz step discards
the matrix structure of \(W\ip{W}{Z}\) just as the trace discards that of \(\mathcal C_\sigma\).
Bringing \(B_{d,\rho}\) below \(d^{3/2}\) would require controlling that contracted moment more
sharply.

The operator-norm/trace step is the price of assuming nothing whatever about \(q\). Under a
bounded-support assumption \(\operatorname{supp}q\subseteq B(m,r)\) one has \(\mathcal
C_\sigma=\sigma^{-2}\Cov(X_0\mid X_\sigma)\preceq(r/\sigma)^2\Idd\), so that \(\norm{\mathcal
C_\sigma N_\sigma}_{\mathbb L_2}\le(r/\sigma)^2\sqrt d\) and the \(\sqrt d\) rate is restored for
the covariance term, at the cost of a \(\sigma\)-dependent constant. Whether the true universal rate
is \(\Theta(\sqrt d)\), \(\Theta(d)\), or \(\Theta(d^{3/2})\)---in particular whether the
operator-norm/trace step can be avoided without a moment assumption on \(q\)---is left open; closing
it would directly sharpen the constant \(B_{d,\rho}\) of \Cref{thm:global_curvature_bound}. At the
CIFAR--10 dimension \(d=\auditDim\), the current constant is
\(B_{\auditDim,1}\approx6.8\times10^5\).

\paragraph{Relation to prior work: local bias and dimension.}
The posterior identities of \Cref{app:curvature} are Gaussian-posterior moment identities of the
kind used for higher-order Tweedie formulas \cite{efron_tweedie_2011,meng_high_order_2021}. They
bound the local bias for every data law, at the price of an explicit dimension dependence in
\(B_{d,\rho}\). Total-variation and Kullback--Leibler guarantees can adapt their dimension
dependence to low-dimensional structure in the data
\cite{de_bortoli_manifold_2022,li_yan_lowdim_2024,liang_huang_chen_lowdim_2025}; the sharp dimension
dependence of a universal bias bound is left open (\Cref{app:curvature}).

\section{EDM Preconditioning}
\label{app:preconditioning}

This appendix complements \Cref{sec:preconditioning_residual}. It reads the EDM margin as a
monotonicity property of the learned velocity, gives the noise threshold above which a Lipschitz
envelope certifies contraction, and checks the high-noise contraction for the oracle denoiser.

\subsection{Complements on the EDM Margin}

The margin \eqref{eq:def_b_k_edm} has an infinitesimal reading. By
\cref{eq:edm_learned_denoiser_step},
\(\vhat_{\sigma_j}(x)=\sigma_j^{-1}\bigl[(1-\alpha_j)x-\beta_j\Fhat_j(c_jx)\bigr]\), and the
Lipschitz bound on \(\Fhat_j\) yields, for all \(x,y\in\R^d\),
\begin{equation}\label{eq:edm_velocity_monotonicity}
  \ip{x-y}{\vhat_{\sigma_j}(x)-\vhat_{\sigma_j}(y)} \ge \frac{b_j}{\sigma_j}\norm{x-y}^2 :
\end{equation}
the learned velocity is \(b_j/\sigma_j\)-monotone, although \(b_j/\sigma_j\), built from a
worst-case envelope, need not be its best monotonicity modulus. Monotonicity alone would give a
weaker finite-step factor, since squaring the predictor displacement
\(\Phihat_j=\Id-\ell_j\vhat_{\sigma_j}\), with \(\ell_j=\sigma_j-\sigma_{j+1}\), produces a positive
term of order \(\ell_j^2\); the affine EDM representation avoids that loss.

If \(\LF{j}\le\Lhi\) at every level above a noise threshold \(\sighi\), then \(b_j\ge\bhi\) at these
levels as soon as
\begin{equation*}
  \frac{\sigma_j}{\sigmad} \ge \frac{\Lhi+\sqrt{\Lhi^2+4\bhi(1-\bhi)}}{2(1-\bhi)}.
\end{equation*}
Any uniform normalized Lipschitz envelope therefore provides a sufficiently large admissible
threshold \(\sighi\) for the high-noise block. More generally, writing \(L_F(\sigma)\) for the
envelope at noise level \(\sigma\), \(L_F(\sigma)=o(\sigma)\) gives \(b_j\to1\) as
\(\sigma\to\infty\), and the certified contraction approaches the damping by the noise ratio.

\subsection{High-Noise Contraction for the Oracle}
\label{subsec:oracle_contraction}

For the oracle, the high-noise contraction of \Cref{ass:high_noise_contraction} is part of the
hypotheses and is not automatic: \(\nabla D_\sigma(x)=\sigma^{-2}\Cov(X_0\mid X_\sigma=x)\), whose
supremum over \(x\) can be infinite when \(q\) has unbounded support. If
\(\operatorname{supp}q\subseteq B(0,R)\), then \(0\preceq\nabla D_\sigma\preceq(R/\sigma)^2\Idd\),
so \(\operatorname{Lip}(\Phi_j)\le1-a_j(1-R^2/\sigma_j^2)\) and the assumption holds with margin
\(\bhi\) above \(\sighi=R/\sqrt{1-\bhi}\).

\paragraph{Relation to prior work: contraction through denoiser regularity.}
\textcite{li_zeng_stein_stabilization_2026} study deterministic diffusion updates whose contraction
certificates fail, measure their local Lipschitz behavior, and propose an inference-time correction
based on Stein identities. Their first-order certificate includes the condition
\(\operatorname{Lip}(\Dhat_\sigma)<1\) in the EDM signal-scaling convention \cite[Sec.~4.1 and
App.~C.1]{li_zeng_stein_stabilization_2026}. In our parametrization, this is the condition that
makes the predictor contractive: by \cref{eq:learned_edm_predictor},
\(\Phihat_j=(1-a_j)\Id+a_j\Dhat_{\sigma_j}\), hence
\(\operatorname{Lip}(\Phihat_j)\le1-a_j\bigl(1-\operatorname{Lip}(\Dhat_{\sigma_j})\bigr)\). The EDM
preconditioning turns it into a condition on the raw network, since
\(\operatorname{Lip}(\Dhat_{\sigma_j})\le\alpha_j+\beta_jc_j\LF{j}=1-b_j<1\) whenever
\(\LF{j}<\sigma_j/\sigmad\), by \cref{eq:edm_learned_denoiser_step,eq:def_b_k_edm}. Our analysis
uses this certificate only at high noise, applies to the uncorrected predictor, and replaces
worst-case certificates at low noise by the amplification factor \(\gamma_j\).

\section{Propagation}
\label{app:propagation}

This appendix complements \Cref{sec:two_regime_stability}. It details the split index, proves the
directional and field-level bounds and their hierarchy, reads the directional rate geometrically,
and proves the strict separation between the realized amplification and the worst-case majorants.

\subsection{The Split Index}

With the convention \(\Khi=0\) when \(\sigma_0<\sighi\), the split index \cref{eq:def_split_index}
is the last level at or above \(\sighi\). The high-noise block is the set of steps
\(j=0,\ldots,\Khi-1\), each of which ends at \(\sigma_{j+1}\ge\sighi\); the low-noise block is
\(j=\Khi,\ldots,K-1\), so the step \(\sigma_{\Khi}\to\sigma_{\Khi+1}\) that crosses the threshold
belongs to the low-noise block (\Cref{fig:split_index}). Either block may be empty, with the
conventions that empty sums vanish and empty products equal one: \(\Khi=0\) when no step ends at or
above the threshold, and \(\Khi=K\) when \(\sigma_K\ge\sighi\).

The threshold plays no role in \Cref{def:schedule_family}: the family \(\Gsched\) is fixed first,
and each member \(G_K\) then acquires its own split index \(\Khi(K)\). When
\(\underline\sigma<\sighi\le\overline\sigma\), the nested grids give
\(\sigma_{\Khi(K)}\downarrow\sighi\) and
\(\Khi(K)/K\to(\overline\eta-\sighi^\rho)/(\overline\eta-\underline\eta)\) as \(K\to\infty\), so
each block occupies a fixed asymptotic fraction of the schedule. Otherwise one block is empty for
every \(K\). In \Cref{ass:high_noise_contraction} the inequality needs no convention at \(e_j=0\),
where both sides vanish.

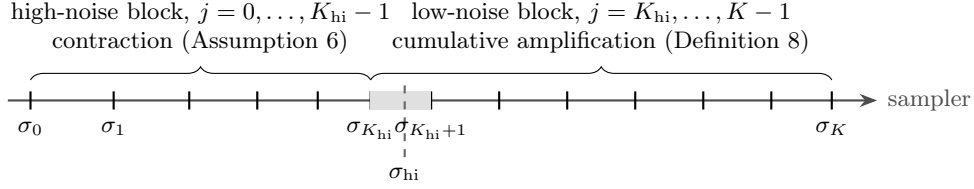
\begin{figure}[t]
  \centering
  \begin{tikzpicture}[x=1cm,y=1cm,>=Stealth]
  \def\levels{0,1.1,2.1,3.0,3.8,4.5,5.3,6.2,7.1,8.0,8.9,9.8,10.6}
  \draw[->,thick,black!70] (-0.3,0) -- (11.2,0) node[right,font=\footnotesize] {sampler};
  \foreach \x in \levels { \draw[thick] (\x,-0.12) -- (\x,0.12); }
  \fill[black!12] (4.5,-0.12) rectangle (5.3,0.12);
  \draw[dashed,thick,black!60] (4.95,-0.75) -- (4.95,0.3);
  \node[font=\footnotesize,below] at (4.95,-0.75) {\(\sighi\)};
  \node[font=\footnotesize,below] at (0,-0.15) {\(\sigma_0\)};
  \node[font=\footnotesize,below] at (1.1,-0.15) {\(\sigma_1\)};
  \node[font=\footnotesize,below] at (4.5,-0.15) {\(\sigma_{\Khi}\)};
  \node[font=\footnotesize,below] at (5.3,-0.15) {\(\sigma_{\Khi+1}\)};
  \node[font=\footnotesize,below] at (10.6,-0.15) {\(\sigma_K\)};
  \draw[decorate,decoration={brace,amplitude=5pt,raise=2pt}] (0,0.2) -- (4.5,0.2)
    node[midway,above=8pt,font=\footnotesize,align=center]
    {high-noise block, \(j=0,\ldots,\Khi-1\)\\ contraction (\Cref{ass:high_noise_contraction})};
  \draw[decorate,decoration={brace,amplitude=5pt,raise=2pt}] (4.5,0.2) -- (10.6,0.2)
    node[midway,above=8pt,font=\footnotesize,align=center]
    {low-noise block, \(j=\Khi,\ldots,K-1\)\\ cumulative amplification (\Cref{def:low_noise_budget})};
  \end{tikzpicture}
  \caption[Split index and the two blocks]{The split index \(\Khi\) of \cref{eq:def_split_index} is the last level at or above the
  threshold \(\sighi\). The step \(\sigma_{\Khi}\to\sigma_{\Khi+1}\) that crosses the threshold
  (shaded) is the first step of the low-noise block, so every high-noise step ends at or above
  \(\sighi\).}
  \label{fig:split_index}
\end{figure}

\subsection{Low-Noise Majorants: Proofs}

\Cref{fig:low_noise_hierarchy} summarizes the majorants of \(\Lambda_K\) and their order.

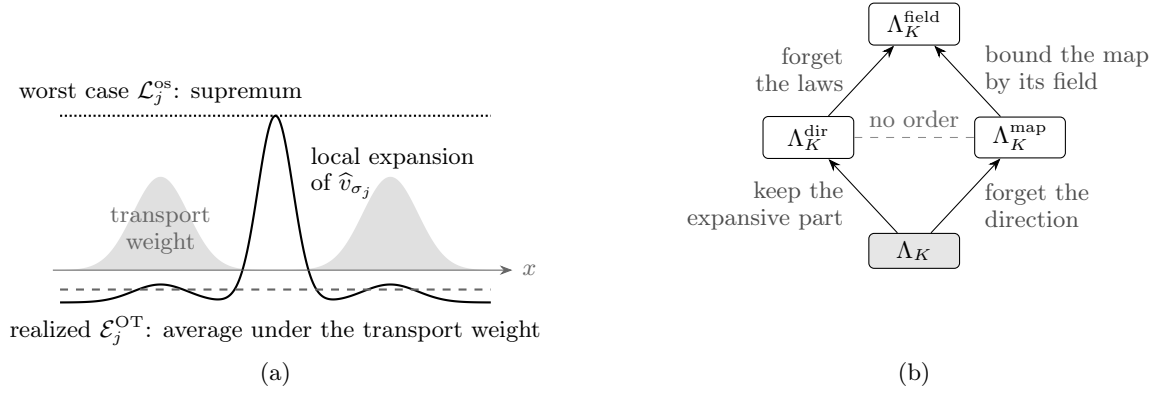
\begin{figure}[t]
  \centering
  \begin{minipage}[b]{0.52\textwidth}
    \centering
    \begin{tikzpicture}[x=0.95cm,y=0.95cm,>=Stealth,font=\footnotesize]
  \fill[black!12] plot[domain=-3:3,samples=121,smooth]
    (\x,{1.3*exp(-(\x+1.6)^2/0.3)+1.3*exp(-(\x-1.6)^2/0.3)}) -- (3,0) -- (-3,0) -- cycle;
  \draw[thick] plot[domain=-3:3,samples=161,smooth]
    (\x,{2.6*exp(-(\x)*(\x)/0.12)-0.45+0.25*exp(-(\x+1.6)^2/0.3)+0.25*exp(-(\x-1.6)^2/0.3)});
  \draw[->,black!60] (-3.1,0) -- (3.3,0) node[right] {\(x\)};
  \draw[densely dotted,thick] (-3,2.15) -- (3,2.15);
  \node[above,align=center] at (-1.6,2.15) {worst case \(\Los{j}\): supremum};
  \draw[dashed,thick,black!60] (-3,-0.27) -- (3,-0.27);
  \node[below,align=center] at (0,-0.5) {realized \(\Erate[\OT]{j}\): average under the transport weight};
  \node[align=left,right] at (0.35,1.35) {local expansion\\[-1pt]of \(\vhat_{\sigma_j}\)};
  \node[black!60,align=center] at (-1.6,0.55) {transport\\[-1pt]weight};
  \node[below=1pt] at (0,-1.1) {(a)};
    \end{tikzpicture}
  \end{minipage}%
  \hfill
  \begin{minipage}[b]{0.45\textwidth}
    \centering
    \begin{tikzpicture}[x=1cm,y=1cm,>=Stealth,font=\footnotesize,
    lev/.style={draw,rounded corners=2pt,inner sep=3pt,minimum width=1.2cm}]
  \node[lev,fill=black!10] (L) at (0,0)    {\(\Lambda_K\)};
  \node[lev]               (D) at (-1.4,1.5) {\(\Lambda_K^{\mathrm{dir}}\)};
  \node[lev]               (M) at (1.4,1.5)  {\(\Lambda_K^{\mathrm{map}}\)};
  \node[lev]               (F) at (0,3)    {\(\Lambda_K^{\mathrm{field}}\)};
  \draw[->] (L) -- (D);
  \draw[->] (L) -- (M);
  \draw[->] (D) -- (F);
  \draw[->] (M) -- (F);
  \draw[dashed,black!45] (D) -- node[above,black!60] {no order} (M);
  \node[left,align=right,black!70] at (-0.8,0.6) {keep the\\[-1pt]expansive part};
  \node[right,align=left,black!70] at (0.8,0.6) {forget the\\[-1pt]direction};
  \node[left,align=right,black!70] at (-0.8,2.4) {forget\\[-1pt]the laws};
  \node[right,align=left,black!70] at (0.8,2.4) {bound the map\\[-1pt]by its field};
  \node[below=1pt] at (0,-1.3) {(b)};
    \end{tikzpicture}
  \end{minipage}
  \caption[Realized transport and worst-case bounds at low noise]{Schematic, not to scale.
  (a)~Along one direction, the local expansion rate of the learned velocity (solid) peaks between
  the two modes of the law. The shaded density, with one lobe per mode, is the squared-displacement
  occupation of the optimal-coupling segments between \(\mu_j\) and \(\widehat\mu_j\), which is the
  weight in \cref{eq:transport_directional_expansion_rate}: the transport moves mass within each
  mode, not across the gap where the expansion peaks. The global one-sided rate \(\Los{j}\) is the
  supremum of the expansion rate (dotted); the directional rate \(\Erate[\OT]{j}\) is its average
  under this weight (dashed), and stays small. In higher dimension the supremum also runs over all
  directions. (b)~Majorants of \(\Lambda_K\); an arrow from \(A\) to \(B\) means \(A\le B\), and
  each arrow discards the information named beside it. The directional and map-level majorants are
  not ordered, and \Cref{prop:strict_law_map_separation} shows that the chain cannot be reversed.}
  \label{fig:low_noise_hierarchy}
\end{figure}

\begin{proof}[Proof of \Cref{prop:law_level_sufficiency}]
Since \(\Phihat_j(x)=x-\ell_j\vhat_{\sigma_j}(x)\), for all \(x,y\in\R^d\),
\begin{equation}\label{eq:predictor_secant_identity}
  \norm{\Phihat_j(x)-\Phihat_j(y)}^2 =\norm{x-y}^2
  -2\ell_j\ip{x-y}{\vhat_{\sigma_j}(x)-\vhat_{\sigma_j}(y)}
  +\ell_j^2\norm{\vhat_{\sigma_j}(x)-\vhat_{\sigma_j}(y)}^2.
\end{equation}
If \(e_j=0\), then \(\gamma_j=1\) by \cref{eq:def_raw_amplification} and
\(\Epos[\OT]{j}=\Qsec[\OT]{j}=0\) by convention, so \cref{eq:directional_gamma_chain} and
\cref{eq:law_level_stepwise_log} hold trivially. Let \(e_j>0\), so that
\(\E\norm{U_j-V_j}^2=e_j^2\), and define the \emph{signed} directional expansion rate
\begin{equation}\label{eq:def_signed_directional_rate}
  \Erate[\OT]{j} =\frac{\E\ip{V_j-U_j}{\mathsf v_j}}{e_j^2},
  \qquad\text{so that}\qquad \Epos[\OT]{j}=\bigl(\Erate[\OT]{j}\bigr)_+ .
\end{equation}
Pushing the optimal coupling \((U_j,V_j)\) through \(\Phihat_j\) gives a coupling of the two
pushforwards, so \(\gamma_j^2e_j^2\le\E\norm{\Phihat_j(U_j)-\Phihat_j(V_j)}^2\). Taking expectations
in \cref{eq:predictor_secant_identity} at \((x,y)=(U_j,V_j)\) and dividing by \(e_j^2\),
\begin{equation}\label{eq:signed_gamma_bound}
  \gamma_j^2 \le \frac{\E\norm{\Phihat_j(U_j)-\Phihat_j(V_j)}^2}{e_j^2}
  =1+2\ell_j\Erate[\OT]{j}+\ell_j^2\Qsec[\OT]{j} \le1+2\ell_j\Epos[\OT]{j}+\ell_j^2\Qsec[\OT]{j},
\end{equation}
which is \cref{eq:directional_gamma_chain}. Its right-hand side is at least one, so it also bounds
\(\gambar_j^2=\max\{1,\gamma_j^2\}\); then \(\log\gambar_j=\tfrac12\log\gambar_j^2\) and
\(\log(1+x)\le x\) give \cref{eq:law_level_stepwise_log}. Summing over the low-noise block gives
\cref{eq:law_level_summed_condition}.
\end{proof}

The signed rate may be negative, in which case the step contracts the realized discrepancy at first
order; only its positive part enters \(\Lambda_K^{\mathrm{dir}}\). At a step with \(e_j=0\) the two
laws coincide, \(\gamma_j=1\), and every stepwise bound on \(\log\gambar_j\) holds trivially; the
signed rate, and the bounds on \(\gamma_j\) itself, are stated only for \(e_j>0\).

The hierarchy \cref{eq:law_map_field_budget_hierarchy} rests on the following certificate.

\begin{corollary}[Field-level certificate]
  \label{cor:field_level_certificate}
  Assume that \(\operatorname{Lip}(\vhat_{\sigma_j})<\infty\) at every low-noise step, and let
  \(\Los{j}\) be the global one-sided rate \eqref{eq:global_one_sided_expansion_rate}. Then
  \(|\Los{j}|\le\operatorname{Lip}(\vhat_{\sigma_j})\), and \(\Los{j}\) is the smallest constant
  such that, for all \(x,y\in\R^d\),
  \begin{equation}\label{eq:field_one_sided_condition}
    -\ip{\vhat_{\sigma_j}(x)-\vhat_{\sigma_j}(y)}{x-y} \le\Los{j}\norm{x-y}^2.
  \end{equation}
  At every low-noise step, \(\Qsec[\OT]{j}\le\operatorname{Lip}(\vhat_{\sigma_j})^2\) and
  \(\Epos[\OT]{j}\le(\Los{j})_+\). Consequently, with \(\Lambda_K^{\mathrm{dir}}\),
  \(\Lambda_K^{\mathrm{map}}\) and \(\Lambda_K^{\mathrm{field}}\) as in
  \cref{eq:law_level_summed_condition,eq:map_gamma_bound,eq:field_level_summed_condition}, both
  chains of \cref{eq:law_map_field_budget_hierarchy} hold:
  \begin{equation*}
    \Lambda_K\le\Lambda_K^{\mathrm{dir}}\le\Lambda_K^{\mathrm{field}},
    \qquad \Lambda_K\le\Lambda_K^{\mathrm{map}}\le\Lambda_K^{\mathrm{field}}.
  \end{equation*}
  These bounds remain valid if each \(\Los{j}\) is replaced by any larger constant, in particular by
  any constant satisfying \cref{eq:field_one_sided_condition}, both in the stepwise bounds and in
  \(\Lambda_K^{\mathrm{field}}\).
\end{corollary}

The certificate follows from a stepwise chain, which also orders the directional, map and field
bounds on a single step.

\begin{corollary}[Field-level certificate, stepwise form]
  \label{cor:field_level_certificate_full}
  Under the hypotheses of \Cref{cor:field_level_certificate}, fix a low-noise step. Then
  \(\Qsec[\OT]{j}\le\operatorname{Lip}(\vhat_{\sigma_j})^2\), \(\Epos[\OT]{j}\le(\Los{j})_+\), and
  if \(e_j>0\) also \(\Erate[\OT]{j}\le\Los{j}\) and
  \begin{align}
    \gamma_j^2 &\le 1+2\ell_j\Erate[\OT]{j}+\ell_j^2\Qsec[\OT]{j}\nonumber \\
    &\le \operatorname{Lip}(\Phihat_j)^2
    \le 1+2\ell_j\Los{j}+\ell_j^2\operatorname{Lip}(\vhat_{\sigma_j})^2 .
    \label{eq:field_gamma_chain}
  \end{align}
  Moreover, at every low-noise step,
  \begin{align}
    \log\gambar_j &\le \ell_j\Epos[\OT]{j}+\tfrac12\ell_j^2\Qsec[\OT]{j}
    \le \ell_j\bigl(\Los{j}\bigr)_+
    +\tfrac12\ell_j^2\operatorname{Lip}(\vhat_{\sigma_j})^2, \label{eq:field_level_stepwise_log} \\
    \log\gambar_j &\le \log\max\bigl\{1,\operatorname{Lip}(\Phihat_j)\bigr\}
    \le \ell_j\bigl(\Los{j}\bigr)_+
    +\tfrac12\ell_j^2\operatorname{Lip}(\vhat_{\sigma_j})^2. \label{eq:map_level_stepwise_log}
  \end{align}
\end{corollary}

\paragraph{Proof of the field-level certificate.}

\begin{proof}[Proof of \Cref{cor:field_level_certificate,cor:field_level_certificate_full}]
By Cauchy--Schwarz, the quotient in \cref{eq:global_one_sided_expansion_rate} lies in
\([-\operatorname{Lip}(\vhat_{\sigma_j}),\allowbreak\operatorname{Lip}(\vhat_{\sigma_j})]\), so its supremum
\(\Los{j}\) is finite with \(|\Los{j}|\le\operatorname{Lip}(\vhat_{\sigma_j})\); being a supremum,
it satisfies \cref{eq:field_one_sided_condition} and no smaller constant does. If \(e_j>0\), then
\(\E\norm{U_j-V_j}^2=e_j^2\). Applying \cref{eq:field_one_sided_condition} at \((x,y)=(U_j,V_j)\)
and taking expectations gives \(\Erate[\OT]{j}\le\Los{j}\), hence \(\Epos[\OT]{j}\le(\Los{j})_+\);
the Lipschitz bound gives \(\E\norm{\mathsf v_j}^2\le\operatorname{Lip}(\vhat_{\sigma_j})^2e_j^2\),
that is, \(\Qsec[\OT]{j}\le\operatorname{Lip}(\vhat_{\sigma_j})^2\). When \(e_j=0\), the conventions
give \(\Qsec[\OT]{j}=\Epos[\OT]{j}=0\), so these two bounds remain valid.

In \cref{eq:field_gamma_chain}, the first inequality and the identification of the middle member
with \(\E\norm{\Phihat_j(U_j)-\Phihat_j(V_j)}^2/e_j^2\) are \cref{eq:signed_gamma_bound}; this
quotient is at most \(\operatorname{Lip}(\Phihat_j)^2\), which is the second inequality. For the
third, \cref{eq:field_one_sided_condition} and the Lipschitz bound on \(\vhat_{\sigma_j}\) bound the
right-hand side of \cref{eq:predictor_secant_identity} by
\(\bigl(1+2\ell_j\Los{j}+\ell_j^2\operatorname{Lip}(\vhat_{\sigma_j})^2\bigr)\norm{x-y}^2\) for
every pair \((x,y)\). The restriction to \(e_j>0\) matters only for the first two members: at
\(e_j=0\) the convention \(\gamma_j=1\) is unrelated to \(\operatorname{Lip}(\Phihat_j)\), which may
be smaller than one.

In \cref{eq:field_level_stepwise_log}, the first inequality is \cref{eq:law_level_stepwise_log} and
the second follows from the two bounds just proved. In \cref{eq:map_level_stepwise_log}, the first
inequality follows from \(\gamma_j\le\operatorname{Lip}(\Phihat_j)\) when \(e_j>0\)
(\cref{eq:field_gamma_chain}) and from \(\gamma_j=1\) when \(e_j=0\). For the second, the last
inequality of \cref{eq:field_gamma_chain}, whose proof does not use \(e_j\), and
\(\Los{j}\le(\Los{j})_+\) give \(\max\{1,\operatorname{Lip}(\Phihat_j)^2\}
\le1+2\ell_j(\Los{j})_++\ell_j^2\operatorname{Lip}(\vhat_{\sigma_j})^2\); taking half the logarithm
and using \(\log(1+x)\le x\) concludes.

Summing over the low-noise block, \cref{eq:law_level_summed_condition} and the second inequality of
\cref{eq:field_level_stepwise_log} give
\(\Lambda_K\le\Lambda_K^{\mathrm{dir}}\le\Lambda_K^{\mathrm{field}}\), and the two inequalities of
\cref{eq:map_level_stepwise_log} give
\(\Lambda_K\le\Lambda_K^{\mathrm{map}}\le\Lambda_K^{\mathrm{field}}\). Replacing \(\Los{j}\) by a
larger constant only enlarges the right-hand sides of
\cref{eq:field_gamma_chain,eq:field_level_stepwise_log,eq:map_level_stepwise_log} and
\(\Lambda_K^{\mathrm{field}}\).
\end{proof}

\paragraph{The hierarchy.}
The directional and map-level majorants are not ordered in general, because each discards something
the other keeps: the directional sum keeps the direction of the displacement but has already dropped
signed contraction (through the positive part) and relaxed the logarithm, whereas the map sum keeps
the logarithm but tests every direction. For example, for \(\vhat(x)=-x\), one step contributes
\(\ell+\ell^2/2\) to the directional sum but only \(\log(1+\ell)\) to the map sum. Before these two
relaxations, the squared amplification is ordered, as \cref{eq:field_gamma_chain} shows.

\subsection{Geometric Reading of the Directional Rate}

Under the additional assumption \(\vhat_{\sigma_j}\in C^1(\R^d;\R^d)\), since \(\mu_j\) is
absolutely continuous, Brenier's theorem \cite[Theorem~2.12]{villani_topics_2003} allows the choice
\(V_j=T_j(U_j)\) for an optimal map \(T_j\). With \(\xi_j(x)=T_j(x)-x\), for \(e_j>0\), the signed
rate \eqref{eq:def_signed_directional_rate} reads
\begin{equation}\label{eq:transport_directional_expansion_rate}
  \Erate[\OT]{j} =
  -\frac{ \int_0^1\!\int \xi_j(x)^\top
  \operatorname{Sym}\nabla\vhat_{\sigma_j}\bigl(x+t\xi_j(x)\bigr) \xi_j(x)\,\mu_j(\dd x)\,\dd t }{
  \int\norm{\xi_j(x)}^2\,\mu_j(\dd x) }.
\end{equation}
The directional rate is therefore an energy-weighted average of Rayleigh quotients of
\(-\operatorname{Sym}\nabla\vhat_{\sigma_j}\) along the segments actually travelled by the
transport, whereas \cref{eq:global_one_sided_expansion_rate} is a supremum over all points and
directions. The directional rate may be much smaller when the displacement avoids the most expansive
eigendirections. Evaluated on finitely many sampled segments, the local quotients compare the
displacement with the worst direction at the same points; they are an empirical diagnostic, not a
certificate of \cref{eq:global_one_sided_expansion_rate}, which is how the CIFAR--10 probes of
\Cref{subsec:evidence_directional_cifar} must be read. At high noise the worst case is already
favorable: by \cref{eq:edm_velocity_monotonicity}, \(\Los{j}\le-b_j/\sigma_j<0\) whenever the EDM
margin is positive, so following the realized displacement pays off only below the threshold.

\begin{remark}[An integrable one-sided rate bounds \(\Lambda_K\)]
  \label{rem:integrable_rate}
  The field-level majorant separates a mesh-free rate from a finite-step remainder, and gives the
  one sufficient condition of this paper for \(\Lambda_K\) to stay bounded along \(\Gsched\). Let
  \(\ell_{\max,K}=\max_j\ell_j\). If \(\sighi\) lies in the interior of the noise range and
  \(\sigma\mapsto(\Los{}(\sigma))_+\) and \(\sigma\mapsto\operatorname{Lip}(\vhat_\sigma)^2\) are
  Riemann integrable, the first-order part of \(\Lambda_K^{\mathrm{field}}\) converges to
  \(\int_{\underline\sigma}^{\sighi}(\Los{}(\sigma))_+\,\dd\sigma\). The quadratic part is at most
  \(\tfrac12\ell_{\max,K}\sum_j\ell_j\operatorname{Lip}(\vhat_{\sigma_j})^2\to0\), so integrability
  of the one-sided rate yields a majorant, hence a \(\Lambda_K\), bounded uniformly along
  \(\Gsched\). For the oracle field, \(-\nabla v_\sigma=\sigma\nabla^2\log p_\sigma\) by
  \cref{eq:velocity}, so both rates are governed by the small-noise behaviour of the log-Hessian of
  the smoothed data law; for laws supported on manifolds with corners, this behaviour is described
  by \textcite{brosse_boundary_layer_2026}. The condition is worst-case in space; \Cref{prop:strict_law_map_separation} shows how
  much it can overstate.
\end{remark}

\paragraph{Relation to prior work: stability from the structure of the data law.}
Like Gao--Zhu, most Wasserstein guarantees obtain stability from structure of the data law or of the
forward dynamics: strong log-concavity \cite{gao_nguyen_zhu_wasserstein_2025}, weak log-concavity,
which the forward smoothing strengthens along the path
\cite{gentiloni_ocello_beyond_2025,stephanovitch_lipschitz_2026,kremling_weak_logconcavity_2025}, or
semiconvexity with strong convexity at infinity \cite{bruno_sabanis_semiconvexity_2025}. These
assumptions yield one-sided or global Lipschitz control of the flow between any two points. Our
directional condition is also one-sided, but it concerns the learned velocity, and only along the
optimal displacement between the exact and sampler laws: it is a property of the realized transport,
not a consequence of a global property of the data law.

\paragraph{Relation to prior work: deterministic samplers and Jacobian control.}
Non-asymptotic guarantees for reverse-SDE samplers are commonly stated in total variation or
Kullback--Leibler divergence
\cite{chen_sampling_2023,lee_convergence_2023,chen_improved_2023,benton_nearly_2024,
conforti_kl_2025}. These analyses typically compare path measures through Girsanov's theorem, so the
score error is evaluated along the exact process and no propagation factor through the learned
dynamics appears. The probability-flow ODE admits no such change of measure, and guarantees for
deterministic samplers require additional control. \textcite{chen_probability_flow_fast_2023} obtain
polynomial guarantees by interleaving the deterministic flow with stochastic corrector steps.
\textcite{li_nonasymptotic_2024} analyze the deterministic sampler directly in total variation,
assuming, in addition to \(\mathbb L_2\) score accuracy, control of the error on the score Jacobian;
sharp total-variation analysis of probability-flow discretization likewise assumes both score and
score-Jacobian accuracy \cite{li_sharp_probability_flow_2024}.
\textcite{chen_daras_dimakis_ddim_2023} give polynomial guarantees for DDIM-type deterministic
samplers, the family of the first-order EDM predictor, by decomposing each step into a restoration
and a degradation step. In our \(W_2\) setting, this additional requirement is carried by the
propagation factor \(\gamma_j\), and Jacobian-type conditions play the role of its field-level
majorants \cref{eq:field_gamma_chain}.
\textcite{benton_flow_matching_2024} bound the \(W_2\) error of flow-matching ODEs by a
Gr\"onwall argument under Lipschitz control of the learned field, which is the worst-case route of
\Cref{cor:field_level_certificate}; \textcite{kwon_fan_lee_wasserstein_2022} likewise bound the
\(W_2\) error of the probability-flow ODE by the score-matching objective under Lipschitz conditions
on the learned score. \textcite{arsenyan_assessing_2025} obtain \(W_2\) guarantees for the
stochastic DDPM sampler that are robust to noisy score evaluations, with rates matching the Gaussian
case. Other recent Wasserstein analyses of diffusion samplers include
\textcite{koike_follmer_wasserstein_2026,lyu_load_reserve_2026,tang_decentralized_ode_2026}.

\subsection{The Strict Separation from Worst-Case Stability}
\label{app:law_map_separation}

The proof of \Cref{prop:strict_law_map_separation} is given below, together with the exact limit;
the constant \(C\) grows with \(\overline\sigma/\underline\sigma\) and is large on the EDM noise
range (\Cref{rem:separation_cost}). The divergence is not a mesh-refinement instability: for each
fixed \(\varepsilon\), both worst-case majorants have a finite limit in \(K\). They grow like
\(\varepsilon^{-1}\) across perturbations whose amplitude vanishes while their slope grows, whereas
\(\Lambda_K\) stays identically zero. The map and field majorants agree at first order
(\Cref{lem:map_field_first_order}), so the loss lies between the laws and the worst-case majorants,
not between map and field. \Cref{app:law_map_separation} also gives a fixed construction with
\(\Lambda_K=0\) in which every predictor has infinite Lipschitz constant and every one-sided rate is
infinite (\Cref{prop:fixed_infinite_separation}).

The mechanism is the translated initialization: it keeps the discrepancy large relative to the
perturbation amplitude, so that the Gaussian contraction absorbs the perturbation. The proof bounds
the cost of the optimal input coupling pushed through the predictor, without reoptimizing the output
coupling. The construction therefore shows an advantage of realized displacements over worst-case
regularity, but not the same separation with exact initialization and purely discretization-induced
error. The mixture experiment of \Cref{subsec:evidence_hierarchy} examines this non-adversarial
counterpart.

\begin{lemma}[Map and field agree at first order]
  \label{lem:map_field_first_order}
  Fix a step \(j\), assume \(\operatorname{Lip}(\vhat_{\sigma_j})<\infty\), and let \(\Los{j}\in\R\)
  be the smallest constant in \cref{eq:field_one_sided_condition}. Then
  \begin{equation}\label{eq:field_below_map_first_order}
    1+\ell_j\Los{j} \le \operatorname{Lip}(\Phihat_j)
    \le \sqrt{1+2\ell_j\Los{j}+\ell_j^2\operatorname{Lip}(\vhat_{\sigma_j})^2}.
  \end{equation}
\end{lemma}

\begin{proof}
The smallest constant exists and is finite, since by Cauchy--Schwarz the quotient below lies in
\([-\operatorname{Lip}(\vhat_{\sigma_j}),\operatorname{Lip}(\vhat_{\sigma_j})]\). The upper bound is
the last inequality of \cref{eq:field_gamma_chain}, whose proof does not use \(e_j\). For the lower
bound, take pairs \(x\ne y\) along which
\(-\ip{\vhat_{\sigma_j}(x)-\vhat_{\sigma_j}(y)}{x-y}/\norm{x-y}^2\) approaches \(\Los{j}\). By
Cauchy--Schwarz and \(\Phihat_j(x)=x-\ell_j\vhat_{\sigma_j}(x)\),
\begin{equation*}
  \frac{\norm{\Phihat_j(x)-\Phihat_j(y)}}{\norm{x-y}}
  \ge \frac{\ip{\Phihat_j(x)-\Phihat_j(y)}{x-y}}{\norm{x-y}^2}
  = 1-\ell_j\frac{\ip{\vhat_{\sigma_j}(x)-\vhat_{\sigma_j}(y)}{x-y}}{\norm{x-y}^2},
\end{equation*}
whose right-hand side approaches \(1+\ell_j\Los{j}\).
\end{proof}

When \(\Los{j}\ge0\), the lower bound gives \(\operatorname{Lip}(\Phihat_j)\ge1\), and
\(x-x^2/2\le\log(1+x)\le x\) for \(x\ge0\), together with
\(\Los{j}\le\operatorname{Lip}(\vhat_{\sigma_j})\), turns \cref{eq:field_below_map_first_order} into
\begin{equation*}
  \bigl|\log\operatorname{Lip}(\Phihat_j)-\ell_j\Los{j}\bigr|
  \le\tfrac12\ell_j^2\operatorname{Lip}(\vhat_{\sigma_j})^2 .
\end{equation*}
The step contribution \(\log\max\{1,\operatorname{Lip}(\Phihat_j)\}\) to
\(\Lambda_K^{\mathrm{map}}\) and the step contribution
\(\ell_j(\Los{j})_++\tfrac12\ell_j^2\operatorname{Lip}(\vhat_{\sigma_j})^2\) to
\(\Lambda_K^{\mathrm{field}}\) therefore share the first-order term \(\ell_j\Los{j}\) and differ by
at most \(\ell_j^2\operatorname{Lip}(\vhat_{\sigma_j})^2\).

Throughout the rest of the appendix, set \(\ell_{\mathrm{tot}}=\overline\sigma-\underline\sigma\)
and \(\vartheta=\underline\sigma/\overline\sigma\).

\begin{proof}[Proof of \Cref{prop:strict_law_map_separation}]
The construction perturbs the Gaussian denoiser by an oscillation of amplitude \(\sigma\varepsilon\)
and slope of order \(\varepsilon^{-1}\), and starts the sampler from a translate of
\(\mu_{\overline\sigma}\) by \(R_\star\varepsilon\). Part (a) follows from the amplitude of the
perturbation. For (b), a lower bound on the mean of the sampler law keeps the discrepancy
\(e_j^\varepsilon\) of order \(\varepsilon\), and pushing an optimal coupling through the predictor
then bounds \(\gamma_j^\varepsilon\) by one. For (c), the Lipschitz constants and one-sided rates
are computed exactly, and the majorants become Riemann sums with explicit limits. Parts (a) and (b),
and the exact stepwise formulas for these constants, hold for every \(\varepsilon>0\); only the
limits in (c) use \(\varepsilon<2\sigmad\). The schedules are the grids \(G_K\in\Gsched\) of
\Cref{def:schedule_family}; for each \(\varepsilon\), the denoiser family and the initialization
below are fixed across refinements. Since the construction is the same on every \(G_K\in\Gsched\),
we suppress \(K\) from stepwise objects, writing for instance \(\Phihat_j^\varepsilon\) for
\(\Phihat_{j,K}^\varepsilon\).

\paragraph{Construction.}
The Gaussian target \(q=\mathcal N(0,\sigmad^2)\) is centered with scale \(\sigmad\), and
\begin{equation*}
  D_\sigma(x)=\frac{\sigmad^2}{\sigmad^2+\sigma^2}x, \qquad v_\sigma(x)=\varpi(\sigma)x,
  \qquad \varpi(\sigma)=\frac{\sigma}{\sigmad^2+\sigma^2}.
\end{equation*}
The rate \(\varpi\) is unimodal with maximum \(1/(2\sigmad)\) at \(\sigma=\sigmad\), so
\(\varpi_{\min}=\min_{[\underline\sigma,\overline\sigma]}\varpi=\min\{\varpi(\underline\sigma),\varpi(\overline\sigma)\}>0\).
Set
\begin{equation*}
  R_\star=\frac{1}{\vartheta}\max\Bigl\{2\ell_{\mathrm{tot}},\frac{4}{\varpi_{\min}}\Bigr\}.
\end{equation*}
Initialize from the translate \(\nu_0^\varepsilon=(x\mapsto
x+R_\star\varepsilon)_\#\mu_{\overline\sigma}\), so that
\(\Wtwo(\nu_0^\varepsilon,\mu_{\overline\sigma})=R_\star\varepsilon\), and set
\begin{equation*}
  \Dhat^\varepsilon_\sigma(x) =D_\sigma(x)-\sigma\varepsilon\sin(x/\varepsilon^2),
  \qquad \sigma\in[\underline\sigma,\overline\sigma].
\end{equation*}
This denoiser admits the EDM parametrization \cref{eq:edm_learned_denoiser_step}: since
\(\beta_j,c_j>0\), inverting \cref{eq:edm_learned_denoiser_step} defines the raw output
\(\Fhat_j(u)=\beta_j^{-1}\bigl(\Dhat^\varepsilon_{\sigma_j}(u/c_j)-\alpha_ju/c_j\bigr)\). The
learned velocity and predictor are
\begin{equation}\label{eq:separation_predictor}
  \vhat^\varepsilon_\sigma(x)=\varpi(\sigma)x+\varepsilon\sin(x/\varepsilon^2),
  \qquad \Phihat^\varepsilon_j(x)=\chi_jx-\ell_j\varepsilon\sin(x/\varepsilon^2),
  \qquad \chi_j=1-\ell_j\varpi(\sigma_j).
\end{equation}
Each predictor is smooth and globally Lipschitz, hence satisfies \Cref{ass:wellposed}.

\paragraph{(a) Small errors.}
The initialization error is \(R_\star\varepsilon\) by construction, and the denoiser error is
pointwise at most \(\sigma\varepsilon\). By \cref{eq:edm_den_error_residual} and
\(\beta(\sigma)=\sigma\sigmad/\sqrt{\sigmad^2+\sigma^2}\),
\begin{equation*}
  \Rres_j^\varepsilon
  =\frac{\norm{\Dhat^\varepsilon_{\sigma_j}(X_{\sigma_j})-D_{\sigma_j}(X_{\sigma_j})}_{\mathbb L_2}}{\beta_j}
  \le\frac{\sqrt{\sigmad^2+\sigma_j^2}}{\sigmad}\,\varepsilon
  \le\frac{\sqrt{\sigmad^2+\overline\sigma^2}}{\sigmad}\,\varepsilon .
\end{equation*}
This is (a), with \(C=\max\{R_\star,\sqrt{\sigmad^2+\overline\sigma^2}/\sigmad\}\).

\paragraph{(b) Zero cumulative log-amplification.}
We first show that the discrepancy never falls below the perturbation scale. Since
\(\varpi(\sigma)<1/\sigma\),
\begin{equation*}
  1>\chi_j>1-\frac{\ell_j}{\sigma_j}=\frac{\sigma_{j+1}}{\sigma_j}>0,
\end{equation*}
so every product of consecutive factors \(\chi_i\) starting at level \(0\) is at least
\(\vartheta\). Let \(\bar x_j^\varepsilon\) be the mean of \(\widehat\mu_j^\varepsilon\), so \(\bar
x_0^\varepsilon=R_\star\varepsilon\). Taking expectations in \cref{eq:separation_predictor} gives
\(\bar x_{j+1}^\varepsilon\ge \chi_j\bar x_j^\varepsilon-\ell_j\varepsilon\). Iterating with
\(0<\chi_i<1\),
\begin{equation*}
  \bar x_j^\varepsilon \ge \vartheta R_\star\varepsilon-\varepsilon\sum_{i<j}\ell_i
  \ge\varepsilon(\vartheta R_\star-\ell_{\mathrm{tot}}) \ge\frac{\vartheta R_\star}{2}\,\varepsilon,
\end{equation*}
by the choice \(\vartheta R_\star\ge2\ell_{\mathrm{tot}}\). The exact laws \(\mu_j\) are centered,
and \(\Wtwo\) dominates the distance between means, hence
\begin{equation}\label{eq:separation_error_lower_bound}
  e_j^\varepsilon\ge \bar x_j^\varepsilon\ge\frac{\vartheta R_\star}{2}\,\varepsilon>0
  \qquad\text{at every step of every schedule.}
\end{equation}

Next, push an optimal coupling \((U_j,V_j)\) of \((\mu_j,\widehat\mu_j^\varepsilon)\) through
\(\Phihat_j^\varepsilon\). Minkowski's inequality, \cref{eq:separation_predictor}, and \(|\sin
u-\sin v|\le2\) give
\(\Wtwo\bigl((\Phihat_j^\varepsilon)_\#\mu_j,(\Phihat_j^\varepsilon)_\#\widehat\mu_j^\varepsilon\bigr)
\le \chi_je_j^\varepsilon+2\ell_j\varepsilon\). With
\cref{eq:separation_error_lower_bound},
\begin{equation*}
  \gamma_j^\varepsilon \le \chi_j+\frac{2\ell_j\varepsilon}{e_j^\varepsilon}
  \le1-\ell_j\Bigl(\varpi(\sigma_j)-\frac{4}{\vartheta R_\star}\Bigr) \le1,
\end{equation*}
by the choice \(\vartheta R_\star\ge4/\varpi_{\min}\). This holds for every \(\varepsilon>0\). Since
\(\sighi>\overline\sigma\), every step belongs to the low-noise block, so the cumulative
log-amplification (\Cref{def:low_noise_budget}) vanishes, \(\Lambda_K=0\). This is (b).

\paragraph{(c) Worst-case expansion.}
We first compute the one-sided rates and Lipschitz constants. By \cref{eq:separation_predictor},
\((\vhat_{\sigma_j}^\varepsilon)'(x)=\varpi(\sigma_j)+\varepsilon^{-1}\cos(x/\varepsilon^2)\), so in
dimension one \cref{eq:global_one_sided_expansion_rate} reads
\begin{equation*}
  \Los{j} =\sup_x\bigl(-(\vhat^\varepsilon_{\sigma_j})'(x)\bigr)
  =\varepsilon^{-1}-\varpi(\sigma_j) \ \ge\ \varepsilon^{-1}-\frac{1}{2\sigmad}.
\end{equation*}
This lower bound holds at each step separately, with a constant independent of \(j\) and \(K\); it
does not come from accumulation over steps. Summed over the schedule, it makes the first term of
\cref{eq:field_level_summed_condition} at least
\(\sum_j\ell_j(\varepsilon^{-1}-1/(2\sigmad))=\ell_{\mathrm{tot}}(\varepsilon^{-1}-1/(2\sigmad))\).
The same slope drives the map: differentiating \cref{eq:separation_predictor} gives
\((\Phihat_j^\varepsilon)'(x)=\chi_j-(\ell_j/\varepsilon)\cos(x/\varepsilon^2)\), and because
\(\chi_j>0\) its supremum in absolute value is attained where the cosine equals \(-1\):
\begin{equation*}
  \operatorname{Lip}(\Phihat_j^\varepsilon) =\chi_j+\frac{\ell_j}{\varepsilon} =1+\ell_j\Los{j},
\end{equation*}
so \cref{eq:field_below_map_first_order} holds with equality here and
\begin{equation*}
  \frac{\operatorname{Lip}(\Phihat_j^\varepsilon)-1}{\ell_j}
  =\Los{j}\ge\varepsilon^{-1}-\frac{1}{2\sigmad}.
\end{equation*}
These exact stepwise formulas hold for every \(\varepsilon>0\), and show that the map and field
majorants quantify the same gap.

We now pass to the limits, and assume \(\varepsilon<2\sigmad\). Then
\(\varepsilon^{-1}>1/(2\sigmad)\ge\varpi(\sigma_j)\), so \(\Los{j}>0\) and every
\(\operatorname{Lip}(\Phihat_j^\varepsilon)=1+\ell_j\Los{j}\) exceeds one. This is what allows us to
drop the maximum with one in \(\Lambda_K^{\mathrm{map},\varepsilon}\), which becomes
\(\sum_{j=0}^{K-1}\log\operatorname{Lip}(\Phihat_j^\varepsilon)\). Write
\(x_j=\ell_j\Los{j}=\ell_j(\varepsilon^{-1}-\varpi(\sigma_j))\in[0,\ell_j/\varepsilon]\). Since
\(x-x^2/2\le\log(1+x)\le x\) for \(x\ge0\),
\begin{equation*}
  \Bigl|\sum_{j=0}^{K-1}\log\operatorname{Lip}(\Phihat_j^\varepsilon) -\sum_{j=0}^{K-1}x_j\Bigr|
  \le\frac12\sum_{j=0}^{K-1}x_j^2 \le\frac{\ell_{\mathrm{tot}}}{2\varepsilon^2}\,\ell_{\max,K} .
\end{equation*}
As \(K\to\infty\) along \(\mathscr K\), \(\ell_{\max,K}\to0\), so the remainder vanishes and the
Riemann sum converges to
\(\int_{\underline\sigma}^{\overline\sigma}(\varepsilon^{-1}-\varpi(\sigma))\,\dd\sigma\). Since
\(\int\varpi(\sigma)\,\dd\sigma=\tfrac12\log(\sigmad^2+\sigma^2)\), this gives the map-level limit
\begin{equation}\label{eq:separation_worst_case_budget_limit}
  \lim_{K\to\infty}\Lambda_K^{\mathrm{map},\varepsilon}
  = \frac{\overline\sigma-\underline\sigma}{\varepsilon}
  -\frac12 \log\frac{\sigmad^2+\overline\sigma^2}{\sigmad^2+\underline\sigma^2}.
\end{equation}

It remains to identify the limit of the field-level majorant
\(\Lambda_K^{\mathrm{field},\varepsilon}\) of \cref{eq:field_level_summed_condition}. Since
\(\varepsilon^{-1}>\varpi(\sigma_j)\), the derivative computed above gives
\begin{equation*}
  \operatorname{Lip}(\vhat_{\sigma_j}^\varepsilon) =\varepsilon^{-1}+\varpi(\sigma_j).
\end{equation*}
Therefore
\begin{equation*}
  0 \le \frac12\sum_{j=0}^{K-1}\ell_j^2 \operatorname{Lip}(\vhat_{\sigma_j}^\varepsilon)^2
  \le \frac{\ell_{\mathrm{tot}}}{2} \left(\varepsilon^{-1}+\frac1{2\sigmad}\right)^2 \ell_{\max,K}
  \longrightarrow0.
\end{equation*}
The first-order term of \(\Lambda_K^{\mathrm{field},\varepsilon}\) is
\(\sum_j\ell_j(\Los{j})_+=\sum_j x_j\), whose Riemann-sum limit is the map-level limit just
computed. Hence the field and map majorants have the common limit
\eqref{eq:separation_worst_case_budget_limit}, which is of the form
\((\overline\sigma-\underline\sigma)/\varepsilon+O(1)\) with
\(O(1)=-\tfrac12\log\bigl((\sigmad^2+\overline\sigma^2)/(\sigmad^2+\underline\sigma^2)\bigr)\). This
term depends only on \(\sigmad\), \(\underline\sigma\) and \(\overline\sigma\), not on
\(\varepsilon\), and lies in \([-\log(\overline\sigma/\underline\sigma),0]\). This proves (c).
\end{proof}

\begin{remark}[Cost of the construction]
  \label{rem:separation_cost}
  The price of the separation is the size of the translation \(R_\star\varepsilon\): the constant
  \(C\) of \Cref{prop:strict_law_map_separation} is at least
  \(R_\star=(\overline\sigma/\underline\sigma)\max\{2(\overline\sigma-\underline\sigma),4/\varpi_{\min}\}\),
  and both factors are large on wide noise ranges. The first is the ratio of the endpoints. In the
  second, \(4/\varpi(\sigma)=4(\sigmad^2+\sigma^2)/\sigma\), so \(4/\varpi_{\min}\) is of order
  \(\max\{\sigmad^2/\underline\sigma,\overline\sigma\}\) when
  \(\underline\sigma\ll\sigmad\ll\overline\sigma\). With the EDM endpoints
  \(\underline\sigma=0.002\), \(\overline\sigma=80\) and \(\sigmad=0.5\), the ratio is
  \(4\times10^4\) and \(4/\varpi_{\min}\approx500\).
\end{remark}

The first construction gives a quantitative separation: both worst-case majorants are finite for
each fixed \(\varepsilon\) but arbitrarily large as \(\varepsilon\downarrow0\). A fixed
construction, in which the perturbation does not vary along a vanishing parameter, gives the
qualitative non-implication.

\begin{proposition}[Zero cumulative log-amplification with infinite worst-case constants]
  \label{prop:fixed_infinite_separation}
  In the setting of \Cref{prop:strict_law_map_separation}, for every \(\bar\varepsilon>0\) there are
  a smooth denoiser family satisfying \Cref{ass:wellposed} and the EDM parametrization, and an
  initialization, such that on the same refinement grids \(\gamma_{j,K}\le1\) at every step of every
  schedule, so \(\Lambda_K=0\), while every predictor has infinite global Lipschitz constant and
  every one-sided rate \(\Los{j}\) is infinite.
\end{proposition}

\begin{proof}
Let \(\bar\varepsilon>0\), take \(R_\star\) as in the proof above, let
\(\nu_0^{\mathrm{fix},\bar\varepsilon}\) be the translate of \(\mu_{\overline\sigma}\) by
\(R_\star\bar\varepsilon\), and define the second, now fixed, denoiser family
\begin{equation*}
  \Dhat_\sigma^{\mathrm{fix},\bar\varepsilon}(x)=D_\sigma(x)-\sigma\bar\varepsilon\sin(x^2).
\end{equation*}
Attach this denoiser family and initialization, fixed across refinements, to the grids
\(G_K\in\Gsched\). The residual, mean, and coupling arguments above use only the amplitude bounds
\(|\sin|\le1\) and \(|\sin u-\sin v|\le2\), not the frequency of the oscillation, which changes from
\(x/\varepsilon^2\) to \(x^2\). They apply verbatim with \(\varepsilon\) replaced by
\(\bar\varepsilon\): (a) of \Cref{prop:strict_law_map_separation} holds with \(\bar\varepsilon\) in
place of \(\varepsilon\), and the corresponding factors satisfy
\(\gamma_{j,K}^{\mathrm{fix},\bar\varepsilon}\le1\) at every step of every \(G_K\in\Gsched\). The
associated predictor
\(\Phihat_j^{\mathrm{fix},\bar\varepsilon}(x)=\chi_jx-\ell_j\bar\varepsilon\sin(x^2)\) is smooth,
hence locally Lipschitz, and satisfies
\(|\Phihat_j^{\mathrm{fix},\bar\varepsilon}(x)|\le|x|+\ell_j\bar\varepsilon\), so
\Cref{ass:wellposed} holds. But
\begin{equation*}
  (\Phihat_j^{\mathrm{fix},\bar\varepsilon})'(x)=\chi_j-2\ell_j\bar\varepsilon x\cos(x^2)
\end{equation*}
is unbounded, so \(\operatorname{Lip}(\Phihat_j^{\mathrm{fix},\bar\varepsilon})=\infty\) at every
step; and since
\((\vhat^{\mathrm{fix},\bar\varepsilon}_{\sigma_j})'(x)=\varpi(\sigma_j)+2\bar\varepsilon
x\cos(x^2)\) is unbounded below as well, \(\Los{j}=\infty\) and \cref{eq:field_one_sided_condition}
holds for no finite constant. Neither worst-case certificate is available, while \(\Lambda_K=0\).
\end{proof}

\begin{remark}[Conjecture: localized perturbations]
  We conjecture, without proof, that the oscillatory form is not essential: a narrow smooth bump
  should confine the large derivative to an arbitrarily small spatial region, and a cutoff in
  \(\sigma\) should confine the perturbation to a prescribed low-noise interval while leaving the
  high-noise denoiser unchanged, with the same arguments.
\end{remark}

\section{The Global Wasserstein Bound}
\label{app:global_bound}

This appendix complements \Cref{sec:global_bound}. It proves the high-noise and low-noise block
estimates, with the summation lemmas they rely on, and concatenates them into the proofs of
\Cref{thm:global_edm_wasserstein,cor:budget_refinement_rates}; it then bounds the initialization
error and extends the bound to the deployed sample.

\subsection{Explicit Form of the Global Bound}

Two quantities of the proof are left implicit in \Cref{thm:global_edm_wasserstein}. On a finite mesh
the kernel contracts slightly less than its continuous form, by the factor
\begin{equation}\label{eq:def_kappa_mesh_edm_predictor}
  \kappa(\theta) = \frac{1-(1-\theta)^{1/\rho}}{\rho^{-1}\log\bigl((1-\theta)^{-1}\bigr)} \in(0,1)
\end{equation}
at relative step \(h/\eta_j=\theta\), and \(\kappa(\theta)\uparrow1\) as \(\theta\downarrow0\). The
learning terms sum to the learning scale
\begin{align}
  \Upsilon(s_1,s_0) &= \int_{s_1}^{s_0}\frac{\sigmad\,\dd s}{\sqrt{s^2+\sigmad^2}}\nonumber \\
  &\le \min\Bigl\{s_0-s_1,\ \sigmad\log\Bigl(1+\frac{2s_0}{\sigmad}\Bigr)\Bigr\}.
  \label{eq:def_learning_scale}
\end{align}

\begin{theorem}[Global Wasserstein bound, explicit form]
  \label{thm:global_edm_explicit}
  Under the hypotheses of \Cref{thm:global_edm_wasserstein}, set \(\beff=\bhi\kappa(\thetahi)\) and
  \(C_\star=C_{\rho,\max(\thetahi,\thetalo)}\), the constant of \Cref{lem:damped_defect_sum}. For
  every \(G_K\in\Gsched\) satisfying the mesh condition
  \begin{equation}\label{eq:global_mesh_condition}
    h_K\leqslant \min\bigl\{\thetahi\sighi^\rho,\ \thetalo\underline\sigma^\rho\bigr\},
  \end{equation}
  the bound \eqref{eq:transport_global_bound} holds with \(\mathrm{Disc}_K\) as in
  \Cref{thm:global_edm_wasserstein}, that is,
  \begin{equation*}
    \mathrm{Disc}_K =C_\star\Bd[1]\,h_K\bigl[\Xi_\rho(0;\sigma_K,\sigma_{\Khi})
    +\Xi_\rho(\beff;\sigma_{\Khi},\sigma_0)\bigr],
  \end{equation*}
  and with the learning term
  \begin{equation*}
    \mathrm{Learn}_K =\sqrt d\,\Cres\Bigl[\frac{\sigmad}{\bhi}\bigl(1-\Pi_{0,\Khi}\bigr)
    +\Upsilon(\sigma_K,\sigma_{\Khi})\Bigr],
  \end{equation*}
  which is at most the learning term of \Cref{thm:global_edm_wasserstein} by
  \cref{eq:def_learning_scale}.
\end{theorem}

\begin{remark}[Three regimes, one rate]
  \label{rem:three_regimes}
  Since \(\Xi_\rho\) is nonincreasing in \(b\), contraction is not what makes the discretization
  error first order; it decides which noise scale sets the constant. For \(\rho<1\), the high-noise
  scale function \(\Xi_\rho(\beff;\sigma_{\Khi},\sigma_0)\) is at most
  \(\sigma_{\Khi}^{1-\rho}/(\beff+\rho-1)\) if \(\beff>1-\rho\) (supercritical),
  \(\sigma_{\Khi}^{1-\rho}\log(\sigma_0/\sigma_{\Khi})\) if \(\beff=1-\rho\) (critical), and
  \(\sigma_{\Khi}^{\beff}\sigma_0^{1-\rho-\beff}/(1-\rho-\beff)\) if \(\beff<1-\rho\) (subcritical).
  In the supercritical regime the damped sum of the local discretization biases of the high-noise
  block is set by \(\sigma_{\Khi}\) alone, however large \(\sigma_0\) is. At \(\rho=1/7\)
  (\(\rhoEDM=7\)) the supercritical condition reads \(\beff>6/7\). Since
  \(\beff=\bhi\kappa(\thetahi)<\bhi\) and \(\kappa(\thetahi)\uparrow1\) as \(\thetahi\downarrow0\),
  it requires \(\bhi>6/7\), and is then met on every sufficiently fine mesh, that is, for
  \(\thetahi\) small enough (\Cref{subsec:evidence_damping}).
\end{remark}

\subsection{High-Noise Block}

Under the high-noise contraction \cref{eq:high_noise_transport_contraction},
\cref{eq:one_step_learned_map_bound_integrated_curvature} specializes to
\begin{equation}\label{eq:one_step_edm_predictor_residual_bound}
  e_{j+1} \le (1-\bhi a_j)e_j+\Delta_j+a_j\beta_j\Rres_j, \qquad j=0,\ldots,\Khi-1.
\end{equation}
Under the EDM certificate of \Cref{prop:high_noise_stability_certificate}, the step-dependent factor
\(1-a_jb_j\) may replace \(1-\bhi a_j\) throughout. The kernel of \Cref{sec:global_bound},
\(\Pi_{r,k}=\prod_{i=r}^{k-1}(1-\bhi a_i)\) for \(0\leqslant r\leqslant k\), has factors in
\((0,1)\), and \(\Pi_{k,k}=1\) (empty product). It satisfies \(\Pi_{r,k}=(1-\bhi a_r)\Pi_{r+1,k}\)
for \(r<k\).

\begin{proposition}[High-noise contraction estimate]
  \label{prop:large_noise_edm_discrete_gronwall}
  Let \(\bhi\in(0,1)\), and assume \Cref{ass:uniform_residual_accuracy} and the high-noise
  contraction \cref{eq:high_noise_transport_contraction}. Then, for every \(G_K\in\Gsched\) and
  every \(k\in\{0,\ldots,\Khi\}\),
  \begin{equation*}
    e_{k} \le \Pi_{0,k}e_0 + \sum_{j=0}^{k-1}\Pi_{j+1,k}\Delta_j
    + \frac{\sigmad\sqrt{d}\,\Cres}{\bhi} \left(1-\Pi_{0,k}\right).
  \end{equation*}
\end{proposition}

\begin{proof}
For \(k=0\) the bound reads \(e_0\leqslant e_0\). For \(k\geqslant1\), iterating
\cref{eq:one_step_edm_predictor_residual_bound} for \(j=0,\ldots,k-1\), all of whose coefficients
are nonnegative, gives \(e_{k}\leqslant
\Pi_{0,k}e_0+\sum_{j=0}^{k-1}\Pi_{j+1,k}(\Delta_j+a_j\beta_j\Rres_j)\). Since \(\beta_j\leqslant
\sigmad\) and \(\Rres_j\leqslant \sqrt d\,\Cres\), the learning term satisfies
\(a_j\beta_j\Rres_j\le(\sigmad\sqrt d\,\Cres/\bhi)\,\bhi a_j\), and the kernel telescopes:
\(\Pi_{j+1,k}\bhi a_j=\Pi_{j+1,k}-\Pi_{j,k}\), whose sum over \(j=0,\ldots,k-1\) is
\(\Pi_{k,k}-\Pi_{0,k}=1-\Pi_{0,k}\).
\end{proof}

The kernel damps the initialization error and saturates the accumulated learning error at
\(\sigmad\sqrt d\,\Cres/\bhi\), however many steps the block contains. In the damped bias sum
\(\sum_j\Pi_{j+1,k}\Delta_j\), the universal bias bound \(\Delta_j\) is of order
\(h^2\sigma_j^{1-2\rho}\), so for \(\rho<1/2\) the earliest, highest-noise steps contribute the
most, and the kernel discounts those same steps geometrically. The scale function \cref{eq:def_Xi}
measures the outcome; explicitly,
\begin{equation*}
  \Xi_\rho(b;s_1,s_0)
  = \begin{cases} \dfrac{s_1^{1-\rho}}{b+\rho-1}
    \Bigl[1-\Bigl(\dfrac{s_1}{s_0}\Bigr)^{b+\rho-1}\Bigr], & b\neq1-\rho,\\[1.2em]
    s_1^{1-\rho}\log\dfrac{s_0}{s_1}, & b=1-\rho, \end{cases}
\end{equation*}
and \(\Xi_\rho(0;s_1,s_0)=\int_{s_1}^{s_0}s^{-\rho}\dd s\) is the undamped sum. On a finite mesh the
relaxation \(a_j\) is smaller than the log-noise increment \(r_j=\log(\sigma_j/\sigma_{j+1})\),
whose sums give the continuous kernel, and \(\kappa(\theta)\) of
\cref{eq:def_kappa_mesh_edm_predictor} is the value of \(a_j/r_j\) at \(h/\eta_j=\theta\). On every
stretch of steps with \(h/\eta_i\leqslant \theta\) for \(i=j,\ldots,k-1\),
\(\Pi_{j,k}\le(\sigma_{k}/\sigma_{j})^{\bhi\kappa(\theta)}\) (\Cref{lem:kernel_power_bound}).

\begin{proposition}[High-noise estimate on the refinement family]
  \label{prop:high_noise_grid}
  Let \(0<\rho\leqslant 1\) and \(\bhi\in(0,1)\), and assume \Cref{ass:uniform_residual_accuracy}
  and the high-noise contraction \cref{eq:high_noise_transport_contraction}. Fix
  \(\thetahi\in(0,1)\) and set \(\kappahi=\kappa(\thetahi)\) and \(\beff=\bhi\kappahi\). There is a
  constant \(C_{\rho,\thetahi}<\infty\), depending only on \(\rho\) and \(\thetahi\), such that for
  every \(G_K\in\Gsched\) with \(h_K\leqslant \thetahi\sighi^\rho\),
  \begin{equation}\label{eq:large_noise_closed_damped_curvature_bound}
    e_{\Khi} \le \Pi_{0,\Khi}e_0
    + C_{\rho,\thetahi}\Bd[1]\,h_K\,\Xi_\rho(\beff;\sigma_{\Khi},\sigma_0)
    + \frac{\sigmad\sqrt{d}\,\Cres}{\bhi} \left(1-\Pi_{0,\Khi}\right).
  \end{equation}
\end{proposition}

The mesh condition gives \(h_K/\eta_j\leqslant \thetahi\) on the whole high-noise block. The proof,
in \Cref{app:large_noise_extras} together with the explicit constant, bounds
\(\Pi_{j+1,\Khi}\le(\sigma_{\Khi}/\sigma_{j+1})^{\beff}\) and compares the resulting sum with the
integral \cref{eq:def_Xi}. The range \(\rho>1\), that is \(\rhoEDM<1\), is not treated by this
estimate nor by the low-noise one; it plays no role in practice.

\subsection{Summation of the Local Biases}
\label{app:large_noise_extras}

This subsection proves the two summation estimates behind \Cref{prop:high_noise_grid} and
\Cref{prop:terminal_transport_grid_bound}. Throughout, \(\eta_r>\eta_{r+1}>\cdots>\eta_{k+1}>0\) is
a segment of a grid that is uniform in the clock \(\eta=\sigma^\rho\), with \(0<\rho\leqslant 1\),
common step \(h=\eta_j-\eta_{j+1}\), and levels \(\sigma_j=\eta_j^{1/\rho}\); the lemmas below
differ only in the indices of the segment.

The argument has three steps. \Cref{lem:relative_mesh_riemann} compares a Riemann sum on the grid
with the corresponding integral, at a cost controlled by the relative mesh \(h/\eta_j\).
\Cref{lem:damped_defect_sum} uses it to bound the sum of the universal bias bounds \(\Delta_j\)
weighted by a power kernel \((\sigma_{k+1}/\sigma_{j+1})^b\), and identifies the resulting integral
with the scale function \(\Xi_\rho\) of \cref{eq:def_Xi}; with \(b=0\) it bounds the low-noise bias
sum of \Cref{prop:terminal_transport_grid_bound}. For the high-noise block,
\Cref{lem:kernel_power_bound} bounds the discrete contraction kernel \(\Pi_{j+1,k+1}\) by such a
power, with exponent \(\beff\), and inserting both bounds into
\Cref{prop:large_noise_edm_discrete_gronwall} proves \Cref{prop:high_noise_grid}.

\subsubsection{A Relative-Mesh Riemann Comparison}

\begin{lemma}[Relative-mesh Riemann comparison]
  \label{lem:relative_mesh_riemann}
  Let \(f_p(\eta)=\eta^{-p}\) with \(p\in\R\), and assume \(h\leqslant \theta\eta_j\) for
  \(j=r,\ldots,k\) and some \(\theta\in(0,1)\). Then
  \begin{equation}\label{eq:relative_mesh_riemann}
    h\sum_{j=r}^{k} f_p(\eta_{j+1})
    \le (1-\theta)^{-p_+}\int_{\eta_{k+1}}^{\eta_r} f_p(\eta)\,\dd\eta, \qquad p_+=\max(p,0).
  \end{equation}
\end{lemma}

\begin{proof}
From \(\eta_{j+1}=\eta_j-h\ge(1-\theta)\eta_j\) we have \(\eta_j/\eta_{j+1}\le(1-\theta)^{-1}\). For
\(\eta\) in the panel \([\eta_{j+1},\eta_j]\),
\(f_p(\eta_{j+1})/f_p(\eta)=(\eta/\eta_{j+1})^{p}\le(1-\theta)^{-p_+}\), so
\(h\,f_p(\eta_{j+1})\le(1-\theta)^{-p_+}\int_{\eta_{j+1}}^{\eta_j}f_p(\eta)\,\dd\eta\). Summing over
\(j=r,\ldots,k\) gives \cref{eq:relative_mesh_riemann}.
\end{proof}

\subsubsection{The Damped Bias Sum}

\begin{lemma}[Damped bias sum]
  \label{lem:damped_defect_sum}
  Assume \(h\leqslant \theta\eta_j\) for \(j=r,\ldots,k\) and some \(\theta\in(0,1)\). Then for
  every \(b\in[0,1]\),
  \begin{equation}\label{eq:damped_defect_sum}
    \sum_{j=r}^{k}\Bigl(\frac{\sigma_{k+1}}{\sigma_{j+1}}\Bigr)^{b}\Delta_j
    \le C_{\rho,\theta}\,\Bd[1]\,h\,\Xi_\rho(b;\sigma_{k+1},\sigma_r),
    \qquad C_{\rho,\theta}=\frac{(1-\theta)^{-2-2/\rho}}{2\rho},
  \end{equation}
  with \(\Delta_j\) as in \Cref{lem:local_euler_defect} and \(\Xi_\rho\) as in \cref{eq:def_Xi}.
\end{lemma}

\begin{proof}
Write \(\theta_j=h/\eta_j\leqslant \theta\). Since
\(\sigma_{j+1}/\sigma_j=(\eta_{j+1}/\eta_j)^{1/\rho}=(1-\theta_j)^{1/\rho}\), the relaxation
coefficient is \(a_j=1-(1-\theta_j)^{1/\rho}\). As \(1/\rho\ge1\), Bernoulli's inequality gives
\(a_j\leqslant \theta_j/\rho\), and \(1-a_j=(1-\theta_j)^{1/\rho}\ge(1-\theta)^{1/\rho}\). The
function \(g(a)=a+(1-a)\log(1-a)=\sum_{m\ge2}a^m/(m(m-1))\) satisfies \(g(a)\leqslant
a^2/(2(1-a))\). Hence, using \(\theta_j=h\sigma_j^{-\rho}\),
\begin{equation*}
  \Delta_j = \Bd[1]\sigma_jg(a_j)
  \le \frac{(1-\theta)^{-1/\rho}}{2\rho^2}\,\Bd[1]\,h^2\sigma_j^{1-2\rho}.
\end{equation*}
Up to the constant \((1-\theta)^{-1/\rho}\Bd[1]/(2\rho^2)\) and the factor \(h^2\), the \(j\)th term
of \cref{eq:damped_defect_sum} is therefore at most
\((\sigma_{k+1}/\sigma_{j+1})^b\sigma_j^{1-2\rho}
=\sigma_{k+1}^{b}\,\eta_{j+1}^{-b/\rho}\eta_j^{(1-2\rho)/\rho}\) in the clock variable. One factor
\(h\) turns the sum of these terms into a Riemann sum, and the other remains as the factor \(h\) of
\cref{eq:damped_defect_sum}. Evaluating the bias factor \(\eta_j^{(1-2\rho)/\rho}\) at
\(\eta_{j+1}\) instead of \(\eta_j\) costs at most \((1-\theta)^{-(1-2\rho)_+/\rho}\), since
\(\eta_j/\eta_{j+1}\le(1-\theta)^{-1}\), and \Cref{lem:relative_mesh_riemann} with
\(p=(b+2\rho-1)/\rho\le2\) then gives
\begin{equation*}
  h\sum_{j=r}^{k}\eta_{j+1}^{-b/\rho}\eta_j^{(1-2\rho)/\rho}
  \le (1-\theta)^{-(1-2\rho)_+/\rho-2} \int_{\eta_{k+1}}^{\eta_r}\eta^{(1-2\rho-b)/\rho}\,\dd\eta
  = \rho\,(1-\theta)^{-(1-2\rho)_+/\rho-2}\int_{\sigma_{k+1}}^{\sigma_r}s^{-\rho-b}\,\dd s,
\end{equation*}
by the substitution \(\eta=s^\rho\). Multiplying by \(\sigma_{k+1}^b\) turns the last integral into
\(\Xi_\rho(b;\sigma_{k+1},\sigma_r)\). Collecting the factors, whose exponents of
\((1-\theta)^{-1}\) add up to at most \(2+2/\rho\), gives \cref{eq:damped_defect_sum}.
\end{proof}

The constant depends on \(b\) only through the bound \(p\le2\), so it is uniform over \(b\in[0,1]\);
in particular the same lemma serves the damped high-noise sum and, with \(b=0\), the undamped
low-noise sum. As \(\theta\downarrow0\) it tends to \(1/(2\rho)\), the value given by the
leading-order expansion \eqref{eq:local_defect_clock_expansion} of the universal bias bound,
\(\Delta_j\approx\Bd[1]h^2\sigma_j^{1-2\rho}/(2\rho^2)\).

\subsubsection{Proof of the High-Noise Estimate}

\begin{lemma}[Finite-mesh contraction of the kernel]
  \label{lem:kernel_power_bound}
  Let \(\bhi\in(0,1)\), let \(j\leqslant k\) and let the segment be \(\eta_{j}>\cdots>\eta_{k}\),
  with \(a_i=1-\sigma_{i+1}/\sigma_i\) and \(\Pi_{j,k}=\prod_{i=j}^{k-1}(1-\bhi a_i)\). If
  \(h\leqslant \theta\eta_i\) for \(i=j,\ldots,k-1\) and some \(\theta\in(0,1)\), then
  \(\Pi_{j,k}\le(\sigma_{k}/\sigma_{j})^{\bhi\kappa(\theta)}\), with \(\kappa\) as in
  \cref{eq:def_kappa_mesh_edm_predictor}.
\end{lemma}

\begin{proof}
Write
\begin{equation*}
  r_i=\log\frac{\sigma_i}{\sigma_{i+1}}=\frac1\rho\log\frac{1}{1-h/\eta_i}
\end{equation*}
for the log-noise increment of step \(i\), so that \(a_i=1-e^{-r_i}\). The continuous kernel with
exponent \(b\) is \((\sigma_{k}/\sigma_{j})^{b}=\exp(-b\sum_{i=j}^{k-1}r_i)\), and the discrete
relaxation is smaller than the increment, \(a_i\leqslant r_i\), since \(1-e^{-r}\leqslant r\). In
the other direction, \(r_i\) is an increasing function of \(h/\eta_i\leqslant \theta\), so
\(r_i\leqslant r_\theta=\rho^{-1}\log((1-\theta)^{-1})\), and since \(r\mapsto(1-e^{-r})/r\) is
decreasing,
\begin{equation*}
  a_i\ge\kappa(\theta)\,r_i, \qquad \kappa(\theta)=\frac{1-e^{-r_\theta}}{r_\theta}
  =\frac{1-(1-\theta)^{1/\rho}}{\rho^{-1}\log\bigl((1-\theta)^{-1}\bigr)}\in(0,1),
\end{equation*}
which is \cref{eq:def_kappa_mesh_edm_predictor}. Finally \(1-x\leqslant e^{-x}\) for every factor
gives
\begin{equation*}
  \Pi_{j,k} \leqslant \exp\Bigl(-\bhi\sum_{i=j}^{k-1}a_i\Bigr)
  \leqslant \exp\Bigl(-\bhi\kappa(\theta)\sum_{i=j}^{k-1}r_i\Bigr)
  =\Bigl(\frac{\sigma_{k}}{\sigma_{j}}\Bigr)^{\bhi\kappa(\theta)}.
\end{equation*}
For \(j=k\) both sides equal one.
\end{proof}

\begin{proof}[Proof of \Cref{prop:high_noise_grid}]
For \(\Khi=0\) the kernel \(\Pi_{0,0}\) is an empty product, equal to one, the bias sum is empty,
and \(\Xi_\rho(\beff;\sigma_0,\sigma_0)=0\), so \cref{eq:large_noise_closed_damped_curvature_bound}
reduces to \(e_0\leqslant e_0\). For \(\Khi\ge1\), every level \(\eta_j\) with \(j\leqslant \Khi\)
satisfies \(\eta_j\ge\sighi^\rho\), so the mesh condition \(h_K\leqslant \thetahi\sighi^\rho\) gives
\(h_K\leqslant \thetahi\eta_j\) on the whole high-noise block. Apply
\Cref{prop:large_noise_edm_discrete_gronwall} with \(k=\Khi\), bound each kernel \(\Pi_{j+1,\Khi}\),
\(0\leqslant j\leqslant\Khi-1\), by \Cref{lem:kernel_power_bound} on the segment
\(\eta_{j+1}>\cdots>\eta_{\Khi}\), which gives
\(\Pi_{j+1,\Khi}\leqslant(\sigma_{\Khi}/\sigma_{j+1})^{\beff}\), and the resulting sum by
\Cref{lem:damped_defect_sum} with \(r=0\), \(k=\Khi-1\) and \(b=\beff\leqslant 1\):
\begin{equation*}
  \sum_{j=0}^{\Khi-1}\Pi_{j+1,\Khi}\Delta_j
  \leqslant\sum_{j=0}^{\Khi-1}\Bigl(\frac{\sigma_{\Khi}}{\sigma_{j+1}}\Bigr)^{\beff}\Delta_j
  \leqslant C_{\rho,\thetahi}\Bd[1]h_K\,\Xi_\rho(\beff;\sigma_{\Khi},\sigma_0).
\end{equation*}
\end{proof}

\subsection{Low-Noise Block and Proof of the Global Bound}

No stability assumption is made below the threshold. The recursion is closed by \(\Lambda_K\), which
bounds every partial product of the truncated factors; with no contraction to exploit, the universal
bias bounds are summed without a kernel, which is the undamped case \(b=0\) of \cref{eq:def_Xi}.

\begin{proposition}[Low-noise estimate on the refinement family]
  \label{prop:terminal_transport_grid_bound}
  Let \(0<\rho\leqslant 1\), assume \Cref{ass:uniform_residual_accuracy}, and let \(\Lambda_K\) be
  the cumulative low-noise log-amplification of \Cref{def:low_noise_budget}. Fix
  \(\thetalo\in(0,1)\). There is a constant \(C_{\rho,\thetalo}<\infty\), depending only on \(\rho\)
  and \(\thetalo\), such that for every \(G_K\in\Gsched\) with \(h_K\leqslant
  \thetalo\underline\sigma^\rho\),
  \begin{equation}\label{eq:terminal_transport_grid_bound}
    e_K \le e^{\Lambda_K} \left[ e_{\Khi}
    + C_{\rho,\thetalo}\Bd[1]\,h_K\,\Xi_\rho(0;\sigma_K,\sigma_{\Khi})
    + \sqrt{d}\,\Cres\,\Upsilon(\sigma_K,\sigma_{\Khi}) \right],
  \end{equation}
  where \(\Xi_\rho(0;\sigma_K,\sigma_{\Khi})\) equals
  \((\sigma_{\Khi}^{1-\rho}-\sigma_K^{1-\rho})/(1-\rho)\) for \(\rho<1\) and
  \(\log(\sigma_{\Khi}/\sigma_K)\) for \(\rho=1\), and \(\Upsilon\) is the learning scale
  \cref{eq:def_learning_scale},
  \(\Upsilon(s_1,s_0)=\sigmad[\operatorname{arsinh}(s_0/\sigmad)-\operatorname{arsinh}(s_1/\sigmad)]\).
\end{proposition}

\begin{proof}
For \(\Khi=K\) there is nothing to prove, since \(\Lambda_K=0\) and the sums vanish. Otherwise,
iterating \cref{eq:one_step_learned_map_bound_integrated_curvature} from \(j=\Khi\) to \(K-1\),
using \(\gamma_i\leqslant\gambar_i\) and bounding every partial product
\(\prod_{i=j+1}^{K-1}\gambar_i\), \(\Khi\leqslant j\leqslant K-1\), by \(e^{\Lambda_K}\) (an empty
product equals \(1\leqslant e^{\Lambda_K}\)) gives \(e_K\leqslant
e^{\Lambda_K}[e_{\Khi}+\sum_{j=\Khi}^{K-1}(\Delta_j+a_j\beta_j\Rres_j)]\). The bias sum is bounded
by \Cref{lem:damped_defect_sum} with \(r=\Khi\), \(k=K-1\) and \(b=0\), since \(h_K/\eta_j\leqslant
h_K/\underline\eta\leqslant \thetalo\) by \cref{eq:schedule_family_relative_mesh}, where
\(\underline\eta=\underline\sigma^\rho\). For the learning term,
\(a_j\beta_j=\ell_j\sigmad/\sqrt{\sigma_j^2+\sigmad^2}\). The integrand of
\cref{eq:def_learning_scale} is decreasing and \(\sigma_j\) is the upper end of the step, so
\(a_j\beta_j\leqslant \int_{\sigma_{j+1}}^{\sigma_j}\sigmad(s^2+\sigmad^2)^{-1/2}\,\dd s\). With
\Cref{ass:uniform_residual_accuracy}, the learning terms therefore sum to at most \(\sqrt
d\,\Cres\,\Upsilon(\sigma_K,\sigma_{\Khi})\), with no mesh condition.
\end{proof}

The learning scale satisfies \(\Upsilon(s_1,s_0)\leqslant \min\{s_0-s_1,\
\sigmad\log(1+2s_0/\sigmad)\}\): it is at most the length of the noise range, since its integrand is
at most one, and it grows only logarithmically in the upper end once \(s_0\gg\sigmad\), since
\(\operatorname{arsinh}x=\log(x+\sqrt{x^2+1})\leqslant \log(1+2x)\).

\begin{proof}[Proof of \Cref{thm:global_edm_wasserstein,thm:global_edm_explicit}]
The constant is \(C_\star=C_{\rho,\max(\thetahi,\thetalo)}\), that of \Cref{lem:damped_defect_sum}.
By \cref{eq:global_mesh_condition}, the mesh conditions of \Cref{prop:high_noise_grid} and
\Cref{prop:terminal_transport_grid_bound} both hold. Substitute the bound on \(e_{\Khi}\) of
\Cref{prop:high_noise_grid} into \cref{eq:terminal_transport_grid_bound}, use
\(C_{\rho,\theta}\leqslant C_\star\) for \(\theta\leqslant \max(\thetahi,\thetalo)\), since
\(C_{\rho,\theta}\) is increasing in \(\theta\), and add the terminal bias
\cref{eq:edm_final_error_with_terminal_bias}. The two discretization terms combine into
\(C_\star\Bd[1]h_K[\Xi_\rho(0;\sigma_K,\sigma_{\Khi})+\Xi_\rho(\beff;\sigma_{\Khi},\sigma_0)]
=\mathrm{Disc}_K\). For \(\Khi=0\) the high-noise terms vanish (\(\Pi_{0,0}=1\),
\(\Xi_\rho(\beff;\sigma_0,\sigma_0)=0\)); for \(\Khi=K\) the low-noise terms vanish. This gives
\Cref{thm:global_edm_explicit}, and \Cref{thm:global_edm_wasserstein} follows by the bound on
\(\Upsilon\) in \cref{eq:def_learning_scale}.
\end{proof}

The two mesh conditions in \cref{eq:global_mesh_condition} are those of the two blocks; they are
kept separate because \(\thetahi\) sets how much of the margin the kernel retains,
\(\beff=\bhi\kappa(\thetahi)\), whereas \(\thetalo\) enters only the constant. Refinement does not
improve the denoiser or remove an initialization mismatch, so a growing envelope for
\(e^{\Lambda_K}\) can make those terms of the bound increase even when the actual error stays
bounded. For the discretization term with \(\rho<1\), \(\Xi_\rho(0;\sigma_K,\sigma_{\Khi})\leqslant
\sigma_{\Khi}^{1-\rho}/(1-\rho)\), so in the supercritical regime \(\beff>1-\rho\) it is
\(O(e^{\Lambda_K}h_K\sigma_{\Khi}^{1-\rho})\), independent of \(\sigma_0\) at fixed clock mesh.
Raising \(\sigma_0\) to shrink the initialization error therefore costs only through
\(h_K=(\sigma_0^\rho-\sigma_K^\rho)/K\) and through \(\Lambda_K\). In the critical and subcritical
regimes the scale function contributes an additional factor \(\log(\sigma_0/\sigma_{\Khi})\) or
\((\sigma_0/\sigma_{\Khi})^{1-\rho-\beff}\). At \(\rho=1\) every \(\beff>0\) is supercritical.
Because \(\kappahi<1\), the relevant exponent is \(\beff\), not \(\bhi\); but if \(\bhi>1-\rho\),
choosing \(\thetahi\) small enough that \(\bhi\kappa(\thetahi)>1-\rho\) makes the bound
supercritical on every \(G_K\) with \(h_K\leqslant \thetahi\sighi^\rho\).

\begin{proof}[Proof of \Cref{cor:budget_refinement_rates}]
The mesh is \(h_K=(\overline\sigma^\rho-\underline\sigma^\rho)/K\), and \cref{eq:logarithmic_budget}
gives \(e^{\Lambda_K}=O(K^{\budexp})\). The scale function is bounded uniformly in \(K\): since
\(\Xi_\rho\) is nonincreasing in \(b\) and additive over adjacent intervals,
\(\Xi_\rho(0;\sigma_K,\sigma_{\Khi})+\Xi_\rho(\beff;\sigma_{\Khi},\sigma_0)\leqslant\Xi_\rho(0;\sigma_K,\sigma_{\Khi})
+\Xi_\rho(0;\sigma_{\Khi},\sigma_0)=\Xi_\rho(0;\underline\sigma,\overline\sigma)\). Substitute these
bounds in the discretization term. For the oracle with exact initialization, \(e_0=0\) and
\(\Rres_j=0\) for every \(j\), so \Cref{ass:uniform_residual_accuracy} holds with \(\Cres=0\) and
\(\mathrm{Learn}_K=0\); the same estimate bounds
\(e_K=\Wtwo(\widehat\mu_K,\mu_{\underline\sigma})=O(K^{-(1-\budexp)})\). This assertion does not
require \Cref{ass:high_noise_contraction}: choosing \(\sighi>\overline\sigma\) makes the high-noise
block empty (\(\Khi=0\)), and the estimate then uses only \Cref{ass:wellposed} and the envelope
\eqref{eq:logarithmic_budget}. Adding the terminal smoothing bias \(\underline\sigma\sqrt d\) of
\cref{eq:edm_final_error_with_terminal_bias} gives the second assertion.
\end{proof}

\subsection{Initialization Error at the Top Noise Level}

For the canonical EDM initialization \(\nu_0=\widehat\mu_0=\mathcal N(0,\sigma_0^2\Idd)\) of
\Cref{sec:one_step_error}, a Gaussian at the top noise level centered at the data mean, which is
zero for the centered target of \Cref{sec:preconditioning_residual}, the initialization error
\(e_0=\Wtwo(\nu_0,\mu_{\sigma_0})\) is controlled by the top noise scale: \(e_0\leqslant \sqrt
d\,\sigmad\) for every \(\sigma_0>0\), and under a light-tail assumption \(e_0\) decays like
\(\sigma_0^{-1}\), with \(2\sigma_0\,e_0\to\norm{\Sigma_{\mathrm{data}}}_F\leqslant d\,\sigmad^2\)
as \(\sigma_0\to\infty\) (\Cref{prop:initialization_error}). In the bound it is further damped by
the high-noise kernel \(\Pi_{0,\Khi}\).

\begin{proposition}[Initialization error at the top noise level]
  \label{prop:initialization_error}
  Let \(q\in\mathcal P_2(\R^d)\) be centered, let \(X_0\sim q\), so that \(\E[X_0]=0\), and set
  \(\sigmad^2=d^{-1}\E\norm{X_0}^2\), as in \Cref{sec:preconditioning_residual}, and
  \(\mu_{\sigma_0}=\Law(X_0+\sigma_0Z)\) for an independent \(Z\sim\mathcal N(0,\Idd)\). Let
  \(\nu_0=\mathcal N(0,\sigma_0^2\Idd)\). Then, for every \(\sigma_0>0\),
  \begin{equation*}
    e_0=\Wtwo(\nu_0,\mu_{\sigma_0})\leqslant \sqrt d\,\sigmad.
  \end{equation*}
  If moreover \(X_0\) has light tails, in the sense that \(\E\exp(\xi\norm{X_0}^2)<\infty\) for some
  \(\xi>0\), then, with \(\Sigma_{\mathrm{data}}=\E[X_0X_0^\top]\),
  \begin{equation*}
    \lim_{\sigma_0\to\infty}2\sigma_0\,e_0 =\norm{\Sigma_{\mathrm{data}}}_F
    \leqslant \Tr \Sigma_{\mathrm{data}} =d\,\sigmad^2 .
  \end{equation*}
  In particular, \(e_0\) decays at the rate \(\sigma_0^{-1}\): for every \(\varepsilon>0\) there
  exists \(\sigma_\star<\infty\) such that
  \begin{equation*}
    e_0\leqslant \frac{\norm{\Sigma_{\mathrm{data}}}_F+\varepsilon}{2\sigma_0}
    \leqslant \frac{d\,\sigmad^2+\varepsilon}{2\sigma_0}
    \qquad\text{for all }\sigma_0\ge\sigma_\star .
  \end{equation*}
\end{proposition}

\begin{proof}
The first bound follows from the shared-noise coupling. Draw a single \(Z\sim\mathcal N(0,\Idd)\)
and couple \(X_{\sigma_0}=X_0+\sigma_0Z\sim\mu_{\sigma_0}\) with \(\sigma_0Z\sim\nu_0\). Then
\begin{equation*}
  \Wtwo(\nu_0,\mu_{\sigma_0})^2 \le \E\norm{(X_0+\sigma_0Z)-\sigma_0Z}^2 = \E\norm{X_0}^2
  =d\,\sigmad^2 ,
\end{equation*}
so \(e_0\leqslant \sqrt d\,\sigmad\). For the sharp \(\sigma_0^{-1}\) asymptotic, we use
\cite[Proposition~7]{beyler_convergence_2025}, which compares a scaled and Gaussian-smoothed random
vector with a centered Gaussian. Apply it to \(X=X_0\) with coefficients \((1,\sigma_0,\sigma_0)\)
in the notation there. Then \(X+\sigma_0Z=X_{\sigma_0}\) has law \(\mu_{\sigma_0}\), and
\(\sigma_0Z'\), with \(Z'\sim\mathcal N(0,\Idd)\), has law \(\nu_0\), so the distance compared there
equals \(\Wtwo(\nu_0,\mu_{\sigma_0})\). The second bound of that proposition requires \(X\) to be
centered, which holds since \(q\) is, and to satisfy the light-tail condition above, a hypothesis
inherited from the smoothed-Wasserstein asymptotics \cite[Theorem~2.1]{chen_asymptotics_2022} on
which its proof rests. With our coefficients it states that
\(\Wtwo(\nu_0,\mu_{\sigma_0})\sim\norm{\Sigma_{\mathrm{data}}}_F/(2\sigma_0)\) as
\(\sigma_0\to\infty\), which is the limit \(2\sigma_0\,e_0\to\norm{\Sigma_{\mathrm{data}}}_F\);
\cite[Corollary~8]{beyler_convergence_2025} states the same limit for data of bounded support. The
inequality \(\norm{\Sigma_{\mathrm{data}}}_F\leqslant \Tr \Sigma_{\mathrm{data}}=d\sigmad^2\) holds
because \(\Sigma_{\mathrm{data}}\) is positive semidefinite, and the final display restates the
limit with an threshold \(\sigma_\star\).
\end{proof}

\subsection{From the Positive-Noise Iterate to the Deployed Sample}
\label{subsec:deployed_output}

\Cref{thm:global_edm_wasserstein} controls \(\widehat\mu_K\), not the deployed sample. The deployed
EDM sampler appends a final step to \(\sigma=0\). In the frozen-\(\sigma\) predictor
\cref{eq:learned_edm_predictor} that step has relaxation \(a=1\), so it is exactly one evaluation of
the learned denoiser, \(x\mapsto\Dhat_{\sigma_K}(x)\), and the emitted law is
\((\Dhat_{\sigma_K})_\#\widehat\mu_K\). Write
\(\operatorname{mmse}_q(\sigma)=\E\norm{X_0-D_\sigma(X_0+\sigma Z)}^2\).

\begin{corollary}[Deployed-output closure]
  Adopt the hypotheses of \Cref{thm:global_edm_wasserstein}, and assume in addition that the
  terminal denoiser \(\Dhat_{\sigma_K}\) is globally \(L_K^D\)-Lipschitz. Then the law of the
  emitted sample satisfies
  \begin{equation*}
    \Wtwo\bigl((\Dhat_{\sigma_K})_\#\widehat\mu_K,q\bigr)
    \le \underbrace{L_K^De_K}_{\text{propagated sampler error}}
    + \underbrace{\sqrt d\,\Cres\sigma_K}_{\text{learning error}}
    + \underbrace{\sqrt{\operatorname{mmse}_q(\sigma_K)}}_{\text{posterior uncertainty}},
  \end{equation*}
  with \(e_K\leqslant e^{\Lambda_K}[\Pi_{0,\Khi}e_0+\mathrm{Disc}_K+\mathrm{Learn}_K]\), the bound
  on \(e_K\) established in the proof of \Cref{thm:global_edm_wasserstein}.
\end{corollary}

\begin{proof}
Insert \((\Dhat_{\sigma_K})_\#\mu_K\) and \((D_{\sigma_K})_\#\mu_K\) between the emitted law and
\(q\). Pushing an optimal coupling of \((\widehat\mu_K,\mu_K)\) through \(\Dhat_{\sigma_K}\) gives
the first term by the Lipschitz hypothesis. Coupling the learned and exact denoisers at the same
\(Y_K=X_0+\sigma_KZ\sim\mu_K\) gives, by \cref{eq:edm_den_error_residual}, the cost
\(\beta_K\Rres_K\leqslant \sqrt d\,\Cres\sigma_K\), since \(\sigma_K=\underline\sigma\) lies in the
range of \Cref{ass:uniform_residual_accuracy} and \(\beta_K\leqslant \sigma_K\). Finally
\((D_{\sigma_K}(Y_K),X_0)\) couples \((D_{\sigma_K})_\#\mu_K\) with \(q\) at cost
\(\operatorname{mmse}_q(\sigma_K)\).
\end{proof}

The Lipschitz hypothesis is a global condition on the terminal denoiser, of the same status as the
high-noise EDM criterion in \Cref{prop:high_noise_stability_certificate}. The last two terms can
also be read together. Write \(\widehat m(\sigma)=\sigma^{-2}\E\norm{\Dhat_\sigma(X_0+\sigma
Z)-X_0}^2\), the denoising error in units of the per-coordinate noise variance: it equals \(d\) for
the identity map, which removes no noise, and is small when the denoiser removes most of it. The
shared-noise coupling \((\Dhat_{\sigma_K}(X_0+\sigma_KZ),X_0)\) gives the direct diagnostic
\begin{equation*}
  \Wtwo\bigl((\Dhat_{\sigma_K})_\#\widehat\mu_K,q\bigr)
  \le L_K^De_K+\sigma_K\sqrt{\widehat m(\sigma_K)},
  \qquad \sigma_K^2\,\widehat m(\sigma_K)=\beta_K^2\Rres_K^2+\operatorname{mmse}_q(\sigma_K),
\end{equation*}
where the identity is the orthogonality of the posterior-mean error. Since
\(\operatorname{mmse}_q(\sigma_K)\leqslant \sigma_K^2d\), the closure matches the terminal bias
\(\sigma_K\sqrt d\) of \cref{eq:transport_global_bound} in the worst case, and improves on it only
when the final denoiser removes most of the noise, \(\widehat m(\sigma_K)\ll d\); its interest is to
separate the propagated, learning and posterior errors acting on the emitted law.

\section{Numerical Experiments}
\label{app:numerics}
\label{app:numerical_registry}

This appendix supports \Cref{sec:numerical_evidence}. \Cref{tab:evidence_status} lists what each
experiment measures. \Cref{subsec:numerical_definitions} defines the estimands and their numerical
approximation, and \Cref{subsec:numerical_full_figures} gives the full figures and tables summarized
in the main text. \Cref{subsec:numerical_supp_toy,%
subsec:analytic_controls,subsec:numerical_supp_cifar} report results that complement the three
findings of the main text. \Cref{subsec:numerical_methods} collects the checks establishing that the
reported numbers are numerically resolved, including the two that fail.
\Cref{subsec:numerical_reproducibility} describes the code and data.

\begin{table}[H]
  \centering
  \small
  \caption{Quantities of the analysis and what the two experiments measure of them.}
  \label{tab:evidence_status}
  \renewcommand{\arraystretch}{1.15}
  \begin{tabular}{@{}>{\raggedright}p{5.4cm}>{\raggedright}p{4.2cm}>{\raggedright\arraybackslash}p{6.2cm}@{}}
    \toprule
    quantity of the analysis & one-dimensional mixture & CIFAR--10 \\
    \midrule
    final error \(e_K\), factors \(\gamma_{j,K}\), \(\Lambda_K\) (\Cref{def:low_noise_budget})
      & measured by quadrature & not accessible \\
    one-step recursion \eqref{eq:law_level_one_step}
      & every term measured & not accessible \\
    directional majorant \(\Lambda_K^{\mathrm{dir}}\) \cref{eq:law_level_summed_condition}
      & measured & displacement rate \(\lambda^\parallel\) on coarse--fine segments
        \cref{eq:segment_rate_main} \\
    field-level majorant \eqref{eq:field_level_summed_condition}
      & finite-domain proxy & worst-direction rate \(\lambda^{\max}\) \cref{eq:segment_rate_main};
        empirical EDM margin \eqref{eq:probe_margin_main} \\
    stability hypotheses
      & tested quantitatively & not certified \\
    \bottomrule
  \end{tabular}
\end{table}

\subsection{Experimental Definitions and Implementation}
\label{subsec:numerical_definitions}

\subsubsection{Shared Schedule and Implementation}

Unless varied in the analytic controls, experiments use the nested EDM schedule
\begin{equation*}
  \sigma_j=\left(\overline\sigma^{1/7}
  +\frac{j}{K}\left(\underline\sigma^{1/7}-\overline\sigma^{1/7}\right)\right)^7,
  \qquad j=0,\ldots,K, \qquad (\overline\sigma,\underline\sigma)=(80,0.002).
\end{equation*}
Consequently \(G_K\subset G_{2K}\), so the coarse and fine trajectories in the synchronous
experiment meet at exactly the same noise levels. Grid construction and the one-dimensional audit
use float64. The CIFAR--10 network, sampler, and Jacobian actions use float32, matching inference;
subtraction of recorded states and outputs, scalar products, and moment reductions use float64 to
limit cancellation and summation error, which does not recover information already lost during
float32 inference. The CIFAR experiments use the public class-conditional CIFAR--10 EMA checkpoint
released with \textcite{karras_elucidating_2022}. Random seeds and noisy probes are held fixed
across noise levels and refinements.

\subsubsection{One-Dimensional Audit}
\label{exp:toy_law}

\paragraph{Target law and denoisers.}
The target is the four-component Gaussian mixture with weights \((0.15,0.35,0.35,0.15)\), centres
\(m_r\) proportional to \((-3,-1,1,3)\), and common component standard deviation \(\varsigma=0.04\).
The centres are rescaled so that the mixture is centered with \(\sigmad=0.5\), so the EDM
coefficients \cref{eq:edm_coefficients} use the exact \(\sigmad\). The main comparison uses the
exact posterior-mean denoiser and three learned denoisers with seeds 0, 1, and 2. Each learned
denoiser is an EDM-preconditioned three-hidden-layer MLP of width 128 with SiLU activations. It is
trained for 20,000 Adam steps with batches of 4096, initial learning rate \(10^{-3}\), cosine decay,
gradient clipping at one, and an exponential moving average of rate \(0.999\). Training noise
follows the EDM log-normal law with \((P_{\rm mean},P_{\rm std})=(-1.2,1.2)\). Two analytic arms
complete the set and require no training. The misspecified arm is the posterior mean of a
deliberately wrong mixture, with weights \((0.2,0.3,0.3,0.2)\), centres \(1.02\,m_r\) and component
standard deviation \(0.7\varsigma\); it provides a discrepancy dominated by learning error rather
than by discretization. A perturbed oracle adds a smooth, bounded perturbation of controlled
amplitude to the exact normalized residual,
\begin{equation}\label{eq:toy_perturbed_oracle}
  \widehat D^{\varepsilon}_\sigma(x)
  =D_\sigma(x)+\varepsilon\,\beta(\sigma)\tanh\bigl(\cin(\sigma)x\bigr),
  \qquad \varepsilon\in\{0.003,\,0.01,\,0.03,\,0.1\},
\end{equation}
so that \(\Rres_j\leqslant \varepsilon\) at every level and the perturbation of the learned velocity
has slope at most \(\varepsilon/\sigmad\). Unlike the oracle discretization error, the learning
contribution of \cref{eq:toy_perturbed_oracle} does not decrease under refinement;
\Cref{subsec:numerical_supp_toy} reports the effect of the four amplitudes on \(\Lambda_K\).

\paragraph{Quantile representation.}
For \(K\in\{17,34,68,136,272,544,1088\}\), each target and sampler law is represented at
\(M=65{,}536\) midpoint quantiles. Target quantiles are obtained from the mixture CDF; Gaussian or
exact-target initialization quantiles are propagated through the predictor and sorted after every
step. In one dimension this is the monotone optimal coupling.

\paragraph{Directional quantities.}
At step \(j\), let \((U,V)\) denote the monotone coupling of the target and sampler laws, set
\(e_{j,K}^2=\E|U-V|^2\), and write \(\mathsf v=\vhat_{\sigma_j}(U)-\vhat_{\sigma_j}(V)\) for the
velocity increment, as in \cref{eq:directional_step_constants}. The stepwise quantities are
\begin{equation*}
  \Erate[\OT]{j,K}=-\frac{\E[(U-V)\mathsf v]}{e_{j,K}^2},
  \qquad \Qsec[\OT]{j,K}=\frac{\E|\mathsf v|^2}{e_{j,K}^2},
  \qquad \gamma_{j,K}=\frac{W_2((\widehat\Phi_j)_\#\mu_j,
  (\widehat\Phi_j)_\#\widehat\mu_j)}{e_{j,K}}.
\end{equation*}
Up to quantile quadrature, \cref{eq:law_level_stepwise_log} gives
\begin{equation}\label{eq:toy_directional_chain}
  \log\max\{1,\gamma_{j,K}\} \leqslant \ell_j(\Erate[\OT]{j,K})_+
  +\frac{\ell_j^2}{2}\Qsec[\OT]{j,K}.
\end{equation}
We also record the root-mean-square amplification of the fixed input coupling,
\begin{equation}\label{eq:toy_rms_amplification}
  \Amp[\OT]{j,K}=\frac{\|\Phihat_j(U)-\Phihat_j(V)\|_{\mathbb L_2}}{e_{j,K}}\ \ge\ \gamma_{j,K}.
\end{equation}
By the predictor identity \eqref{eq:predictor_secant_identity},
\((\Amp[\OT]{j,K})^2=1+2\ell_j\Erate[\OT]{j,K}+\ell_j^2\Qsec[\OT]{j,K}\), so the slack in
\cref{eq:toy_directional_chain} splits into the effect of reoptimizing the output coupling, from
\(\gamma_{j,K}\) to \(\Amp[\OT]{j,K}\), and the algebraic relaxation that follows.
\(\Amp[\OT]{j,K}\) equals \(\gamma_{j,K}\) for the oracle and misspecified posterior mean because
their predictor maps are increasing, and in the recorded runs it coincides with \(\gamma_{j,K}\) on
every arm and step.

\paragraph{Finite-domain quantities.}
For the mixture centres \(m_r\) and component standard deviation \(\varsigma\), the spatial search
uses
\begin{equation*}
  I_\sigma(R)=\left[\min_r m_r-R\sqrt{\varsigma^2+\sigma^2},\,
  \max_r m_r+R\sqrt{\varsigma^2+\sigma^2}\right].
\end{equation*}
This interval depends on the target law and \(\sigma\), but not on the sampler law or on \(K\). On
\(N_x\) equispaced points of this interval, augmented by bounded scalar refinement around the four
largest grid candidates, we compute
\begin{equation*}
  \Los[\pts]{j,K}=\max_{x\in\mathcal X_{j,K}^{\pts}} [-\vhat'_{\sigma_j}(x)],
  \qquad L_{v,j,K}^{\pts}=\max_{x\in\mathcal X_{j,K}^{\pts}} |\vhat'_{\sigma_j}(x)|,
\end{equation*}
where \(\mathcal X_{j,K}^{\pts}\) is the resulting finite set at \(\sigma=\sigma_j\). The same
search gives the finite-domain EDM margin
\(b^{\pts}(\sigma)=1-\alpha(\sigma)-\max_x|\Dhat'_\sigma(x)-\alpha(\sigma)|\): since
\(\Dhat'_\sigma(x)-\alpha(\sigma)=\beta(\sigma)\cin(\sigma)\Fhat'_\sigma(\cin(\sigma)x)\) by
\cref{eq:edm_learned_denoiser_step}, it is the margin \eqref{eq:def_b_k_edm} with the Lipschitz
constant of the normalized branch replaced by the maximum of its slope over the finite point set.
The main search uses \((R,N_x)=(12,4097)\).

\paragraph{Noise threshold.}
For each denoiser, a candidate threshold is computed before any accumulated quantity: on 400
geometrically spaced noise levels, it is the bottom of the top contiguous block, starting at
\(\overline\sigma\), on which \(b^{\pts}(\sigma)\ge\bhi=0.5\). There are four distinct candidates:
\(0.876\) for the six analytic denoisers, and \(1.53\), \(1.57\) and \(1.84\) for the three trained
networks. The largest, \(\sighi=1.84\), is used as a common threshold for every denoiser, refinement
and initialization, and determines the split index \(\Khi\) of each grid through
\cref{eq:def_split_index}. This is an empirical split, not a global certificate. The largest
candidate gives the largest low-noise block, hence the largest sums of nonnegative stepwise
contributions among the four candidate splits; every accumulated quantity is also re-summed at each
candidate (\Cref{tab:numerical_reliability}).

\paragraph{Accumulated quantities.}
Summing over the low-noise block defines the exact cumulative log-amplification \(\Lambda_K\) of
\Cref{def:low_noise_budget}, its directional bound
\begin{equation*}
  \Lambda_K^{\mathrm{dir}} \coloneqq \sum_{j=\Khi}^{K-1} \left[\ell_j(\Erate[\OT]{j,K})_+
  +\frac{\ell_j^2}{2}\Qsec[\OT]{j,K}\right],
\end{equation*}
and the finite-domain field proxy
\begin{equation}\label{eq:toy_field_proxy}
  \Lambda_K^{\mathrm{field},\pts} \coloneqq \sum_{j=\Khi}^{K-1} \left[\ell_j(\Los[\pts]{j,K})_+
  +\frac{\ell_j^2}{2}(L_{v,j,K}^{\pts})^2\right].
\end{equation}
Because both the domain and the evaluations are finite, \(\Lambda_K^{\mathrm{field},\pts}\) is a
proxy for the field-level bound, not a certified upper bound on \(\mathbb R\); in particular, the
empirical ordering with \(\Lambda_K^{\mathrm{dir}}\) is an outcome to check rather than an identity
imposed by notation. The geometry panel of \Cref{fig:numerical_hierarchy} uses \(G_{272}\), a
representative level near \(\sigma=0.2\), and 1025 spatial points. The error weight \(w_j\) is
accumulated segment by segment over the \(M\) quantile pairs, and its Riemann integral against
\(\lambda_j\) reproduces the logged \(\Erate[\OT]{j}\) at that level to within \(10^{-3}\).

\paragraph{End-to-end recursion.}
The same quantiles make every term of \cref{eq:law_level_one_step} observable. Let
\(u_m=(m-1/2)/M\), let \(T_{j,m}\) be the \(u_m\)-quantile of \(\mu_j\), and let \(T_{j+1,m}\) be
the corresponding quantile at the next noise level. In one dimension the exact flow preserves
quantile rank, so \(\Psi_j(T_{j,m})=T_{j+1,m}\). In addition to \(e_{j,K}\) and \(\gamma_{j,K}\), we
compute the measured local bias, the quantile quadrature of \(\delta_{j,K}\) in
\cref{eq:def_local_defect_learned_map}, and the measured learning term
\begin{align*}
  \widetilde\delta_{j,K}^2 &=\frac1M\sum_{m=1}^M \left|T_{j+1,m}-\Phi_j(T_{j,m})\right|^2, \\
  \widetilde r_{j,K}^2 &=\frac{a_j^2}{M}\sum_{m=1}^M
  \left|\widehat D_{\sigma_j}(T_{j,m})-D_{\sigma_j}(T_{j,m})\right|^2.
\end{align*}
Thus, up to the common quantile quadrature,
\begin{equation}\label{eq:toy_measured_recursion}
  e_{j+1,K} \leqslant \gamma_{j,K}e_{j,K}+\widetilde\delta_{j,K}+\widetilde r_{j,K}.
\end{equation}
We record the one-step slack and unroll the measured raw factors, without truncating contraction,
into
\begin{align}
  C_K^{\rm init} &=e_{0,K}\prod_{k=0}^{K-1}\gamma_{k,K}, &C_K^{\rm disc}
  &=\sum_{j=0}^{K-1}\widetilde\delta_{j,K} \prod_{k=j+1}^{K-1}\gamma_{k,K}, \notag \\
  C_K^{\rm learn} &=\sum_{j=0}^{K-1}\widetilde r_{j,K} \prod_{k=j+1}^{K-1}\gamma_{k,K}, &e_{K,K}
  &\leqslant C_K^{\rm init}+C_K^{\rm disc}+C_K^{\rm learn}. \label{eq:toy_propagated_decomposition}
\end{align}
In the exact-initialization control, \(C_K^{\rm disc}\) is the recursion bound
\eqref{eq:recursion_bound_main} of the main text. These are contributions to an a posteriori upper
bound, not a decomposition of the observed error or the closed-form bound of
\Cref{thm:global_edm_wasserstein}. Stability is not treated as a fourth additive error source: its
effect is precisely the family of downstream products of \(\gamma_{k,K}\) weighting the three terms.
The ratio of their sum to \(e_{K,K}\) is reported only when \(e_{K,K}\) is numerically resolved
(\Cref{tab:numerical_reliability}); no floor is applied to produce a finite ratio otherwise.

The canonical run starts the sampler from the canonical EDM initialization \(\mathcal
N(0,\overline\sigma^2)\) of \Cref{prop:initialization_error}, which enters
\Cref{thm:global_edm_wasserstein} through \(e_0\). A separate oracle control starts instead from the
exact law \(\mu_{\overline\sigma}\). In that control \(C_K^{\rm init}=C_K^{\rm learn}=0\), so both
the observed final error and \(C_K^{\rm disc}\) isolate numerical discretization. This control
records its own \(\gamma_{j,K}\) and \(\Lambda_K\); the amplification of the Gaussian-initialized
run is not substituted in the error--amplification comparison. We also record
\(\sum_j\widetilde\delta_{j,K}\) and the propagated universal bias bound
\(\sum_j\Delta_j\prod_{k=j+1}^{K-1}\gamma_{k,K}\), using the same raw propagation factors. For each
dyadic pair \((K/2,K)\) we report the local orders \cref{eq:empirical_order_main} of the
exact-initialization final error, of \(C_K^{\rm disc}\) and of \(\sum_j\widetilde\delta_{j,K}\), and
compare them with the amplification-based order \(1-(\Lambda_K-\Lambda_{K/2})/\log2\), the order
that \(K^{-1}e^{\Lambda_K}\) would have on the same pair (\Cref{fig:numerical_toy}c). It is computed
from the measured \(\Lambda_K\), so the comparison is an a posteriori consistency check, not a
prediction; equivalently, it tests whether \(c_K=Ke_Ke^{-\Lambda_K}\) is stable on the tested range
(\Cref{tab:toy_refinement}). We do not assign a convergence order to the total learned-sampler error
once initialization or learning sets a visible floor.

\subsubsection{High-Noise CIFAR--10 Diagnostics}

Neither the optimal coupling between the exact and sampler laws nor the factors \(\gamma_{j,K}\) are
computed on CIFAR--10. The coefficients \(\alpha,\beta,\cin\) of \cref{eq:edm_coefficients} are
those of the checkpoint, with \(\sigmad\) replaced by the hyperparameter \(\sigma_{\rm
data}=\auditSigmaData\): CIFAR--10 is neither centered nor of root-mean-square exactly
\(\auditSigmaData\) per coordinate, so this value is not \(\sigmad\), and the margin and residual
below are those of the deployed parametrization. The high-noise calculation probes a local
preconditioning mechanism on held-out noisy data. The probe bank consists of \(N=128\) held-out
CIFAR--10 test images with independent Gaussian noises, reused at every level. For
\(Y_{\sigma,i}=X_i+\sigma Z_i\), the supervised residual per coordinate is
\begin{equation*}
  r_{\rm reg}(\sigma)
  =\frac{1}{\beta(\sigma)}\left(\frac{1}{dN}\sum_{i=1}^N
  \|\widehat D_\sigma(Y_{\sigma,i})-X_i\|^2\right)^{1/2}.
\end{equation*}
It is evaluated at all 273 levels of \(G_{272}\). The population risk includes the excess risk and
irreducible posterior uncertainty (\cref{eq:edm_regr_pyth}); it is an upper proxy for the learning
residual, not a measurement of that residual alone.

At the same probes, eight power iterations from each of two random starts estimate the local stretch
\(\SFp{i}(\sigma)=\norm{\nabla\Fhat_\sigma(\cin(\sigma)Y_{\sigma,i})}_\op\) of the raw EDM network
branch in normalized input; the larger of the two estimates is denoted \(\SFpest{i}(\sigma)\). With
the deployed preconditioning coefficients and the empirical \(p\)-quantile
\(\SFest[p](\sigma)=\Quant{p}{\SFpest{i}(\sigma)}{i}\) over the probes, the empirical margin
\eqref{eq:probe_margin_main} is
\begin{equation}\label{eq:branch_stability_margin}
  \bprobe[p](\sigma)=1-\alpha(\sigma) -\beta(\sigma)\cin(\sigma)\SFest[p](\sigma).
\end{equation}
It is evaluated for \(p=0.5\) and \(p=0.9\) at the 90 levels of \(G_{136}\) satisfying
\(\sigma\ge0.5\). Since the map from stretch to margin is decreasing and affine, \(\bprobe[0.9]\) is
the 10th percentile of the per-probe margins \(1-\alpha-\beta \cin\SFpest{i}\). Compared with the
EDM margin \eqref{eq:def_b_k_edm}, it replaces the Lipschitz constant \(\LF{j}\), a supremum over
all inputs, by a percentile over 128 probes, and each probe stretch by a power-iteration estimate,
which cannot exceed it. Both substitutions can only raise the margin, so the empirical margin is an
optimistic estimate of \(b_j\), not a certificate; its comparison with \(6/7\) is discussed in
\Cref{subsec:evidence_damping}.

\subsubsection{Low-Noise Synchronous-Segment Diagnostic}
\label{exp:transport}

For each \(K\in\{17,34,68,136,272\}\), each of the ten classes, and each of 128 fixed latent
variables per class, a predictor trajectory on \(G_K\) is paired with the predictor trajectory on
\(G_{2K}\) generated from the same latent; the ensemble thus has 1280 trajectories per grid. At a
level \(\sigma_j\) of \(G_K\subset G_{2K}\) let
\begin{equation*}
  x_i=\widehat X_{j,i}^{(K)},\qquad y_i=\widehat X_{2j,i}^{(2K)},\qquad \zeta_i=y_i-x_i,\qquad
  \mathsf v_i=\vhat_{\sigma_j}(y_i)-\vhat_{\sigma_j}(x_i).
\end{equation*}
Pooling raw second moments over the ensemble gives
\begin{equation*}
  \Erate[\sync]{j,K}=-\frac{\sum_i\langle\zeta_i,\mathsf v_i\rangle} {\sum_i\|\zeta_i\|^2},
  \qquad \Qsec[\sync]{j,K}=\frac{\sum_i\|\mathsf v_i\|^2} {\sum_i\|\zeta_i\|^2},
\end{equation*}
and the predictor identity \eqref{eq:predictor_secant_identity} gives, for the root-mean-square
amplification
\(\Amp[\sync]{j,K}=(\sum_i\|\Phihat_j(y_i)-\Phihat_j(x_i)\|^2/\sum_i\|\zeta_i\|^2)^{1/2}\) of the
predictor along this coupling,
\begin{equation}\label{eq:transport_sync_secant_identity}
  (\Amp[\sync]{j,K})^2 =1+2\ell_j\Erate[\sync]{j,K} +\ell_j^2\Qsec[\sync]{j,K}.
\end{equation}
Accumulated synchronous quantities are summed over the levels \(\sigma_j\le2\) of \(G_K\), the
window also used in the calibration of \Cref{subsec:analytic_controls}.

\paragraph{Segment chain.}
The rates below are the rate \eqref{eq:segment_rate_main} of the main text, written per trajectory
and per level. The segment calculation uses only \(K=68,272\). On each grid the starting levels are
those nearest in log noise to \(0.004,0.01,0.03,0.1,0.3,1,3,8\), with duplicate indices removed and
actual levels reported. The last two levels examine the transition region; this set is not the
theoretical low-noise block. The first eight fixed latent variables in each class are used, hence 80
segments per selected level; \(\Erate[\sync]{j,K}\) is recomputed on exactly this subset for the
chain below. For a nonzero segment define \(u_i=\zeta_i/\|\zeta_i\|\) and
\begin{align*}
  \lambda_{i,j}^{\parallel}(t) &=-u_i^\top\operatorname{Sym}\nabla\widehat v_{\sigma_j}
  (x_i+t\zeta_i)u_i, \\
  \lambda_{i,j}^{\max}(t) &=\lambda_{\max}\!\left(-\operatorname{Sym}\nabla\widehat v_{\sigma_j}
  (x_i+t\zeta_i)\right).
\end{align*}
With energy weights \(w_i=\|\zeta_i\|^2/\sum_k\|\zeta_k\|^2\), the fundamental theorem of calculus
along each segment gives the equality below; the first inequality bounds an average by a supremum,
and the second a Rayleigh quotient by the largest eigenvalue. For the continuous quantities the
chain
\begin{equation}\label{eq:segment_aggregate_chain}
  \Erate[\sync]{j,K} =\sum_iw_i\int_0^1\lambda_{i,j}^{\parallel}(t)\,\dd t
  \leqslant \sum_iw_i\sup_{t\in[0,1]}\lambda_{i,j}^{\parallel}(t)
  \leqslant \sum_iw_i\sup_{t\in[0,1]}\lambda_{i,j}^{\max}(t),
\end{equation}
is therefore exact; the synchronous rate and the two segment quantities use the same states and the
same energy weights.

In dimension \(d=3072\) a Rayleigh quotient far below the largest eigenvalue is the default behavior
of any generic direction, so the gap between \(\lambda^{\parallel}\) and \(\lambda^{\max}\) is
uninformative on its own. For each segment a unit direction \(u_i^{\rm rand}\) is therefore drawn
once, uniformly on the sphere with a fixed seed, and
\begin{equation*}
  \lambda_{i,j}^{\rm rand}(t)
  =-(u_i^{\rm rand})^\top\operatorname{Sym}\nabla\widehat v_{\sigma_j}
  (x_i+t\zeta_i)\,u_i^{\rm rand}
\end{equation*}
is evaluated at the same nodes, with the same energy weights. A synchronous direction
indistinguishable from a random one carries no directional information, while one that sits between
the baseline and the maximum is partially aligned with the expansive subspace.

\paragraph{Discrete segment estimators.}
The continuous suprema in \cref{eq:segment_aggregate_chain} are approximated on the nine nodes
\(T_9=\{0,1/8,\ldots,1\}\); the nested sets \(T_5=\{0,1/4,1/2,3/4,1\}\) and \(T_3=\{0,1/2,1\}\)
diagnose resolution in \(t\), and the trapezoidal rule on \(T_9\) approximates the integral.
Jacobian--vector products evaluate \(\lambda_{i,j}^{\parallel}\) and \(\lambda_{i,j}^{\rm rand}\);
32 Lanczos steps estimate \(\lambda_{i,j}^{\max}\) by the largest Ritz value. Two convergence
diagnostics accompany it: the half-iteration drift, the change in the largest Ritz value between 16
and 32 steps, and the Ritz residual, the norm of the residual of the corresponding Ritz pair. Only
the equality in \cref{eq:segment_aggregate_chain} is checked against an independent computation, the
trapezoidal integral against \(\Erate[\sync]{j,K}\). The two inequalities are exact for the
continuous quantities, but their computed versions are not guaranteed: both the finite-node maximum
and an unconverged largest Ritz value can underestimate the continuous supremum, so the computed
inequalities are empirical consistency checks, not certified versions of
\cref{eq:segment_aggregate_chain}. The segment calculation requires at most
\(2\times8\times80\times9\times32=368{,}640\) Lanczos operator applications, about one sixteenth of
evaluating every low-noise level on all five grids, plus two Jacobian--vector products per node for
the directional and random rates.

\paragraph{Sampling and influence statistics.}
The primary estimator is untrimmed, and a dominant trajectory is retained: its removal is a
sensitivity check only. We report the effective sample size
\begin{equation}\label{eq:segment_energy_weights}
  \operatorname{ESS}=\frac{1}{\sum_iw_i^2}
\end{equation}
together with \(\max_iw_i\) and the largest leave-one-out change. Pointwise sensitivity intervals
use the 2.5th and 97.5th percentiles of 1000 bootstrap resamples within classes, retaining eight
segments per class, recomputing energy weights, and using the same resample for all terms of a
comparison. These intervals describe sampling sensitivity, not Jacobian approximation error. The
fixed subset may miss influential trajectories in the full ensemble, so its geometry is not
presented as representative of all rare events.

Only the full ensemble is accumulated over the whole window. Its linear term
\(\sum_j\ell_j(\Erate[\sync]{j,K})_+\) is plotted in \Cref{fig:numerical_cifar}d; its quadratic term
\(\tfrac12\sum_j\ell_j^2\Qsec[\sync]{j,K}\) is not plotted and stays below \(0.09\) on every grid
(\(0.087\) at \(K=17\), \(0.038\) at \(K=272\)). Each accumulated rate is recomputed after removing
each of the 1280 trajectories in turn, always removing the same latent trajectory at every noise
level and renormalizing within the affected class. The minimum and maximum measure single-trajectory
influence, not statistical confidence; a small range does not rule out the joint influence of
several rare events. No cumulative field curve is inferred from the eight sparse segment levels. As
a descriptive complement we also accumulate, over the same window, the median rate
\begin{equation*}
  \sum_{j:\,\sigma_j\le2}\ell_j
  \Bigl(\operatorname{median}_i\Bigl(-\frac{\langle\zeta_i,\mathsf v_i\rangle}{\|\zeta_i\|^2}\Bigr)\Bigr)_+,
\end{equation*}
where at each level the median runs over the trajectories with \(\zeta_i\neq0\). The median is taken
level by level, so it does not follow a single trajectory, and no trajectory is selected by outcome.
This statistic is not the energy-weighted rate of \cref{eq:segment_aggregate_chain} and never
replaces it in an inequality.

Zero displacements carry zero energy and have no defined individual rate. Non-finite states,
outputs, or positive-energy rates invalidate the corresponding aggregate rather than being dropped
or replaced by zero, and displacement size relative to states and outputs is recorded to flag
subtraction near float32 resolution. A reproducible large discrepancy is part of the measured
phenomenon; any unresolved discrepancy is reported as missing or qualified evidence, never as a
favorable stability estimate.

\subsection{Full Figures and Tables}
\label{subsec:numerical_full_figures}

\begin{table}[H]
  \centering
  \small
  \caption{Oracle denoiser with exact initialization, under refinement. Orders are the local
  orders \cref{eq:empirical_order_main} on the pair \((K/2,K)\); the amplification-based order is
  \(1-(\Lambda_K-\Lambda_{K/2})/\log2\), and \(c_K=Ke_Ke^{-\Lambda_K}\). Bound\,/\,error divides the
  recursion bound \eqref{eq:recursion_bound_main} by \(e_K\), with the measured local biases
  \(\widetilde\delta_{j,K}\) or with their universal bounds \(\Delta_j\) (\(B_{1,1}=16\)),
  propagated by the same measured factors.}
  \label{tab:toy_refinement}
  \begin{tabular}{@{}r rr rr rrr rr@{}}
\toprule
 & \multicolumn{2}{c}{final error} & \multicolumn{2}{c}{summed biases} & \multicolumn{3}{c}{amplification} & \multicolumn{2}{c}{bound\,/\,error} \\
\cmidrule(lr){2-3}\cmidrule(lr){4-5}\cmidrule(lr){6-8}\cmidrule(l){9-10}
\(K\) & \(e_K\) & order & \(\sum_j\widetilde\delta_{j,K}\) & order & \(\Lambda_K\) & ampl.\ order & \(c_K\) & measured & universal \\
\midrule
17 & 0.113 & \multicolumn{1}{c}{--} & 0.166 & \multicolumn{1}{c}{--} & 0.325 & \multicolumn{1}{c}{--} & 1.388 & 1.35 & 134 \\
34 & 0.0747 & 0.597 & 0.0934 & 0.826 & 0.548 & 0.679 & 1.469 & 1.40 & 127 \\
68 & 0.0497 & 0.590 & 0.0496 & 0.913 & 0.795 & 0.643 & 1.524 & 1.40 & 122 \\
136 & 0.0332 & 0.583 & 0.0255 & 0.959 & 1.065 & 0.612 & 1.555 & 1.39 & 119 \\
272 & 0.0222 & 0.581 & 0.0129 & 0.981 & 1.344 & 0.596 & 1.572 & 1.38 & 118 \\
544 & 0.0148 & 0.581 & 0.00650 & 0.991 & 1.630 & 0.589 & 1.580 & 1.37 & 117 \\
1088 & 0.00988 & 0.585 & 0.00326 & 0.996 & 1.915 & 0.588 & 1.583 & 1.36 & 116 \\
\bottomrule
\end{tabular}

\end{table}

\begin{figure}[H]
  \centering
  \includegraphics[width=\textwidth]{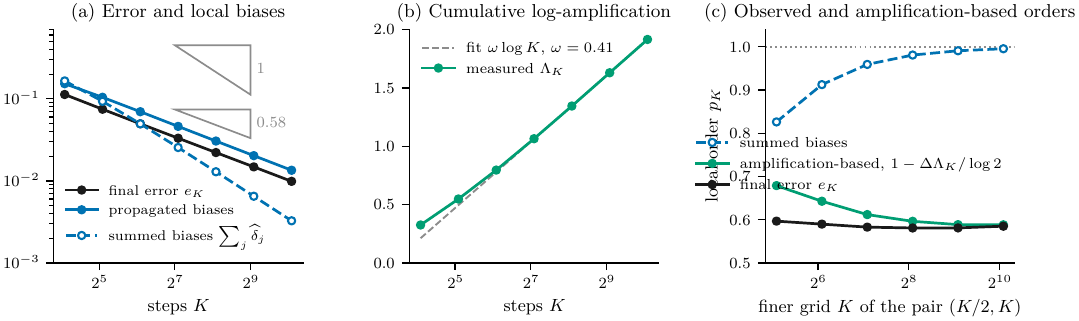}
  \caption[Amplification and the observed rate on the one-dimensional mixture]{
  Oracle denoiser on the four-component mixture, exact initialization, \(M=65{,}536\) quantiles.
  \textbf{(a)} Final error \(e_K\), the recursion bound \eqref{eq:recursion_bound_main} built from
  the measured local biases and factors, and the unweighted sum of the measured biases; triangles
  mark slopes \(1\) and \(\resErrorOrderMin\). \textbf{(b)} Measured \(\Lambda_K\), with a
  logarithmic fit on \(K\ge\resOmegaFitK\). \textbf{(c)} Local orders \cref{eq:empirical_order_main}
  on each pair \((K/2,K)\), against the amplification-based order
  \(1-(\Lambda_K-\Lambda_{K/2})/\log2\) computed from (b). The summed biases converge at first
  order; the final error does not, and the growth of \(\Lambda_K\) accounts for the gap.}
  \label{fig:numerical_toy}
\end{figure}

\begin{table}[H]
  \centering
  \small
  \caption{Measured \(\Lambda_K\) and its two majorants, as ratios to \(\Lambda_K\), on the
  coarsest and finest grids. Directional: \(\Lambda_K^{\mathrm{dir}}\)
  \cref{eq:law_level_summed_condition}; field: the finite-domain proxy for
  \cref{eq:field_level_summed_condition}. Canonical Gaussian initialization, common noise
  threshold.}
  \label{tab:toy_hierarchy}
  \begin{tabular}{@{}l rrr rrr@{}}
\toprule
 & \multicolumn{3}{c}{\(K=17\)} & \multicolumn{3}{c}{\(K=1088\)} \\
\cmidrule(lr){2-4}\cmidrule(l){5-7}
denoiser & \(\Lambda_K\) & directional & field & \(\Lambda_K\) & directional & field \\
\midrule
oracle & 0.325 & 1.976 & 143 & 1.921 & 1.022 & 11.1 \\
trained, seed 0 & 0.324 & 1.970 & 27.5 & 1.328 & 1.020 & 5.1 \\
trained, seed 1 & 0.334 & 1.934 & 28.0 & 1.875 & 1.017 & 3.8 \\
trained, seed 2 & 0.320 & 1.990 & 29.6 & 1.879 & 1.018 & 3.8 \\
misspecified mean & 0.205 & 3.938 & 944 & 0.673 & 1.012 & 75.6 \\
\bottomrule
\end{tabular}

\end{table}

For the oracle, \(D'_\sigma\) is a posterior variance divided by \(\sigma^2\), hence nonnegative, so
the predictor preserves order and the optimal coupling remains optimal after each step: the gap
between \(\Lambda_K\) and \(\Lambda_K^{\mathrm{dir}}\) comes only from the relaxation
\(\log(1+x)\leqslant x\) and the positive part in \cref{eq:law_level_stepwise_log}. On the coarsest
grid most of the excess of the field proxy is the quadratic term
\(\tfrac12\ell_j^2\operatorname{Lip}(\vhat_{\sigma_j})^2\), which vanishes under refinement; what
remains on the finest grid is first order and geometric. In one dimension both rates read off the
local expansion rate \(\lambda_j(x)=(D'_{\sigma_j}(x)-1)/\sigma_j\) of the oracle velocity:
\(\Los{j}\) is its supremum over \(x\), and the directional rate is its average
\(\Erate[\OT]{j}=\int\lambda_j\,w_j\) against the error weight
\begin{equation*}
  w_j(x)=\frac{\E\bigl[\lvert U_j-V_j\rvert\,\mathbf 1\{x\text{ between }V_j\text{ and }U_j\}\bigr]}{e_j^2},
\end{equation*}
a probability density that spreads each coupled pair along its segment
(\Cref{fig:numerical_hierarchy}a). \(\Lambda_K\) itself depends on the denoiser
(\Cref{fig:numerical_hierarchy}b): the trained networks stay close to the oracle, except seed 0,
whose \(\Lambda_K\) is lower (\(1.33\) against \(1.88\) for the other two seeds at \(K=\resKmax\),
\Cref{tab:toy_hierarchy}), whereas the misspecified posterior mean, whose error is dominated by
learning, settles near \(\resMisspecLambdaFine\).

\begin{figure}[H]
  \centering
  \includegraphics[width=\textwidth]{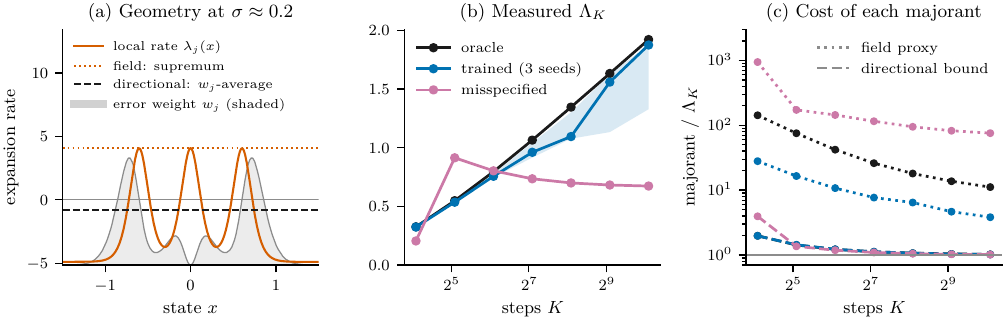}
  \caption[Stability controls on the one-dimensional mixture]{
  Where expansion sits, and what each stability control costs. Canonical Gaussian initialization.
  \textbf{(a)} Oracle on \(G_{272}\) at one low noise level: local expansion rate \(\lambda_j(x)\),
  the error weight \(w_j\) (shaded, arbitrary scale), and the two rates each bound pays at this
  step, the supremum of \(\lambda_j\) (field) and its \(w_j\)-average \(\Erate[\OT]{j}\)
  (directional). The error sits where the predictor contracts. \textbf{(b)} Measured \(\Lambda_K\)
  by denoiser; trained networks as median and range over three seeds. \textbf{(c)} Directional bound
  (dashed) and finite-domain field proxy (dotted), divided by \(\Lambda_K\), with the colours of
  (b).}
  \label{fig:numerical_hierarchy}
\end{figure}

\begin{figure}[H]
  \centering
  \includegraphics[width=\textwidth]{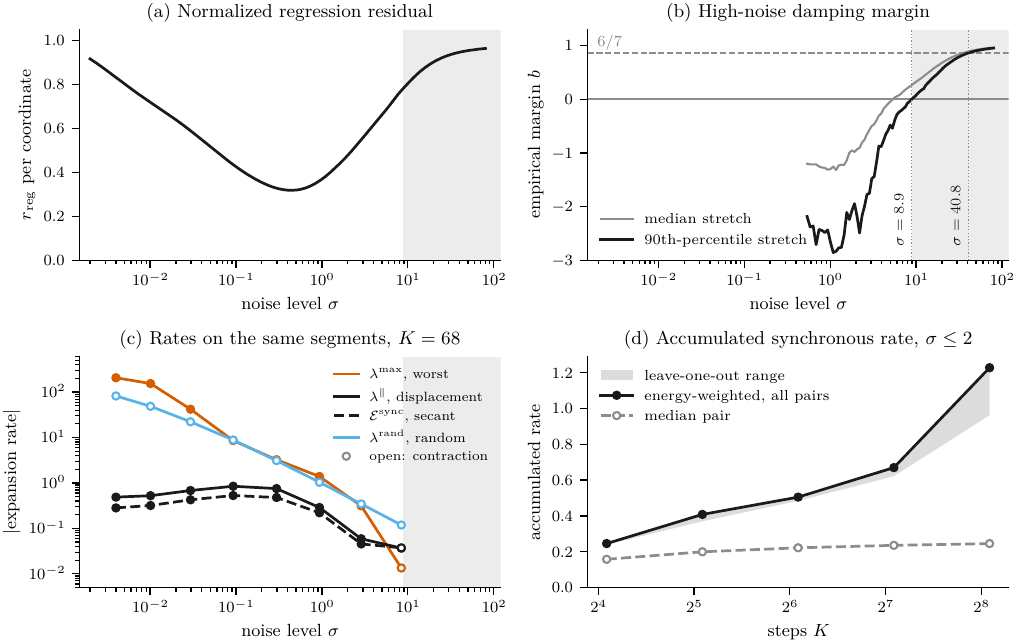}
  \caption[Stability mechanisms on CIFAR--10]{
  Stability diagnostics for the pretrained CIFAR--10 EDM model. Shading marks the noise levels at
  which the 90th-percentile empirical margin of (b) is positive. \textbf{(a)} Normalized denoising
  loss \(r_{\mathrm{reg}}\) per coordinate. \textbf{(b)} Empirical margin
  \eqref{eq:probe_margin_main} from the median and the 90th percentile of the local stretch; dotted
  lines mark where the latter crosses \(0\) and the supercritical threshold \(6/7\). \textbf{(c)}
  Energy-weighted rates on the coarse--fine segments of \(G_{\auditSegK}\), at the same nodes: worst
  direction, displacement, secant, and random direction \cref{eq:segment_rate_main}. Absolute
  values; open markers are negative rates. \textbf{(d)} Accumulated synchronous rate over
  \(\sigma\le2\), energy-weighted over all \auditSyncTrajectories trajectories with its
  leave-one-trajectory-out range, and the same accumulation of the level-wise median of the
  per-trajectory rates.}
  \label{fig:numerical_cifar}
\end{figure}

\subsection{Supplementary Results: One-Dimensional Audit}
\label{subsec:numerical_supp_toy}

\begin{remark}[Finite-range growth and a uniform oracle bound]
  \label{rem:toy_uniform_budget}
  The oracle mixture is not a counterexample to a uniform bound on \(\Lambda_K\). Let
  \(\varsigma=0.04\) be its common component standard deviation, and \(m_{\mathsf r}\) the centre of
  the component from which \(X_0\) is drawn, a discrete random variable. Differentiating its
  posterior mean gives
  \begin{equation*}
    v'_\sigma(x)=\frac{\sigma}{\varsigma^2+\sigma^2}
    -\frac{\sigma}{(\varsigma^2+\sigma^2)^2}\operatorname{Var}(m_{\mathsf r}\mid X_\sigma=x).
  \end{equation*}
  The centres are bounded, so their conditional variance is bounded uniformly in \(x\). Hence
  \(L=\sup_{\sigma\in[\underline\sigma,\overline\sigma],x}|v'_\sigma(x)|<\infty\), and
  \begin{equation*}
    \log\max(1,\gamma_{j,K})\leqslant \log(1+\ell_jL)\leqslant \ell_jL,
    \qquad \Lambda_K\leqslant L(\overline\sigma-\underline\sigma).
  \end{equation*}
  This analytic bound holds for every refinement and initialization, but its constant is
  quantitatively uninformative. The logarithmic growth and the sublinear order of
  \Cref{subsec:evidence_law_level} are therefore preasymptotic: first order must eventually return,
  but not within the tested range, which already far exceeds the step counts used in practice
  \cite{karras_elucidating_2022}. The growth motivates the logarithmic envelope
  \eqref{eq:logarithmic_budget} without proving it.
\end{remark}

\Cref{fig:numerical_toy_supplementary} shows the measured recursion from the canonical Gaussian
initialization. With the oracle, the Gaussian start reproduces the exact-initialization curve of
\Cref{fig:numerical_toy}a: its propagated initialization term stays below \(\resGaussInitMax\),
against \(\resGaussDiscFine\) for discretization at \(K=\resKmax\), so the run is
discretization-dominated as well. With the learned denoisers, the propagated learning contribution
overtakes discretization from \(K=272\) on, for every seed, and the observed error levels off near
\(0.018\) on the finest grids for two of the three seeds; no order is assigned to these curves.

\paragraph{Refinement regime of \(\Lambda_K\).}
For every arm we record, in addition to the three accumulated curves, the increment
\(\Lambda_K-\Lambda_{K/2}\), its distribution over fixed noise decades, and the signed directional
rate \(\Erate[\OT]{j,K}\) at the levels of each \(G_K\) nearest to
\(\sigma\in\{0.04,0.065,0.09,0.15\}\). The question was whether \(\Lambda_K\) stabilizes only for
arms whose low-noise discrepancy has a learning-error floor. The answer is partial.
\begin{itemize}[leftmargin=1.6em]
  \item For the oracle, \(\Lambda_K\) grows by \(0.22\)--\(0.29\) per doubling, and at \(K=1088\),
  \(92\) percent of it comes from \(\sigma\in[0.01,0.1)\).
  \item The two arms with a large floor, the misspecified mean and the perturbed oracle with
  \(\varepsilon=0.1\), converge: their \(\Lambda_K\) and their fixed-noise rates settle at every
  tested level.
  \item The trained networks remain close to the oracle in \(\Lambda_K\)
  (\Cref{tab:toy_hierarchy}), but for all four denoisers the fixed-noise rates depend on the level
  chosen. At \(\sigma\approx0.065\) the oracle rate keeps growing (\(3.3\) to \(34\)) and the
  trained ones slow down (to about \(21\)). At \(\sigma\approx0.09\) the oracle rate levels off near
  \(11\) while two of the trained ones still grow (to \(17\)--\(18\)). The nearest grid level also
  differs from the target by up to \(10\) percent in \(\sigma\) for \(K\ge34\), and more on
  \(G_{17}\).
  \item The perturbed oracles with \(\varepsilon\leqslant 0.03\) show isolated collapses of
  \(\Lambda_K\) (for instance \(0.12\) at \(K=272\) for \(\varepsilon=0.01\)). At these step counts
  the local bias and the perturbation nearly cancel, so the displacement whose amplification
  \(\gamma_{j,K}\) measures nearly vanishes. These are not trends.
\end{itemize}
We therefore do not claim that a model-error floor regularizes \(\Lambda_K\) in general.

\begin{figure}[H]
  \centering
  \includegraphics[width=0.8\textwidth]{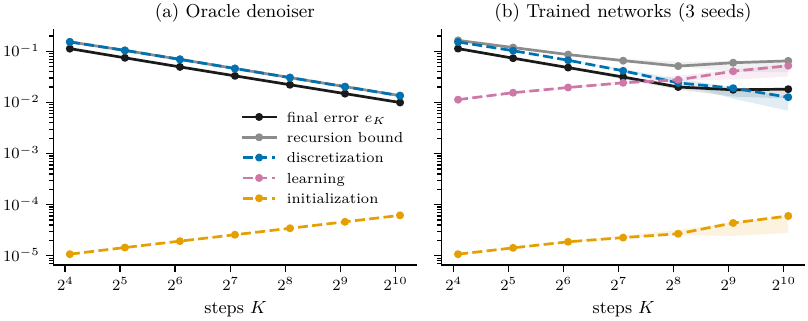}
  \caption[Measured recursion from the Gaussian start]{
  Terms of the measured recursion \eqref{eq:toy_propagated_decomposition} from the canonical
  Gaussian initialization: final error, recursion bound, and its propagated discretization, learning
  and initialization contributions. \textbf{(a)} Oracle denoiser. \textbf{(b)} Trained networks,
  median and range over three seeds. The contributions add up to an upper bound, not to the observed
  error.}
  \label{fig:numerical_toy_supplementary}
\end{figure}

\subsection{Analytic Controls}
\label{subsec:analytic_controls}

These controls reuse the mixture and analytic denoisers above, without network training or spatial
maximization. They use 4096 midpoint quantiles; the finest settings are repeated with 8192, and
differences comparable to the quadrature discrepancy are left uninterpreted.
\Cref{tab:control_clock,tab:control_start,tab:analytic_controls} report the outcomes.

\paragraph{Clock.}
For the oracle with exact initialization, we compare \(\rhoEDM=2,3,7\) at \(K=68,136,272,544\), with
the same endpoints, recording final error, measured local biases, their universal bounds, and the
relative mesh \(\theta_K=h_K/\underline\eta\) of \cref{eq:schedule_family_relative_mesh}. Equal
\(K\) means equal numbers of denoiser evaluations, but not equal relative meshes: with
\(\underline\sigma=0.002\), the clock \(\rhoEDM=2\) has \(\theta_K>1\) for \(K\leqslant 136\) and
reaches \(\theta_K<1/2\) only at \(K=544\), whereas \(\rhoEDM=7\) has \(\theta_K<0.06\) already at
\(K=68\). The clock \(\rhoEDM=1\) is excluded because \(\theta_K\) exceeds one on every step count
considered, so none of the estimates of \Cref{sec:two_regime_propagation} applies to it.

Outcome. The universal bias bound of \Cref{lem:local_euler_defect} exceeds the measured biases by a
factor of about \(280\) to \(1190\). It also reverses the ranking of the clocks: at \(K=544\) the
measured biases and the final error order the clocks \(\rhoEDM=7<3<2\), while the universal bound
orders them \(2<3<7\). The clock \(\rhoEDM=2\) balances the leading universal bound, but it need not
minimize observed biases or final error, since the mixture concentrates its local biases in a narrow
noise window. The bound controls rates, not the choice of schedule.

\begin{table}[H]
  \centering
  \small
  \caption{Clock control at \(K=544\), oracle with exact initialization; every grid satisfies
  \(\theta_K<1/2\). Ratio: universal bound over measured biases.}
  \label{tab:control_clock}
  \begin{tabular}{@{}rrrrr@{}}
    \toprule
    \(\rhoEDM\) & final error & measured biases & universal bound & ratio \\
    \midrule
    2 & 0.0229 & 0.0164 & 4.64 & 282 \\
    3 & 0.0173 & 0.0095 & 5.13 & 539 \\
    7 & 0.0139 & 0.0065 & 7.48 & 1150 \\
    \bottomrule
  \end{tabular}
\end{table}

\paragraph{Starting noise.}
We use \(\rhoEDM=7\) and fix \(h\) to the clock step of the reference grid with
\(\overline\sigma=80\) and \(K=272\). Integer step counts are chosen whose upper endpoints are
nearest 20, 80, and 320 in the clock, and the actual endpoint \((\underline\sigma^{1/7}+Kh)^7\) is
reported, so every grid is clock-uniform with exactly the same \(h\). The oracle starts from the
exact top law. For this control alone \(\sighi=4\) and \(\bhi=0.9\). A global oracle certificate is
available from
\begin{equation*}
  0\leqslant D'_\sigma(x)\le \frac{\varsigma^2}{\varsigma^2+\sigma^2}
  +\frac{\sigma^2(\max_r m_r-\min_r m_r)^2}{4(\varsigma^2+\sigma^2)^2},
\end{equation*}
since the posterior variance is the within-component variance plus the variance of component means,
the latter bounded by one quarter of their squared span. If the displayed upper bound is at most
\(1-\bhi\) at every high-noise level, the oracle predictor is \((1-\bhi a_j)\)-Lipschitz.

Outcome. The certificate holds on every grid (\(\beff=0.886>6/7\)). Raising \(\overline\sigma\) from
\(20\) to \(318\) multiplies the undamped high-noise sum of universal bounds by \(14\) but the
damped one only by \(2.6\) (\Cref{tab:control_start}), and changes the final error by
\(1.6\times10^{-6}\) in absolute terms (\(7.5\times10^{-5}\) relatively). Early error is damped, as
the high-noise block predicts.

\begin{table}[H]
  \centering
  \small
  \caption{Starting-noise control, \(\rhoEDM=7\) with a fixed clock step; oracle started
  from the exact top law. The change in final error is relative to \(\overline\sigma=80\). The last
  two columns sum the universal bias bounds \(\Delta_j\) over the high-noise steps
  (\(\sigma_{j+1}\ge\sighi=4\)), without and with the weights \(\Pi_{j+1,\Khi}\) of the high-noise
  kernel.}
  \label{tab:control_start}
  \begin{tabular}{@{}rrrrrr@{}}
    \toprule
    \(\overline\sigma\) & \(K\) & final error & change & \(\sum_j\Delta_j\) & \(\sum_j\Pi_{j+1,\Khi}\Delta_j\) \\
    \midrule
    19.8 & 209 & 0.02176 & \(-1.6\times10^{-6}\) & 3.36  & 1.54 \\
    80   & 272 & 0.02176 & \multicolumn{1}{c}{--} & 13.77 & 2.81 \\
    318  & 348 & 0.02176 & \(+0.1\times10^{-6}\) & 47.60 & 4.00 \\
    \bottomrule
  \end{tabular}
\end{table}

\paragraph{Calibration of the synchronous diagnostic.}
For the oracle and misspecified posterior mean, we compare synchronous \(G_K/G_{2K}\) pairs at
\(K=68,136,272\) with the optimal target--sampler coupling, both from canonical Gaussian
initialization, on the common window \(\sigma_j\le2\). Latent ranks are preserved for synchronous
pairs and sorted only for optimal transport.

Outcome. For the oracle the two accumulated rates grow together (ratio \(1.29\) to \(1.24\) for
\(K=68\) to \(272\)). For the misspecified denoiser they move apart: the synchronous rate rises from
\(1.11\) to \(1.87\) while the optimal-coupling rate falls from \(0.76\) to \(0.69\), and so does
\(\Lambda_K\) on the same window, from \(0.80\) to \(0.70\). A growing synchronous rate therefore
need not track \(\Lambda_K\) once the denoiser is inexact.

\begin{table}[H]
  \centering
  \small
  \caption{Calibration of the synchronous diagnostic: accumulated linear rates
  \(\sum_j\ell_j(\Erate[\sync]{j,K})_+\) and \(\sum_j\ell_j(\Erate[\OT]{j,K})_+\) over the levels
  \(\sigma_j\le2\), canonical Gaussian initialization. Ratio: synchronous over optimal coupling.
  Last column: \(\Lambda_K\) restricted to the same levels, \(\sum_j\log\max\{1,\gamma_{j,K}\}\).}
  \label{tab:analytic_controls}
  \begin{tabular}{@{}lrrrrr@{}}
    \toprule
    denoiser & \(K\) & synchronous & optimal coupling & ratio & \(\Lambda\) on window \\
    \midrule
    oracle       & 68  & 0.909 & 0.704 & 1.29 & 0.793 \\
                 & 272 & 1.605 & 1.296 & 1.24 & 1.326 \\
    misspecified & 68  & 1.112 & 0.762 & 1.46 & 0.804 \\
                 & 272 & 1.870 & 0.687 & 2.72 & 0.699 \\
    \bottomrule
  \end{tabular}
\end{table}

\subsection{Supplementary CIFAR--10 Diagnostics}
\label{subsec:numerical_supp_cifar}

\paragraph{Accumulated synchronous rates.}
\Cref{subsec:evidence_directional_cifar} discusses the endpoints of the two curves of
\Cref{fig:numerical_cifar}d. On the intermediate grids \(K=34,68,136\), the accumulated
energy-weighted rate is \(0.41\), \(0.50\), \(0.67\), and the accumulated median rate of
\Cref{exp:transport} is \(0.20\), \(0.22\), \(0.24\), starting from \(0.16\) at \(K=17\); the
successive ratios of the median rate decrease from \(1.27\) to \(1.04\).

\paragraph{Influence of rare trajectories.}
Concentration increases with refinement. Half of the total leave-one-out change is carried by
\(129\) trajectories at \(K=17\) and by \(21\) at \(K=272\). At \(K=272\) five coarse--fine pairs
have a final squared displacement more than \(100\) times their class median; the largest is
\(3268\) times that median. The median squared displacement decreases with a local exponent
approaching \(-2\) (\(-1.93\) on the last doubling), the first-order expectation. The most
influential trajectory is different on every grid, and every large leave-one-out change is a
decrease, so these pairs cause the rise rather than offset it. The refinement trend is therefore
described as sensitive to rare pairs, not as a property of the ensemble.

\subsection{Numerical Reliability and Sensitivity Checks}
\label{subsec:numerical_methods}

\Cref{tab:numerical_reliability} lists, for each reported quantity, the check applied, its
acceptance criterion, and the outcome. Two checks fail, both on the fine CIFAR segment grid
\(G_{272}\); the corresponding statistics are withdrawn from \Cref{fig:numerical_cifar} and from the
main text.

For the exact-initialization comparison, we additionally repeat all three grids \(K=272,544,1088\)
with \(4M=262{,}144\) quantiles. This checks both endpoints of each of the last two refinement pairs
at a common quadrature resolution (\Cref{tab:oracle_paired_orders}). The finest-grid error changes
from \(0.00987629\) to \(0.00996318\), and its cumulative log-amplification from \(1.91498\) to
\(1.92367\). On each pair the error order and the amplification-based order shift by the same amount
and in the same direction, by about \(0.003\) on \((272,544)\) and \(0.0068\) on \((544,1088)\), so
their difference changes by less than \(10^{-4}\). The individual orders are thus not resolved to
the third decimal, but their agreement is robust to this refinement. These differences measure
sensitivity to the tested quadrature refinement, not certified errors relative to exact integration.

\begin{table}[H]
  \centering
  \small
  \caption{Paired quadrature check with exact initialization. Local orders
  \cref{eq:empirical_order_main} on the pair \((K/2,K)\); the amplification-based order is
  \(1-(\Lambda_K-\Lambda_{K/2})/\log2\), and the bias-sum order refers to
  \(\sum_j\widetilde\delta_{j,K}\). Digits show numerical drift rather than certified precision.}
  \label{tab:oracle_paired_orders}
  \begin{tabular}{@{}rrrrr@{}}
\toprule
pair $(K/2,K)$ & quantiles & error order & ampl.-based order & bias-sum order \\
\midrule
$(272,544)$ & 65536 & 0.5811 & 0.5886 & 0.9908 \\
$(272,544)$ & 262144 & 0.5779 & 0.5854 & 0.9908 \\
$(544,1088)$ & 65536 & 0.5850 & 0.5882 & 0.9955 \\
$(544,1088)$ & 262144 & 0.5782 & 0.5814 & 0.9955 \\
\bottomrule
\end{tabular}

\end{table}

\begin{table}[H]
  \centering
  \reliabilitytablestyle
  \caption{Numerical reliability checks. Failures lead to omission, not to a qualified report.}
  \label{tab:numerical_reliability}
  \renewcommand{\arraystretch}{1.15}
  \begin{tabular}{@{}>{\raggedright}p{2.9cm}>{\raggedright}p{4.4cm}>{\raggedright}p{3.4cm}>{\raggedright\arraybackslash}p{5.1cm}@{}}
    \toprule
    quantity & check & criterion & outcome \\
    \midrule
    \multicolumn{4}{@{}l}{One-dimensional mixture}\\
    endpoint error \(e_{K,K}\)
      & \(M\) vs \(4M\) quantiles on \(G_{1088}\) (analytic arms, trained seed 0)
      & discrepancy \(<e_{K,K}/10\)
      & pass; oracle error changes by \(<1\%\). Bound/error ratio not reported for arms without a
        \(4M\) rerun \\
    cumulative log-amplification \(\Lambda_K\)
      & \(M\) vs \(4M\) on \(G_{1088}\)
      & small relative change
      & Gaussian start: about \(0.1\%\); exact start: about \(0.45\%\).
        Paired orders checked separately in \Cref{tab:oracle_paired_orders} \\
    one-step recursion \eqref{eq:toy_measured_recursion}
      & measured slack
      & nonnegative up to quadrature
      & pass on every arm and step \\
    field proxy \cref{eq:toy_field_proxy}
      & domain \((R,N_x)=(12,4097)\to(16,8193)\)
      & small relative change
      & pass; \(\leqslant 0.8\%\). The maximizer hits the domain boundary at 225 of 507 oracle levels on
        \(G_{1088}\) (never for trained arms) but sits on a flat plateau, not clipped \\
    noise threshold
      & re-summation at the four candidate thresholds of \Cref{exp:toy_law}
      & range reported
      & \(\Lambda_K\) bit-identical; \(\Lambda_K^{\rm dir}\) varies by \(0.13\%\) at \(K=1088\) \\
    clock control
      & \(M=4096\) vs \(8192\) at \(\rhoEDM=7\)
      & discrepancy \(<10\%\)
      & pass; \(8.7\%\), below the interpreted differences \\
    starting-noise control
      & finite-mesh certificate
      & \(\bhi\kappa(\thetahi)>6/7\) on every grid
      & pass; \(\beff=0.886\) \\
    \midrule
    \multicolumn{4}{@{}l}{CIFAR--10, high noise}\\
    margin crossings \(\sigma=\auditBcrossQninety,\auditBsupercritQninety\)
      & independent rerun
      & same grid levels
      & pass; reproduced exactly \\
    empirical margin \eqref{eq:branch_stability_margin}
      & 8, 16, 32 power iterations at 13 levels of \(G_{136}\) around the crossings
      & crossings unchanged
      & pass; \(\bprobe[0.9]\) moves by \(\le5\cdot10^{-5}\) (tail probes by up to \(0.1\)
        below \(\sigma\approx10\)); it stays \(+5.7\cdot10^{-4}\) at \(\sigma=\auditBcrossQninety\)
        and exceeds \(6/7\) by \(6.6\cdot10^{-3}\) at \(\sigma=\auditBsupercritQninety\) \\
    \midrule
    \multicolumn{4}{@{}l}{CIFAR--10, low-noise segments}\\
    secant integral, \(G_{68}\)
      & trapezoid of \(\lambda^\parallel\) on \(T_9\) vs \(\Erate[\sync]{j,K}\)
      & agreement
      & pass; within \(0.011\) \\
    \(\lambda^{\max}\), \(G_{68}\)
      & Lanczos Ritz residual, half-iteration drift
      & small residual
      & pass; \(<0.11\%\) of \(\lambda^{\max}\) except at \(\sigma=0.004\) (\(4.6\%\), an
        underestimate that only widens the reported gap) \\
    segment sample, \(G_{68}\)
      & ESS \cref{eq:segment_energy_weights} of 80 segments
      & not dominated by one segment
      & pass; ESS \(18\)--\(40\) \\
    segment sample, \(G_{272}\)
      & ESS; bootstrap at \(\sigma=0.004\)
      & not dominated by one segment
      & \textbf{fail}; one segment carries \(88\)--\(90\%\) of the energy at every
        \(\sigma\leqslant 1\) (ESS \(1.2\)--\(1.3\)); bootstrap interval \(8.8\) times the point
        estimate. Not reported \\
    segment maxima, \(G_{272}\)
      & nested \(T_3/T_5/T_9\); rerun with 65 nested nodes
      & maxima resolved in \(t\)
      & \textbf{fail}; nested check passes, but with 65 nodes the nine-node maxima at
        \(\sigma\approx0.03,0.1,0.3\) are low by factors \(2\) to \(4.5\), and at
        \(\sigma\approx0.3\) the refined \(\lambda^\parallel\) exceeds the nine-node
        \(\lambda^{\max}\). The secant integral is reproduced within \(0.001\). Not
        reported; the nested check is necessary, not sufficient \\
    accumulated synchronous rate
      & leave-one-out over 1280 trajectories
      & trend stable under single removal
      & sensitive; two trajectories dominate at \(K=272\) (\Cref{subsec:evidence_directional_cifar}).
        Trend reported as driven by rare pairs (\Cref{subsec:numerical_supp_cifar}) \\
    \bottomrule
  \end{tabular}
\end{table}

\paragraph{Software checks.}
The one-dimensional code checks the predictor secant identity, the ordering between the realized and
directional quantities, the independence of the spatial interval from \(K\), the vanishing learned
term for the oracle, and the residual bound \(\Rres_j\leqslant \varepsilon\) of the perturbed
oracle. The CIFAR code checks exact nesting of the schedules, the synchronous predictor identity,
the segment chain on known linear maps, the random-direction rate against the exact Rayleigh
quotient of a known symmetric matrix, nestedness of the 3/5/9-node estimators, energy-weighted
aggregation, zero-displacement handling, conservation of trajectory identities in the influence
analysis, stratified resampling, and propagation of invalid observations. Reduced smoke
configurations exercise each calculation before the full runs. The targeted oracle audit also checks
that sharing target quantiles across nested grids reproduces the independent-grid calculation, and
that orders compare identical quadrature resolutions and initializations.

\subsection{Code and Data Availability}
\label{subsec:numerical_reproducibility}

The code and the frozen raw outputs of every run, with their resolved configurations, are
\codeavailability. Two scripts regenerate from these outputs the numerical
figures, the generated tables and the measured values quoted in the main text; a check mode fails
when one of them is stale with respect to the outputs. The other tables and values of this appendix
are computed from the same outputs but are not regenerated by the scripts. The one-dimensional
experiments run on a CPU in a few hours in total. The CIFAR--10 diagnostics use the public
class-conditional EDM checkpoint of
\textcite{karras_elucidating_2022} and the CIFAR--10 test images; the longest of them, the
synchronous diagnostic, takes about six hours on a single NVIDIA L4 GPU.

\section{Comparison with Prior Wasserstein Analyses}
\label{app:related_work_extended}

This appendix compares our analysis with the two closest Wasserstein analyses of probability-flow
samplers, which follow the two routes of \Cref{app:alt_decomposition}: Beyler--Bach transport the
learned map, Gao--Zhu the exact flow. It closes with the counterexamples that motivate measuring
propagation separately from score accuracy.

\subsection{Closest Comparison: Beyler--Bach}

\textcite{beyler_convergence_2025} study deterministic and stochastic samplers on the same Gaussian
smoothing path. Their heat parameter is the variance \(t=\sigma^2\) in our notation. If the data law
is supported in a ball of radius \(R\), their score Hessian estimate implies that the exact Euler
map
\begin{equation*}
  f_{t,h}(x)=x+h\nabla\log p_t(x)
\end{equation*}
satisfies, for \(h\le t\),
\begin{equation}\label{eq:beyler_bach_exact_map_factor}
  \operatorname{Lip}(f_{t,h}) \le L_{t,h} =1+h\left(\frac{R^2}{t^2}-\frac1t\right).
\end{equation}
In particular, this certificate becomes contractive for \(t>R^2\)
\cite[Lemma~3]{beyler_convergence_2025}. The map \(f_{t,h}\) is our oracle predictor in their
variables: since \(v_\sigma=-\sigma\nabla\log p_\sigma\) by \cref{eq:velocity}, with
\(\ell_j=\sigma_j-\sigma_{j+1}\), \cref{eq:oracle_edm_predictor} reads
\(\Phi_j=\Id-\ell_jv_{\sigma_j}=f_{\sigma_j^2,\sigma_j\ell_j}\), and the condition \(h\le t\) holds
automatically because \(\ell_j<\sigma_j\). Their learned-score assumption imposes an analogous
Lipschitz control on \(x\mapsto x+h s_\theta(t,x)\) \cite[Assumption~2]{beyler_convergence_2025}.
The current discrepancy is then propagated through that worst-case map factor, while score and
discretization errors are evaluated on the exact law. This produces the same three-term recursion as
ours.

For Euler, their global estimate is first order in the maximal step
\cite[Proposition~10]{beyler_convergence_2025}. Its constants, however, inherit the bounded-support
radius and deteriorate as the sampler approaches the data law; the early-stopping dependence
contains, in particular, a factor of the form \(\exp(R^2/(2\varepsilon))\). The difference is thus
not the order of convergence but the mechanism and the constants used to obtain it. The comparison
can be summarized as follows.

\begin{center}
  \small
  \begin{tabular}{@{}p{0.20\linewidth}p{0.35\linewidth}p{0.37\linewidth}@{}}
\toprule
 & \textcite{beyler_convergence_2025} & This paper \\
\midrule
Path
  & Heat variance \(t\)
  & \(X_\sigma=X_0+\sigma Z\), with \(t=\sigma^2\) \\
Sampler
  & Euler, Heun, and Euler--Maruyama
  & First-order deterministic EDM predictor \\
Local bias
  & Bounds depending on \(R\) or on Gaussian smoothing
  & Universal acceleration bound, with no structural assumption on \(q\) \\
Learned error
  & \(\mathbb L_2\) score error under the exact law
  & Normalized EDM regression residual under the exact law \\
Propagation
  & Global Lipschitz envelope of the learned map
  & Realized amplification factor \(\gamma_j\), with directional and worst-case majorants \\
High noise
  & Contraction from smoothing and bounded support
  & Contraction explicitly certified by EDM preconditioning \\
\bottomrule
  \end{tabular}
\end{center}

Our local-bias estimate replaces the support-dependent control by posterior identities valid for
every Borel data law. For propagation, \cref{eq:beyler_bach_exact_map_factor} is analogous to our
map-level bound \eqref{eq:map_gamma_bound}, an intermediate link of the chain
\cref{eq:field_gamma_chain}. Our recursion instead retains the smaller quantity it actually
multiplies, the amplification factor \(\gamma_j\), which admits the transport-directional bound
\eqref{eq:directional_gamma_chain}. Conversely, Beyler--Bach cover Heun and Euler--Maruyama; the
global theorem of this paper covers only the first-order deterministic predictor.

Retaining \(\gamma_j\) sharpens the propagation step but does not by itself give an a priori
guarantee, since \(\gamma_j\) depends on the sampler laws (\Cref{sec:two_regime_stability}). The
separation example of \Cref{prop:strict_law_map_separation} shows that a worst-case map bound can be
unnecessarily large, not that an envelope for \(\Lambda_K\) holds for general learned denoisers.

\subsection{A Structural Wasserstein Route: Gao--Zhu}

\textcite{gao_zhu_probability_flow_2025} analyze probability-flow ODE samplers directly in
\(W_2\) for a general linear forward diffusion
\begin{equation*}
  \dd X_t=-f(t)X_t\,\dd t+g(t)\,\dd B_t.
\end{equation*}
They assume that the data density \(p_0\) is positive and differentiable, with \(-\log p_0\)
\(m_0\)-strongly convex and \(L_0\)-smooth. They also impose a time-regularity bound on the exact
score within each step: for all \(x\) and \(k=1,\ldots,K\), \(\sup_{(k-1)h_{\mathrm{GZ}}\le t\le
kh_{\mathrm{GZ}}}
\norm{\nabla\log p_{T-t}(x)-\nabla\log p_{T-(k-1)h_{\mathrm{GZ}}}(x)}
\le L_1h_{\mathrm{GZ}}(1+\norm{x})\), where \(h_{\mathrm{GZ}}\) is their step size, written so
to avoid a clash with our clock \(\eta\), and \(T=Kh_{\mathrm{GZ}}\) is the horizon. They further
assume an \(\mathbb L_2\) score-estimation error bounded by \(M\) along the numerical iterates and a
sufficiently small step size \cite[Assumptions~3.1--3.3,
Theorem~3.4]{gao_zhu_probability_flow_2025}. Their exponential-integrator sampler then satisfies
\begin{equation*}
  W_2\bigl(\Law(u_K),p_0\bigr)
  \le \exp\!\left(-\int_0^{Kh_{\mathrm{GZ}}}\mu_{\mathrm{GZ}}(t)\,\dd t\right)
  \norm{X_0}_{\mathbb L_2} +E_1(f,g,K,h_{\mathrm{GZ}},L_1) +E_2(f,g,K,h_{\mathrm{GZ}},M,L_1),
\end{equation*}
where \(E_1\) and \(E_2\) are respectively the discretization and score-matching terms
\cite[Theorem~3.4]{gao_zhu_probability_flow_2025}. The continuous contraction rate is explicit:
\begin{equation*}
  \mu_{\mathrm{GZ}}(t) = \frac{m_0g(t)^2}{ 2\left[ \exp\!\left(-2\int_0^t f(s)\,\dd s\right)
  +m_0\int_0^t \exp\!\left(-2\int_s^t f(v)\,\dd v\right)g(s)^2\,\dd s \right] }>0.
\end{equation*}
Strong log-concavity is what makes \(\mu_{\mathrm{GZ}}\) positive. At the discrete level, a
synchronous coupling yields a one-step factor, written here \(\gamma_j^{\mathrm{GZ}}\), constrained
to \((0,1)\) by the step-size condition and then multiplied through the accumulated error
\cite[Proposition~4.2]{gao_zhu_probability_flow_2025}. This factor should not be confused with our
\(\gamma_j\): theirs is an explicit coupling-based contraction coefficient, whereas ours is the
ratio of two Wasserstein distances after applying the learned finite map.

The comparison with the present analysis is summarized below.

\begin{center}
  \small
  \begin{tabular}{@{}p{0.20\linewidth}p{0.35\linewidth}p{0.37\linewidth}@{}}
\toprule
 & \textcite{gao_zhu_probability_flow_2025} & This paper \\
\midrule
Forward path
  & General linear diffusion with coefficients \(f,g\)
  & Gaussian path \(X_\sigma=X_0+\sigma Z\) \\
Sampler
  & Exponential integrator for the learned probability-flow ODE
  & First-order finite EDM predictor \\
Data assumptions
  & Smooth, strongly log-concave density
  & No structural assumption for the local bias; centered \(q\in\mathcal P_2\) for the global bound \\
Local bias
  & Time regularity and spatial Lipschitz control of the score
  & Universal acceleration bound \\
Learned error
  & \(\mathbb L_2\) score error along the numerical iterates
  & Normalized EDM regression residual under the exact law \\
Propagation
  & Synchronous coupling with explicit contraction from log-concavity
  & Realized amplification factor \(\gamma_j\), with directional and worst-case majorants \\
\bottomrule
  \end{tabular}
\end{center}

Both analyses therefore separate initialization, discretization and learning errors in the same
metric, but place the structural burden at different points. Gao--Zhu derive contractivity of the
continuous flow and of the coupled discretization from the target geometry. We obtain the local bias
without that geometry, then expose propagation as a separate condition on the learned finite
transport: EDM preconditioning certifies it at high noise, while at low noise only the cumulative
low-noise log-amplification \(\Lambda_K\) is required.

\subsection{Score Accuracy Does Not Imply Stability}

The counterexamples of \textcite{zhou_score_accuracy_2026} show that a small score error does not by
itself make a sampler stable. A tail perturbation can have arbitrarily small score error under every
forward marginal while explicit Euler develops diverging moments; in the deterministic
probability-flow construction, the discrete laws converge weakly but diverge in every \(W_p\),
\(p\ge1\) \cite[Theorem~3.1 and Corollary~3.3]{zhou_score_accuracy_2026}. A residual evaluated on
the exact law therefore cannot control the error alone: the propagation of an existing discrepancy
through the learned map must be measured separately, which is the role of \(\gamma_j\) in our
recursion. The deterministic example does not contradict our theorem, since its field has
superlinear tail growth and violates \Cref{ass:wellposed}. Their bounded, globally Lipschitz neural
construction, which would satisfy that assumption, concerns Euler--Maruyama rather than the
deterministic predictor, and its constants are not uniform along the family
\cite[Theorem~3.5]{zhou_score_accuracy_2026}.

\end{document}